%% file: main.tex
\documentclass{article} 
\usepackage{iclr2024_conference,times}

\input{math_commands}

\input{macros}

\usepackage[utf8]{inputenc} 
\usepackage[T1]{fontenc}    
\usepackage[pdfencoding=unicode, psdextra,colorlinks=true,citecolor=blue]{hyperref}       

\usepackage{url}            
\usepackage{booktabs}       
\usepackage{amsfonts}       
\usepackage{amsthm}
\usepackage{amsmath}
\usepackage{amssymb}
\usepackage{nicefrac}       
\usepackage{microtype}      
\usepackage{algorithm}
\usepackage[noend]{algpseudocode}
\usepackage{graphicx}
\usepackage{wrapfig}
\usepackage{float, caption, subcaption}
\usepackage{mathtools}
\usepackage{enumerate}
\usepackage{caption}
\usepackage{cases}
\usepackage{mathrsfs}
\usepackage{verbatim}
\usepackage[most]{tcolorbox}
\usepackage{empheq}
\usepackage[capitalize]{cleveref}
\usepackage{pifont}

\crefname{equation}{Eqn.}{Eqns.}
\crefname{ass}{Assumption}{Assumptions}
\crefname{thm}{Theorem}{Theorems}
\crefname{figure}{Figure}{Figures}
\newtheorem{thm}{Theorem}

\newtheorem{prop}{Proposition}
\newtheorem{lem}[prop]{Lemma}
\newtheorem{cor}[prop]{Corollary}
\newtheorem{example}{Example}

\newtheorem{defi}{Definition}
\newtheorem{ass}{Assumption}
\theoremstyle{remark}
\newtheorem*{remark}{Remark}

\usepackage[symbol,bottom]{footmisc}

\usepackage{tikz}
\usetikzlibrary{calc, backgrounds}
\usepackage{pgfplots}
\pgfplotsset{compat=1.11,
legend image code/.code={
\draw[mark repeat=2,mark phase=2]
plot coordinates {
(0cm,0cm)
(0.15cm,0cm)        
(0.3cm,0cm)         
};%
}
}

\title{Understanding Augmentation-based \\ Self-Supervised Representation Learning via RKHS Approximation and Regression}

\author{Runtian Zhai, Bingbin Liu, Andrej Risteski, Zico Kolter, Pradeep Ravikumar  \\
  Carnegie Mellon University  \\ 
  \href{mailto:rzhai@cs.cmu.edu?subject=[ICLR'24] Questions about Understanding Augmentation-based SSL\&cc=pradeepr@cs.cmu.edu}{{\color{black} \texttt{\{rzhai,bingbinl,aristesk,zkolter,pradeepr\}@cs.cmu.edu}}}
}

\iclrfinalcopy 

\begin{document}

\maketitle
\input{body}

\ificlrfinal
\subsubsection*{Code}
The code of \Cref{sec:bt-big} can be found at \url{https://colab.research.google.com/drive/1loSZLLI-qfoKE7BCIi1SWJKgruU6i4ku?usp=sharing}.

\subsubsection*{Acknowledgments}
We would like to thank Yutong He, Di He and Elan Rosenfeld for their very useful discussions, and Jeremy Cohen, Amrith Setlur, Xiaoyu Huang, Samuel Sokota, Zhili Feng, Swaminathan Gurumurthy and Zhengyang Geng for their feedback on the early draft of this work. We are also very grateful to our anonymous ICLR reviewers, with whose help this work has been greatly improved. We acknowledge the support of NSF via IIS-2211907, IIS-1909816, CCF-2238523, ONR via N00014-23-1-2368, and Amazon.
\fi

\bibliographystyle{iclr2024_conference}

\clearpage 
\appendix
\input{appendix}

\end{document}

%% file: math_commands.tex
\usepackage{amsmath,amsfonts,bm}

\def\eqref#1{equation~\ref{#1}}

\def\1{\bm{1}}

\def\diag{{\textnormal{diag}}}
\def\rank{{\textnormal{rank}}}
\def\sspan{{\textnormal{span}}}

\def\vsigma{{\bm{\sigma}}}

\def\ve{{\bm{e}}}

\def\vg{{\bm{g}}}

\def\vq{{\bm{q}}}

\def\vu{{\bm{u}}}
\def\vv{{\bm{v}}}

\def\mA{{\bm{A}}}
\def\mB{{\bm{B}}}
\def\mC{{\bm{C}}}
\def\mD{{\bm{D}}}

\def\mF{{\bm{F}}}
\def\mG{{\bm{G}}}

\def\mI{{\bm{I}}}

\def\mK{{\bm{K}}}

\def\mM{{\bm{M}}}

\def\mP{{\bm{P}}}
\def\mQ{{\bm{Q}}}

\def\mS{{\bm{S}}}

\def\mU{{\bm{U}}}
\def\mV{{\bm{V}}}

\def\mSigma{{\bm{\Sigma}}}

\DeclareMathAlphabet{\mathsfit}{\encodingdefault}{\sfdefault}{m}{sl}
\SetMathAlphabet{\mathsfit}{bold}{\encodingdefault}{\sfdefault}{bx}{n}

\def\gA{{\mathcal{A}}}

\def\gF{{\mathcal{F}}}
\def\gG{{\mathcal{G}}}
\def\gH{{\mathcal{H}}}

\def\gL{{\mathcal{L}}}

\def\gN{{\mathcal{N}}}

\def\gX{{\mathcal{X}}}

\def\sP{{\mathbb{P}}}

\newcommand{\E}{\mathbb{E}}

\newcommand{\R}{\mathbb{R}}

\DeclareMathOperator*{\argmin}{arg\,min}

\DeclareMathOperator{\Tr}{Tr}

%% file: macros.tex
\usepackage{scalerel}
\renewenvironment{boxed}
    {
    \begin{tcolorbox}[colback=black!2!white,colframe=black!75!black,left=2pt,right=2pt,top=2pt,bottom=2pt]
    }
    { 
    \end{tcolorbox}
    }

\newcommand{\rebuttal}[1]{{#1}}

\newcommand{\set}[1] {{\left \{ #1 \right \}}}
\newcommand{\sset}[2]{{\left \{ #1 \; \middle \vert \; #2 \right \}}}
\newcommand{\res}[2]{{#1 $\pm$ #2}}

\def \ie {\textit{i.e.}}
\def \eg {\textit{e.g.}}
\def \aew {\textit{a.e.}}
\def \iid {\textit{i.i.d.}}
\def \wrt {\textit{w.r.t.}}

\def \err {\textnormal{err}}

\def \px {{P_\gX}}

\def \pa {{P_\gA}}
\def \hpa {{\hat{P}_\gA}}
\def \hpx {{\hat{P}_\gX}}
\def \lxp {{L^2(P_\gX)}}

\def \lap {{L^2(P_\gA)}}
\def \lhxp {{L^2(\hpx)}}
\def \lhap {{L^2(\hpa)}}
\def \pax {{P_{\gA \gX}}}

\def \tax {\Gamma}

\def \rad {\mathfrak{R}}
\def \lrad {{\bar{\rad}}}

\def \bt {\kappa}

\def \downe {\gF(\tax; \epsilon)}
\def \Bdowne {\gF_{B}(\tax; \epsilon)}

\def \nse {\nu}
\def \eph {{\hat{\gH}_{\tax}}}

\def \laph {{\gH_{\tax}}}
\def \ephd {{\hat{\gH}_d}}

\def \hPhi {{\hat{\Phi}}}
\def \hPsi {{\hat{\Psi}}}
\def \learnh {\gH_{\hPsi}}
\def \learnk {{\hat{K}_{\hPsi}}}

\def \Phistar {{\boldsymbol{\Phi}^*}}
\def \Psistar {{\boldsymbol{\Psi}^*}}
\def \hPhistar {{\bar{\boldsymbol{\Phi}}^*}}
\def \Phistard {{\boldsymbol{\Phi}^*_d}}
\def \hPhistard {{\bar{\boldsymbol{\Phi}}^*_d}}
\def \dlbd {{\mD_{\lambda}}}

\def \mhG {{\hat{\mG}}}
\def \mhF {{\hat{\mF}}}

\def \fproj {{f_{\hat{\Psi}}}}

\def \ka {{K_A}}
\def \kx {{K_X}}
\def \taxlong {{\Gamma_{x \rightarrow a}}}
\def \taxstar {{\Gamma_{a \rightarrow x}}}
\def \pxp {{P_X^+}}
\def \pap {{P_A^+}}

%% file: body.tex
\begin{abstract}
Data augmentation is critical to the empirical success of modern self-supervised representation learning, such as contrastive learning and masked language modeling.
However, a theoretical understanding of the exact role of augmentation remains limited.
Recent work has built the connection between self-supervised learning and the approximation of the top eigenspace of a graph Laplacian operator, suggesting that learning a linear probe atop such representation can be connected to RKHS regression.
Building on this insight, this work delves into a statistical analysis of augmentation-based pretraining.
Starting from the isometry property, a geometric characterization of the target function given by the augmentation, we disentangle the effects of the model and the augmentation,
and prove two generalization bounds that are free of model complexity.
Our first bound works for an arbitrary encoder, where the prediction error is decomposed as the sum of an estimation error incurred by fitting a linear probe with RKHS regression, and an approximation error entailed by RKHS approximation.
Our second bound specifically addresses the case where the encoder is near-optimal, that is it approximates the top-$d$ eigenspace of the RKHS induced by the augmentation.
A key ingredient in our analysis is the \textit{augmentation complexity},
which we use to quantitatively compare different augmentations and analyze their impact on downstream performance. 
\end{abstract}

\section{Introduction}

It is widely acknowledged that better data augmentation techniques have been a major driving force in many recent breakthroughs in self-supervised representation learning.
For example, contrastive learning~\citep{chen2020simple} with aggressive random cropping and strong color distortion greatly improves the performance of vision tasks,
and masked prediction with random masking and replacement is among the state-of-the-art representation learning methods for both natural language processing~\citep{devlin2018bert} and vision~\citep{he2022masked}.
Due to their extraordinary empirical successes, developing a rigorous theoretical understanding of augmentation-based representation learning is an important open problem, and is crucial for inventing new, better and principled augmentation techniques.

A key advance was recently made by \citet{haochen2021provable}, who analyzed a variant of contrastive representation learning, termed spectral contrastive learning, by connecting it to learning eigenfunctions of a Laplacian operator over a population \textit{augmentation graph}. This has been followed up by multiple works which extended these results to more general contrastive learning approaches~\citep{saunshi2022understanding,johnson2022contrastive,cabannes2023ssl}. However, the generalization guarantees in these works have to explicitly grapple with the function class being used to learn the eigenfunctions (typically deep neural networks in practice). Such a dependency was even deemed necessary by \citet{saunshi2022understanding} who argued that ``function-class-agnostic analysis leads to vacuous guarantees.''
Consequently, the generalization bounds such as \citet[Theorem~4.2]{haochen2021provable} need to depend on the Rademacher or other complexities of this encoder function class,
yet finding proper complexity measures for flexible function classes such as neural networks remains an open problem.
Moreover, with the effects of the augmentation and the encoder function class intertwined, these analyses cannot discern the exact role of augmentation in self-supervised representation learning.

In this work, we factor out the effect of the encoder function class completely.
Our starting point is the observation that target functions that are ``soft invariant'' to the augmentation lie in a specific RKHS $\laph$ which we call the \textit{(augmentation) induced RKHS}.
Meanwhile, the $d$-dimensional feature encoder we pretrain also spans an RKHS, and the closer this RKHS is to the induced RKHS, the better the encoder.
Thus, pretraining can be viewed as approximating the induced RKHS with unlabeled data.
Then, learning a linear probe atop corresponds to regression on the encoder's RKHS with labeled data.
This perspective elucidates what roles the two stages of representation learning are playing:

\begin{itemize}
    \item \textbf{Upstream:} The self-supervised pretraining stage can be viewed as doing RKHS approximation.
    It learns, with unlabeled data, an encoder $\hat{\Psi}$ whose learned RKHS $\learnh$ (\cref{eq:Khat}) approximates $\laph$.
    \item \textbf{Downstream:} The supervised stage is doing RKHS regression over $\learnh$ with labeled data.
\end{itemize}
\cref{fig:teaser} provides an illustration of this RKHS approximation/regression perspective.
The learning performance is measured by the prediction error of the downstream predictor,
which decomposes into two parts based on the above framework:
the approximation error upstream entailed by RKHS approximation, and the estimation error downstream entailed by RKHS regression.

Thus, using classical function analytic tools, this leads us to our first set of results: 
For an \textit{arbitrary} pretrained encoder, we provide a generalization bound on the $L^2$ distance between the predictor and the target function (\cref{thm:main-2}).
We emphasize that the encoder can be \textit{arbitrary}, \ie{} any architecture and size.
As a result, \cref{thm:main-2} starkly contrasts prior work in that it disentangles the effects of the model and the augmentation: our generalization bound (a) is nonparametric, (b) does not depend on any model complexity or model inductive bias, and hence (c) allows the user to choose any model class for the encoder.
Our bound depends on two key quantities: 
(a) the augmentation complexity $\bt$, which is smaller for stronger and ``better behaved'' augmentations;
(b) the trace gap, which serves as a proxy of how well the encoder approximates the induced RKHS.

The bound for an arbitrary encoder is great, but one might also be interested in proving guarantees for the \textit{optimal} encoder that minimizes the worst-case approximation error.
We will show that the optimal $d$-dimensional encoder consists of the top-$d$ eigenfunctions of the induced RKHS (\cref{prop:reg}).
However, these eigenfunctions are not accessible given only finite samples, so we instead study the near-optimal $d$-dimensional encoder, which is a Monte-Carlo approximation of the optimal one.
Our second set of results provide a generalization bound for the near-optimal encoder, by bounding its trace gap  (which captures its quality) with high probability in terms of some spectral aspects of the induced RKHS.
The bound shows that: As the number of unlabeled samples goes to infinity, the near-optimal encoder will become optimal.

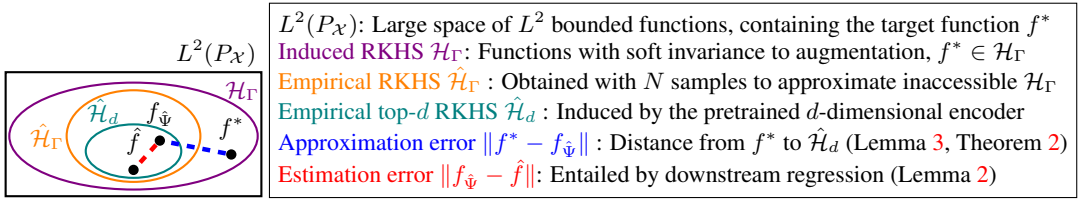
\begin{figure}[!t]
\vskip -.4in
    \input{framework}
    \vskip -.22in
    \caption{\small{Overall RKHS approximation/regression framework illustration and commentary.}}
    \label{fig:teaser}
    \vskip -.15in
\end{figure}

A practical implication of our framework is that the \emph{augmentation complexity} $\bt$, defined as the $L^{\infty}$ norm of the kernel of $\laph$,
can be used as a tool to quantitatively compare different augmentations and analyze their impact on downstream performance.
In Section \ref{sec:bt-big}, we demonstrate this point with mask-type augmentations on synthetic and real datasets,
and show that (i) $\bt$ depends on both the augmentation strength and the augmentation strategy;
(ii) a smaller $\bt$ (\eg{} stronger augmentation) leads to a smaller generalization gap,
but an overly strong augmentation causes poor training performance.
Thus, there is a ``sweet spot'' in the middle with the best test performance.
Overall, we believe that this work places modern representation learning on a more scientific footing and can inspire future empirical advancements.
Finally, it is worth emphasizing that this work is about studying self-supervised learning from a kernel perspective, rather than studying kernel methods.

\section{Problem Setup, Geometry of the Augmentation Induced RKHS}
\label{sec:prelim}

Let $\gX \subset \R^{d_\gX}$ be the data space, and $\px$ be the data distribution.
Let $\lxp$ be the $L^2$ function space, such that any $f \in \lxp$ satisfies $\E_{\px}[f(X)^2] < \infty$.
Let $\langle f_1, f_2 \rangle_\px = \int f_1(x) f_2(x) d \px (x)$ and $\| f \|_\px = \langle f, f \rangle_\px^{1/2}$ be the inner product and the norm of the Hilbert space $\lxp$.
Let $f^* \in \lxp$ be the target function we want to learn. 
This work studies the least-squares regression problem (see \citet{gyorfi2002distribution} for an introduction), which is formally stated as follows:

\begin{boxed}
\textbf{Problem.}
    \textit{
    Given unlabeled samples $x_1,\cdots,x_N$ and labeled samples $\tilde{x}_1,\cdots,\tilde{x}_n$ \iid{} sampled from $\px$, and labels $\tilde{y}_k = f^*(\tilde{x}_k) + \nse_k$ for $k \in [n]$ and random noise $\nse_k$, 
    find a predictor $\hat{f} \in \lxp$ with a low prediction error $\err(\hat{f}, f^*) :=  \| \hat{f} - f^* \|_\px^2= \E_{\px}[(\hat{f}(X) - f^*(X))^2]$.
    }
\end{boxed}

\vspace{-.07 in}
\subsection{The Augmentation Operator \texorpdfstring{$\tax$}{Gamma}: A Source of Prior Knowledge}
\label{sec:gamma}
\vspace{-.07 in}

The core problem in self-supervised representation learning is how to leverage the unlabeled data $x_1,\cdots,x_N$.
Without labels, the pretraining task must be built upon some \textit{prior knowledge} we have about the target function $f^*$.
Data augmentation based pretraining tasks are built upon the following prior knowledge:
\textit{Two augmentations of the same sample $x$ should be similar to each other.}

Formally, let $\gA$ denote an augmented space that contains the augmentations of $x \in \gX$.
Data augmentation induces a joint distribution $\pax$,
with marginal distributions $\pa$ and $\px$.
We will use capital letters $A$ and $X$ to denote random variables from $\pa$ and $\px$.
We define an \textit{augmentation operator} $\tax = \taxlong : \lxp \rightarrow \lap$ as $(\taxlong f)(a) = \E[f(X) | a]$ for all $f \in \lxp$,
and define its adjoint operator $\tax^* = \taxstar : \lap \rightarrow \lxp$ as $(\taxstar g)(x) = \E[g(A) | x]$ for all $g \in \lap$.
We can see that $\taxstar$ is indeed the adjoint of $\taxlong$,
since for all $f \in \lxp$ and $g \in \lap$,
there is $\langle \taxlong f, g \rangle_\pa  = \iint f(x) g(a) p(a,x) da dx = \langle f, \taxstar g \rangle_\px$.
Intuitively, the augmentation operator connects any single sample $x$ to the distribution of its augmentations $p(a|x)$.

\textbf{Example:}
In BERT, suppose we use 15\% random masking.
Then, $\gX$ is the space of original sentences, and $\gA$ is the space of 15\% masked sentences;
$\px$ is the distribution over original sentences, $A \sim p(\cdot|x)$ is the 15\% randomly masked version of an original sentence $x$, and $\pax(a,x) = \px(x) p(a|x)$.
Thus, $(\taxstar g)(x)$ is the mean of $g$ over all 15\% randomly masked versions of $x$.

Now we define the ``soft invariance'' \wrt{} the augmentation.
Denote the function class $\downe$ to be the set of function $f^*$ such that $\exists g^* \in \lap$ such that $f^*(x) = (\taxstar g^*)(x) = E[g^*(A)|x]$, and
\vspace{-0.2em}
\begin{equation}
\label{eqn:ass-lap}
    \frac{1}{2} \E_{X \sim \px} \E_{A, A' \sim p(\cdot | X)} \left [ (g^*(A) - g^*(A'))^2 \right ] \le \epsilon \|g^*\|_\pa^2 ,
\end{equation}
for some small positive constant $\epsilon$.
Eqn. (\ref{eqn:ass-lap}) is a regularization constraint induced by the random walk normalized Laplacian over the \textit{augmentation graph} defined in \citet{haochen2021provable},
where the edge weight between $a$ and $a'$ is given by $\pap(a,a')$ in Eqn. (\ref{eqn:ka-def}).
Compared to Assumption 1.1 in \citet{johnson2022contrastive},
\cref{eqn:ass-lap} has an additional $\|g^*\|_\pa^2$ term on the right-hand side so it is homogeneous.
For any $f \in \downe$, if $\|f\|_\px \le B$ for some constant $B$, then we denote $f \in \Bdowne$.
\vspace{-.05 in}
\begin{boxed}
\begin{ass}[See discussions in \cref{app:discussion}]
\label{ass:lap}
    We assume that $f^* \in \Bdowne$ for given $B, \epsilon$.
\end{ass}
\end{boxed}
Let the range of $\tax^*$ be $R(\tax^*) = \sset{f = \taxstar g}{g \in \lap}$.
Then $f^* \in R(\tax^*)$ by Assumption \ref{ass:lap}.

\vspace{-.05 in}
\subsection{Induced RKHS \texorpdfstring{$\laph$}{H} and the Isometry Property}
\label{sec:induced-rkhs}
\vspace{-.05 in}
Now, we show that $R(\tax^*)$ is an RKHS.
Define the following kernel over $\gA \times \gA$:
\begin{equation}
\label{eqn:ka-def}
    \ka(a_1,a_2) := \frac{d \pap}{d(\pa \otimes \pa)} =  \frac{\pap(a_1,a_2)}{ \pa (a_1) \pa (a_2)} , \  \pap(a_1,a_2) := \int p(a_1|x) p(a_2|x) d \px (x).
\end{equation}
Here, $\pap$ is called the \textit{associative distribution},
and $\ka$ is the \textit{Radon-Nikodym derivative} of $\pap$ \wrt{} $\pa \otimes \pa$, also called the positive-pair kernel in \citet{johnson2022contrastive}.
Next, we define a \textit{dual kernel} as:
\begin{equation}
\label{eq:Ka}
    \kx(x_1,x_2) := \frac{d \pxp}{d(\px \otimes \px)} = \frac{\pxp(x_1,x_2)}{ \px (x_1) \px (x_2)} = \int \frac{ p(a|x_1) p(a|x_2)}{ \pa (a)}  da  ,
\end{equation}
which is constructed by swapping $a$ and $x$ in $\kx$ and then applying the Bayes rule.
It holds that:
\begin{equation*}
\vspace{-.05 in}
\begin{cases}
(\taxstar \taxlong f)(x) = (\tax^* \tax f)(x) = \int  \kx(x,x') f(x') \px (x') dx';\\ 
(\taxlong \taxstar g)(a) = (\tax \tax^* g)(a) = \int \ka(a,a')  g(a') \pa (a') da'  .
\end{cases}
\end{equation*}
See App. \ref{app:sec2} for derivation.
Let $\lambda_i \in \R$ and $\psi_i \in \lxp$ be the sorted \textit{eigenvalue} and \textit{eigenfunction} of $\tax^* \tax$, \ie{} $\tax^* \tax \psi_i = \lambda_i \psi_i$,
with $\lambda_1 \ge \lambda_2 \ge \cdots \ge 0$.
Suppose $\int \kx(x,x')^2 d \px (x) d \px (x') < \infty$,
so that $\tax^* \tax$ is a compact Hilbert-Schmidt integral operator.
Then, by Hilbert-Schmidt theorem,
we can choose $\psi_1,\psi_2,\cdots$ that form an orthonormal basis of $\lxp$,
such that $\langle \psi_i, \psi_j \rangle_\px = \delta_{i,j}$,
and any $f \in \lxp$ can be written as $f = \sum_i u_i \psi_i$ for some $\{u_i\}_i$.

Note that $\lambda_i \in [0,1]$, since $\|g\|_\pa^2 \ge \| \tax^* g \|_\px^2$ for any $g \in \lap$ by Jensen's inequality.
It is easy to check that $\lambda_1 = 1$ and $\psi_1 \equiv 1$ is always a pair of eigenvalue and eigenfunction of $\tax^* \tax$.
Moreover, when $d$ is sufficiently large, $\lambda_d$ should be close to $0$.
Similarly, we can define a set of orthonormal eigenfunctions of $\tax \tax^*$ denoted by $\{\phi_i \}$,
and we can relate these two sets of eigenfunctions by:
\begin{prop}[Duality]
\label{prop:dual-2}
Operators $\tax \tax^*$ and $\tax^* \tax$ share the same non-zero eigenvalues, and there exist eigenfunctions $\{ \phi_i \}$ of $ \tax \tax^*$ that form an orthonormal basis of $\lap$, such that for any $\lambda_i > 0$,\vspace{-.05 in}
\begin{equation*}
    \psi_i = \lambda_i^{-1/2} \tax^* \phi_i = \lambda_i^{-1/2} \taxstar \phi_i \quad \text{and} \quad \phi_i = \lambda_i^{-1/2} \tax \psi_i = \lambda_i^{-1/2} \taxlong \psi_i .
\end{equation*}
Moreover, we have the following spectral decomposition of the Radon-Nikodym derivative:
\begin{equation}
\label{eqn:rn-der}
   \frac{dP_{ \gA \gX}}{d(\pa \otimes \px)} = \frac{ P_{ \gA \gX} (a,x)}{\pa (a) \px (x)} = \sum_i \lambda_i^{1/2} \phi_i(a) \psi_i(x)  .
\end{equation}
\end{prop}
\vspace{-.1in}
Define a Hilbert space $\laph := \sset{f = \sum_i u_i \psi_i \in \lxp}{\sum_i \lambda_i^{-1} u_i^2 < \infty}$,
whose inner product is $\langle f_1, f_2 \rangle_{\laph} = \sum_i \lambda_i^{-1} u_i v_i$
for $f_1 = \sum_i u_i \psi_i$ and $f_2 = \sum_i v_i \psi_i$.
It satisfies (proof in Appendix \ref{app:sec2}):\vspace{-.05 in}
\textit{
\begin{enumerate}[(i)]
    \item $\kx$ is the reproducing kernel of $\laph$, such that for all $f \in \laph$, $f(x) = \langle f, \kx(x, \cdot) \rangle_{\laph}$.
    \item $\laph = R(\tax^*)$. We call this the \textit{augmentation induced RKHS}.
    \item $\laph$ is isometric to $\textnormal{span}(\set{\phi_i}_{\lambda_i > 0})$, a subspace of $\lap$, and $\| f\|_{\laph} = \inf_{g: f = \tax^* g} \|g \|_\pa$.
    \item \rebuttal{For any $f^* \in \Bdowne \subset R(\tax^*)$, let $f^* = \sum_i u_i \psi_i$. Define $g_0 := \sum_i \lambda_i^{-1/2} u_i \phi_i$. Then, $g_0$ must satisfy \cref{eqn:ass-lap}, so we can choose $g^* = g_0$, in which case \cref{eqn:ass-lap} is equivalent to:}
\begin{equation}
\vspace{-.05 in}
\label{eqn:new-cond}
  \langle g^*, (I - \tax \tax^*) g^* \rangle_\pa \le \epsilon \| g^* \|_\pa^2 \; \Leftrightarrow \; \sum_i \frac{1-\lambda_i}{\lambda_i} u_i^2 \le \epsilon \sum_i \frac{1}{\lambda_i} u_i^2 .
\end{equation}
\end{enumerate}
}
\vspace{-.05in}
Moreover, Eqn. (\ref{eqn:new-cond}) is equivalent to a simple and intuitive \textit{isometry property}:\vspace{-.05 in}
\begin{boxed}
\begin{equation}
\label{eqn:lap-equ}
    (1 - \epsilon) \| f^* \|_{\laph}^2 \le \| f^* \|_\px^2 \le \| f^* \|_{\laph}^2  . 
\end{equation}
\end{boxed}
\rebuttal{
\textbf{Understanding the isometry property:}
This property essentially says that the operator $\tax^* \tax$ preserves most variance of the target function $f^*$ in the infinite-dimensional functional space.
Thus naturally, the optimal $d$-dimensional encoder that minimizes the approximation error should keep as much variance as possible under $\tax^* \tax$.
This isometry property can be considered as a counterpart of the restricted isometry property (RIP) in infinite-dimensional functional spaces:
In \cref{sec:top-d-main}, we will show that the optimal encoder consists of the top-$d$ eigenfunctions of $\tax^* \tax$.
Analogously, for a finite-dimensional vector space, PCA finds the $d$ principal components that keep the most variance under a linear transformation $T$, and they consist of the top-$d$ eigenvectors of $T$.
}

\vspace{-.05 in}
\section{Learning Guarantees for an Arbitrary Encoder}
\vspace{-.05 in}

We now present generalization bounds for an arbitrary encoder. But before that, we need to clarify which encoder we are talking about. 
In practice, people don't directly pretrain $\hPsi$ during upstream.
Instead, the common practice is to first pretrain an encoder $\hPhi = [\hat{\phi}_1,\cdots,\hat{\phi}_d]$ on the augmented space with $\hat{\phi}_i \in \lap$, and then transform it into $\hPsi$, on top of which a linear probe is learned downstream.
For example, 
in contrastive learning, the encoder is trained on views of images instead of original images;
in masked language modeling, the encoder is trained on masked sentences instead of full sentences.
In practice, people usually apply $\hPhi$ to downstream tasks directly without transformation,
but for theoretical analysis we need to explicitly write out the transformation since $\hPsi$ and $\hPhi$ work on different spaces.
We consider the commonly used \textit{average encoder} \citep[Eqn.~(4)]{saunshi2022understanding}:
\vspace{-.05 in}
\begin{equation}
\label{eqn:avg-enc}
    \hPsi(x) = \E[\hPhi(A) | x] = \int \hPhi(a) p(a|x) da, \vspace{-.05 in}
\end{equation}

which is equivalent to $\hPsi = \tax^* \hPhi$, 
and thus $\hat{\psi}_i \in R(\tax^*) = \laph$ for all $i \in [d]$.
In this section, $\hPhi$ can be an arbitrary function, even with infinite dimensions, so that the model can have any architecture and size; whereas $\hPsi$ is always the average encoder of $\hPhi$.
Moreover, we consider the following predictor:
\begin{boxed}
\begin{defi}
\label{def:least-squares}
The final predictor is the \textbf{nonparametric least-squares estimate} defined as
\begin{equation}
\vspace{-.05 in}
\label{eqn:least-squares}
    \hat{f} := \argmin_{ f: f = w^{\top} \hat{\Psi} \in \learnh,  \|f\|_{\laph} \le \frac{B}{\sqrt{1-\epsilon}} } \set{\frac{1}{n} \sum_{k=1}^n (\tilde{y}_k - f(\tilde{x}_k))^2}  .
\end{equation}
\end{defi}
\end{boxed}

Here, note that by Eqn. (\ref{eqn:lap-equ}) and $\| f^* \|_{\px} \le B$, there is $\|f^*\|_{\laph} \le \frac{B}{\sqrt{1-\epsilon}}$.

We next introduce two key ingredients crucial to our analyses.
The first ingredient is what we term the \emph{augmentation complexity},
which uniformly bounds the kernel and is
typically required in RKHS generalization analyses; for instance, see \citet[Section~12.1.3]{scholkopf2002learning}).
\vspace{-.05 in}
\begin{boxed}
\begin{defi}
\label{def:bt}
Define the \textbf{augmentation complexity} as $\bt := \| K_X \|_\infty^{1/2}$, \ie{} for $\px$-almost all $x$,
\begin{equation*}
    \kx(x,x)= \sum_i \lambda_i \psi_i(x)^2 = \int \frac{p(a|x)^2}{ \pa (a)} da  = D_{\chi^2} (\pa(\cdot | x) \parallel \pa) + 1 \le \bt^2  . \vspace{-.05 in}
\end{equation*}
\end{defi}
\end{boxed}
Here, $D_{\chi^2}(P \parallel Q) := \int (\frac{dP}{dQ} - 1)^2 dQ$ is the $\chi^2$-divergence.
Let $S_\lambda(d) := \sum_{i=1}^d \lambda_i$ and $S_\lambda := S_\lambda(\infty) = \sum_{i=1}^\infty \lambda_i$. 
\citet{wang2022spectral} showed that $S_\lambda = D_{\chi^2}(P_{\gA |\gX} \parallel P_\gA) + 1$.
By convexity of $D_{\chi^2}$, we have $S_\lambda \le \bt^2$, 
\ie{} $\tax^* \tax$ is a trace-class operator (or $\bt^2 \ge \int K_X(x,x) d \px (x) = S_\lambda$). Thus, $\bt \ge 1$ as $\lambda_1 = 1$.
We will provide examples of this augmentation complexity in \cref{sec:bt-big}.

Our second ingredient is the \textit{trace gap} that captures the quality of the encoder.
It is based on the notion of the \textit{ratio trace}.
Given an encoder $\hPhi$,
without loss of generality, suppose it is full-rank. \vspace{-.05 in}
\begin{boxed}
\begin{defi}
\label{def:rt}
    Define covariance matrices $\mF, \mG$ as
    $\mF(i,j) = \langle \hat{\psi}_i, \hat{\psi}_j \rangle_\px = \langle \tax^* \hat{\phi}_i, \tax^* \hat{\phi}_j \rangle_\px$
    and 
    $\mG(i,j) = \langle \hat{\phi}_i, \hat{\phi}_j \rangle_\pa$. 
Then, the \textbf{ratio trace} is defined as $\Tr(\mG^{-1} \mF)$, if $\mG^{-1}$ is well-defined.
\end{defi}
\end{boxed}
Ratio trace is a classical quantity in linear discriminant analysis (LDA) \citep{wang2007trace} and, as we will show, controls the approximation error.
The largest ratio trace of any $d$-dimensional $\hat{\Phi}$ is $\lambda_1+\cdots+\lambda_d$, and can be achieved by the top-$d$ eigenspace of $\laph$.
Then, define the \textit{learned kernel} as
\begin{equation}
\label{eq:Khat}
    \learnk(x,x') = \langle \tax^* (\mG^{-1/2} \hPhi)(x), \tax^* (\mG^{-1/2} \hPhi)(x') \rangle  ,
\end{equation}
which is the reproducing kernel of $\learnh = \textnormal{span}(\hat{\psi}_1,\hat{\psi}_2,\cdots)$, a subspace of $\laph$.
Here $\mG^{-1/2}$ is used for normalization.
The ratio trace can be viewed as the trace of $\learnh$.
Then, define the \textbf{\textit{trace gap}} as: \vspace{-.05 in}
\begin{equation}
    \tau^2 := \inf_{d' \le d} \inf_{h_1,\cdots,h_{d'}} S_\lambda(d'+1) - \Tr(\mG_h^{-1} \mF_h), \vspace{-.05 in}
\end{equation}
    where $\tau \ge 0$,
    $h_i = w_i^{\top} \hat{\Phi}$,
    $\mG_h = \left ( \langle h_i, h_j \rangle_\pa \right )_{i,j \in [d']}$,
    and $\mF_h = \left ( \langle \tax^* h_i, \tax^* h_j \rangle_\px \right )_{i,j \in [d']}$.
    Note that for any $d' \le d$ there is $\Tr(\mG_h^{-1} \mF_h) \le S_{\lambda}(d')$, so $\tau^2$ is always lower bounded by $\lambda_{d+1}$.
    And by choosing $h_i = \hat{\phi}_i$ for $i \in [d]$,
    we can see that $\tau^2 \le S_\lambda(d+1) - \Tr(\mG^{-1} \mF)$.

We now state our first result.
For simplicity, we assume that the random noise $\nse_1,\cdots,\nse_n$
are \iid{} $\gN(0,\sigma^2)$ variates for some $\sigma > 0$,
though this can be relaxed \citep[Chapter~13]{wainwright2019high}. \vspace{-.05 in}
\begin{boxed}
\begin{thm}
\label{thm:main-2}
    Let $\nse_1,\cdots,\nse_n$ be \iid{} $\gN(0,\sigma^2)$ variates.
    Let $\hat{f}$ be defined by \cref{def:least-squares}.
    If $\hat{\Phi}$ has $d$ dimensions and $\tau < 1$ ($d$ can be $\infty$),
    then there are universal constants $c_0,c_1,c_2$ such that
    with probability at least $1 - c_1 \exp \left ( -\frac{c_2 \sqrt{2n S_\lambda(d+1)}}{\bt} \right ) - \exp \left ( - \sqrt{\frac{2 n \bt^2 B^2 }{1 - \epsilon}}  \right )$, there is
    \begin{equation*}
    \| \hat{f} - f^* \|_\px^2 \le \frac{9 \tau^2 (\tau + \epsilon) B^2}{(1 - \tau^2) (1 - \epsilon)} + \frac{c_0 \bt (B^2 + \sigma B)}{1-\epsilon} \sqrt{\frac{S_\lambda(d+1)}{n} }  \quad \text{for all }  f^* \in \Bdowne  .
    \end{equation*}
\end{thm}
\end{boxed}
\begin{remark}
The first term in the bound controls the approximation error entailed by the limited capacity of the $d$-dimensional encoder $\hat{\Phi}$.
This term may not vanish as $N, n \rightarrow \infty$, since $\tau^2$ is lower bounded by $\lambda_{d+1}$ which can be positive.
For instance, if $d$ is too small for $\hat{\Phi}$ to represent $f^*$, then the approximation error cannot be zero regardless of the number of samples available.
Nevertheless, we can show that with a proper pretraining algorithm, the trace gap $\tau^2$ can be very close to $\lambda_{d+1}$ as $N \rightarrow \infty$ (\cref{thm:main-3}).
A larger ratio trace leads to a smaller approximation error, and indeed many existing objectives are strongly connected to maximizing $\Tr(\mG^{-1} \mF)$ (see Appendix \ref{app:objectives}).
The second term bounds the downstream estimation error, which vanishes as $n \rightarrow \infty$.
Finally, this result requires full access to $p(a|x)$ for any $x$,
which is available in theory since the augmentation scheme is of our choice.
Practical considerations for estimating $p(a|x)$ is outside the scope of this work. 
\end{remark}

\ding{226} \textbf{Proof Sketch of Theorem \ref{thm:main-2}:}
This result follows from two bounds:
\begin{lem}[Estimation error bound]
\label{thm:main}
    Suppose $\nse_1,\cdots,\nse_n$ are \iid{} $\gN(0,\sigma^2)$ variates.
    If $\hat{\Phi}$ has $d$ dimensions ($d$ can be $\infty$),
    then we have the following uniform bound
    over all $f^* = \tax^* g^* \in \Bdowne$: \vspace{-.05 in}
    \begin{equation*}
    \begin{aligned}
        & \underset{\tilde{x}_i,\nse_i}{\sP} \left [ \forall f^* \in \Bdowne, \| \hat{f} - f^* \|_\px^2 \le  9 \|  \fproj - f^* \|_\px^2 + \frac{c_0 \bt (B^2 + \sigma B)}{1-\epsilon} \sqrt{\frac{S_\lambda(d+1)}{n} } \right ] \\ 
        \ge \; & 1 - c_1 \exp \left ( -\frac{c_2 \sqrt{2n S_\lambda(d+1)}}{\bt} \right ) - \exp \left ( - \sqrt{\frac{2 n \bt^2 B^2 }{1 - \epsilon}}  \right )  , \vspace{-.05 in}
    \end{aligned}
    \end{equation*}
    where $\fproj = \tax^*(\Pi_{\hPhi} g^*)$ is the projection of $f^*$ onto $\learnh$ \wrt{} $\langle \cdot, \cdot \rangle_\laph$,
    and $c_0,c_1, c_2$ are universal constants.
    Note that $\hat{f}$ depends on $f^*$.
    Moreover, $S_\lambda(d+1) \le \min \set{d+1, \bt^2}$.
\end{lem}
\begin{remark}
The constant $9$ in the bound is loose
and can be tightened arbitrarily close to 1 (at the cost of increasing other terms in the bound) by straightforward modifications to the proof, which we omit here.
Such a greater-than-one constant is inevitable for any uniform deviation bound.
\end{remark}

\begin{lem}[Approximation error, upper bound]
\label{lem:f_diff}
   If $\tau < 1$, then for any $f^* \in \Bdowne$, there is \vspace{-.05 in}
    \begin{equation}
    \label{eqn:f_diff}
       \|\fproj-f^*\|_\px^2 \le   \frac{\tau^2 }{1 - \tau^2} \frac{\tau + \epsilon}{1-\epsilon} B^2 .
    \end{equation}
\end{lem}

\vspace{-.1 in}
\section{Analyses for the Near-Optimal Encoder}
\vspace{-.05 in}
\label{sec:top-d-main}

In this section, we consider the special case where $\hat{\Psi}$ is near-optimal, \ie{} the Monte-Carlo approximation of the optimal $d$-dimensional encoder that minimizes the worst-case approximation error.
First, we show that the optimal encoder spans the \textit{top-$d$ eigenspace}, which is the linear span of the top-$d$ eigenfunctions $\psi_1,\cdots, \psi_d$.
Define the \textit{worst-case approximation error} over $\Bdowne$ as: \vspace{-.05 in}
\begin{equation*}
    \err(\hPsi; \Bdowne) := \sup_{f \in \Bdowne} \min_{w \in \R^d} \err(w^{\top} \hPsi, f) = \sup_{f \in \Bdowne} \min_{w \in \R^d} \| w^{\top} \hPsi - f \|_\px^2   .
\end{equation*}
\begin{prop}[Approximation error, lower bound]
\label{prop:reg}
   For any $\hPsi=[\hat{\psi}_1,\cdots,\hat{\psi}_d]$ where $\hat{\psi}_i \in \lxp$, 
    \begin{equation*}
     \err(\hPsi; \Bdowne) \ge \frac{\lambda_{d+1}}{1 - \lambda_{d+1}} \frac{\epsilon}{1 - \epsilon}  B^2 \quad \text{given that} \quad  \frac{\lambda_{d+1}}{1 - \lambda_{d+1}}  \frac{\epsilon}{1 - \epsilon} \le \frac{1}{2}  .
    \end{equation*}
    To attain equality, it is sufficient for $\hPsi$ to span the top-$d$ eigenspace, and also necessary if $\lambda_{d+1} < \lambda_d$. \vspace{-.25 in}
\end{prop}
One can use kernel PCA to recover top-$d$ eigenspace \citep{belkin2003laplacian,johnson2022contrastive}, which is however not scalable. Alternatively, one can use variational objectives that attain minimum when $\hat{\Phi}$ spans the top-$d$ eigenspace (see \cref{app:objectives} for examples).
Also, when $\lambda_d > 0$, the top-$d$ eigenspace lies within $R(\tax^*)$, which is the function class of $\Psi$ when using the average encoder.

The optimal $d$-dimensional encoder is inaccessible since $\psi_1,\cdots,\psi_d$ cannot be extracted with only finite samples.
What one can do instead is extracting the \textit{empirical top-$d$ eigenspace} $\ephd$.
Given samples $x_1,\cdots,x_N$,
define two Hilbert spaces $\lhxp$ and $\lhap$, where $\langle f_1, f_2 \rangle_\hpx = \frac{1}{N} \sum_{k=1}^N f_1(x_k) f_2(x_k)$, and $\langle g_1, g_2 \rangle_\hpa = \int g_1(a) g_2(a) d \hpa(a)$ for $\hpa(a) = \frac{1}{N} \sum_{k=1}^N p(a|x_k)$.
With these finite samples, $\tax^*$ will remain the same, but $\tax$ will become the empirical version $\bar{\tax}$: \vspace{-.05 in}
\begin{equation*}
    (\bar{\tax} f)(a) = \frac{1}{N} \sum_{k=1}^N \frac{f(x_k) p(a|x_k)}{\hpa(a)}  .
\end{equation*}
Let the eigenvalues and eigenfunctions of $\tax^* \bar{\tax}$ be $\{ (\bar{\lambda}_i, \bar{\psi}_i) \}$.
Let $\bar{\phi}_i$ be the eigenfunctions of $\bar{\tax} \tax^*$.
$\bar{\phi}_i$ and $\bar{\psi}_i$ have the duality property similar to Proposition \ref{prop:dual-2}.
Let $\eph$ be the RKHS associated with $\tax^* \bar{\tax}$, which is the span of $\set{ \bar{\psi}_i}_{\bar{\lambda}_i > 0}$ and hence is a subspace of $\laph = R(\tax^*)$.
Let $\ephd$ be the top-$d$ eigenspace of $\eph$.
Let $\lambda_{\max}(\mG)$, $\lambda_{\min}(\mG)$ be the largest and smallest eigenvalue of $\mG$.
For the case of interest where $\learnh = \ephd$,
we have our second main learning guarantee that bounds the trace gap:
\begin{boxed}
\begin{thm}
\label{thm:main-3}
Suppose $\hat{\phi}_i = \bar{\phi}_i$ for $i \in [d]$.
Let $\gamma_{\mG} := \lambda_{\max}(\mG) / \lambda_{\min}(\mG)$
be the condition number of $\mG$.
    Then, for any $\delta > 0$, it holds with probability at least $1 - \delta$ that \vspace{-.05 in}
\begin{equation*}
    \tau^2 \le S_{\lambda}(d+1) - \Tr(\mG^{-1} \mF) \le \lambda_{d+1} +  \left (2 + \sqrt{2 \log \frac{2}{\delta}} \right ) \frac{(\lambda_d^{-1} + \bar{\lambda}_d^{-1} \gamma_{\mG}^{1/2} + 2)\bt^2}{\sqrt{N}} d  .
\end{equation*}
\end{thm}
\end{boxed}
\begin{remark}
Combining this result with the first main Theorem \ref{thm:main-2} leads to the bound for $\learnh = \ephd$.
Since $\tau^2 \ge \lambda_{d+1}$,
this result says that the gap between $\tau^2$ and its optimal value is $O(N^{-1/2})$, \rebuttal{ignoring potential dependence of $\gamma_{\mG}$ and $\bar{\lambda}_d^{-1}$ on $N$.}
Comparing the upper bound in \cref{lem:f_diff} + \cref{thm:main-3} to the lower bound in \cref{prop:reg}, we note that the upper bound is near tight:
The only difference is that in \cref{eqn:f_diff}, we have $\frac{\tau + \epsilon}{1 - \epsilon}$ instead of $\frac{\epsilon}{1 - \epsilon}$.
Unlike the bounds in \citet{haochen2021provable,saunshi2022understanding},
this bound only depends on $\lambda_d^{-1}$, $\bar{\lambda}_d^{-1}$ and $\gamma_{\mG}$, but not $(\lambda_{d'} - \lambda_{d+1})^{-1}$ for some $d' \le d$,
\ie{} it does not require the separability of the top-$d$ eigenfunctions,
because it does not need the top-$d$ eigenspace to be close to the empirical top-$d$ eigenspace.
Instead, it only requires the traces of the two top-$d$ eigenspaces to be close.
This is a big improvement since the eigenvalues can have high multiplicity as pointed out by \citet{saunshi2022understanding}. 
\end{remark}

\ding{226} \textbf{Proof Sketch of Theorem \ref{thm:main-3}:}
Denote $S_{\bar{\lambda}}(d) := \sum_{i=1}^{d} \bar{\lambda}_i$.
We define two empirical covariance matrices $\mhF$ and $\mhG$ as: $\mhF(i,j) = \langle \hat{\psi}_i, \hat{\psi}_j \rangle_\hpx$,
and $\mhG(i,j) = \langle \hat{\phi}_i, \hat{\phi}_j \rangle_\hpa$.
Since any invertible linear transformation to $\hat{\Phi}$ does not change $\Tr(\mG^{-1} \mF)$ or the empirical ratio trace $\Tr(\mhG^{-1} \mhF)$,
we can see that for $\learnh = \ephd$ there is $\Tr(\mhG^{-1} \mhF) = S_{\bar{\lambda}}(d)$ (simply consider $\hat{\phi}_i = \bar{\phi}_i$).
Thus, it suffices to bound $| \Tr(\mhG^{-1} \mhF) - \Tr(\mG^{-1} \mF) | $, and $S_\lambda(d) - S_{\bar{\lambda}}(d)$. \vspace{-.05 in}

We start with bounding the gap between empirical and real ratio traces for any encoder $\hat{\Phi}$:

\begin{lem}
\label{lem:ratio-trace-gen}
Suppose there exists a constant $C > 0$ such that $\E_\pa[g^4] \le C^2 \| g \|_\pa^2$, for all $g = w^{\top} \hat{\Phi}$ where $\|g\|_\pa \le 1$.
Then, for any $\delta > 0$, it holds with probability at least $1 - \delta$ that 
\begin{equation}
\label{eqn:ratio-trace-gen}
| \Tr(\mhG^{-1} \mhF) - \Tr(\mG^{-1} \mF) | \le \left (2 + \sqrt{2 \log \frac{2}{\delta}} \right ) \frac{ C \bt + \bt^2}{\sqrt{N}} d  .
\end{equation}
\end{lem}
Lemma~\ref{lem:ratio-trace-gen} considers a general $\hat{\Phi}$, and requires a fourth-moment control assumption \citep[Eqn.~(14.22a)]{wainwright2019high}.
For the specific case $\hat{\phi}_i = \bar{\phi}_i$ studied in this section, we can prove that this assumption holds (that is why it does not appear in \cref{thm:main-3}),
\ie{} the top-$d$ empirical eigenfunctions can be proved to be \textit{delocalized} \citep{erdHos2009local}.
In particular, we prove the following:
\begin{lem}
\label{lem:lambda-diff}
    Suppose $\hat{\phi}_i = \bar{\phi}_i$ for $i \in [d]$.
    Let $\gamma_{\mG} := \lambda_{\max}(\mG) / \lambda_{\min}(\mG)$,
    which is the condition number of $\mG$.
    Then, for any $\delta > 0$, both \vspace{-.05 in}
    \begin{equation*}
        \sum_{j=1}^{d} \bar{\lambda}_j \ge \sum_{i=1}^{d} \lambda_i - \left (2 + \sqrt{2 \log \frac{2}{\delta}} \right ) \frac{ (\lambda_d^{-1}+1) \bt^2 }{\sqrt{N}} d
    \end{equation*}
    and Eqn. (\ref{eqn:ratio-trace-gen}) with $C = \bt \bar{\lambda}_d^{-1} \gamma_{\mG}^{1/2}$ hold simultaneously for $\learnh = \ephd$ with probability at least $1 - \delta$.
\end{lem}

\vspace{-.08 in}
\section{Estimating and Exploiting the Augmentation Complexity}
\label{sec:bt-big}
\vspace{-.08 in}
An important quantity from our analysis is the augmentation complexity $\bt$ --- both approximation and estimation error bounds get smaller when $\bt$ is reduced.
A natural way to reduce $\bt$ is via a stronger augmentation,
which has indeed been helpful in practice~\citep{chen2020simple,wettig2022should}.

In this section, we take a closer look at $\bt$ of mask-type augmentations.
We study $\bt$ of different masking schemes on the hypercube data model introduced in \citet{saunshi2022understanding},
and show that $\bt$ depends on both the mask ratio (augmentation strength) and the masking strategy.
Then, we empirically estimate $\bt$ on the real-world NLP dataset \texttt{wikipedia-simple},
and demonstrate its correlation with the generalization gap on real downstream tasks.
Our results indicate that the augmentation complexity $\bt$ offers a means to compare different augmentations quantitatively and implies a tradeoff in the choice of augmentations.

\vspace{-.04 in}
\subsection{Theoretical Hypercube Data Model}
\label{subsec:hypercube}
\vspace{-.04 in}
\begin{wrapfigure}{r}{0.47\textwidth}
\vskip -.5in
    \centering
    \input{example-demo}
\vskip -.24in
    \caption{Augmentation illustration.}
    \label{fig:demo-aug}
\vskip -.2in
\end{wrapfigure}
Like \citet{cabannes2023ssl}, we study three random masking methods (\cref{fig:demo-aug}): (i) Independent random masking, (ii) Cutout-like block masking, and (iii) BERT-like masking.
Let $\alpha$ denote the mask rate.
Consider the data space $\gX = \{-1, 1 \}^{d_\gX}$ with $\px$ being the uniform distribution over $\gX$.
Denote the corresponding $\bt$ of these three methods to be $\bt_r, \bt_c$, and $\bt_b$.
We will show that these $\bt$'s will be very different even with the same $\alpha$.

Let's first consider the independently random masking (derivations are deferred to \cref{sec:proofs_hypercube_examples}):
\begin{example}
Consider a \emph{random masking} augmentation, \ie{} for any $x \in \gX$, each coordinate $x^{(i)}$ is randomly and independently masked to be $0$ (\ie{} $0$ denotes the \texttt{[MASK]} token)
with probability $\alpha \in (0,1)$ .
Then, its augmentation complexity is given by $\bt_r^2 = (2 - \alpha)^{d_\gX}$.
\end{example}
\vspace{-.03 in}
Next, consider the Cutout-like block masking \citep{devries2017improved}, which has been successfully applied in practice such as ViT \citep{dosovitskiy2021an} and MAE \citet{he2022masked}.
Compared to the previous example, Cutout-like block masking leads to a smaller $\bt$ with the same $\alpha$:
\begin{example}
    Consider \emph{random block masking},
    which masks $x^{(i)},x^{(i+1)},\cdots,x^{(i+r-1)}$ for $r = \lceil \alpha d_{\gX} \rceil$ and a uniformly random $i \in [d_{\gX}-r]$, for any $x \in \gX$.
    Then, $\bt_c^2 \le [2^{(1-\alpha) }]^{d_{\gX}}$.
    \end{example}

Our third and final example resembles the 80-10-10 strategy in BERT \citep{devlin2018bert}:
First randomly mask some tokens, and then randomly replace some unmasked tokens with other tokens.
\begin{example}
Consider \emph{random block masking with flipping}, where for any $x \in \gX$,
first mask $x^{(i)},\cdots,x^{(i+r-1)}$ to be $0$ 
for $r = \lceil \alpha d_{\gX} \rceil$ and a uniformly random $i \in [d_{\gX}-r]$,
then randomly flip the sign of each remaining coordinate
w.p. $\frac{\alpha}{2}$ independently.
Then, $\bt_b^2 \le \left [(\alpha^2-2\alpha+2)^{(1-\alpha/2)} \right ]^{d_{\gX}}$.
\end{example}
While $\bt_r,\bt_c,\bt_b$ all have an exponential dependency on $d_{\gX}$, their bases are different, as shown in Figure \ref{subfig:base}.
This means these methods have different $\bt$ despite having the same mask rate,
which highlights the importance of the masking strategy. \vspace{-.08 in}

\paragraph{On the exponentiality of $\bt$ in $d_{\gX}$.}
We suspect the exponential dependency on $d_\gX$ is a manifestation of the typical curse of dimensionality in high-dimensional statistics in the absence of strong inductive bias.
\citet[Section~3.2]{bengio2013representation} had a discussion on this point.
What our analysis makes salient is that the inductive bias can be expressed via controlled augmentation complexity.
One approach to obtain a polynomial augmentation complexity is by using very strong augmentations, such as using ImageNet pretraining where $\gA$ is a finite set of size $1000$, \ie{} the ImageNet label set. However, such simple strong augmentations usually lead to substantially worse performance in practice. We posit developing more interesting augmentations with bounded augmentation complexity as an important open problem facing the field of representation learning.

We postulate moreover that despite the worst-case bounds above come with exponential dependency on data dimension, the empirical success of existing augmentation-based self-supervised learning suggest that they implicitly adapt to the inherent low-dimensional manifold structure in real-world data.
We conjecture that as long as the augmentation captures such a structure, it can evade the curse.

\vspace{-.05 in}
\subsection{Augmentations on Real Datasets}
\label{sec:exp}
\vspace{-.04 in}

The augmentation complexity $\bt$ can be estimated on real datasets and used to quantitatively analyze and compare augmentations, 
which we demonstrate on the NLP dataset \texttt{wikipedia-simple}.
We study masked language modeling,
where $x$ is a full sentence and $a$ is a masked sentence.
Recall that $\bt^2$ upper bounds $\int \frac{p(x|a)}{p(x)} p(a|x) da$.
Empirically, we replace the integration with sample average;
namely, for $x = [x^{(1)}, \cdots, x^{(l)}]$ where $x^{(i)}$ is the $i^{th}$ token, we have
\begin{equation*}
    \log p \left (x \middle \vert a \right ) = \log p \left (x^{(1)} \middle \vert a \right ) + \log p \left (x^{(2)} \middle \vert a,x^{(1)} \right ) + \cdots + \log p \left (x^{(l)} \middle \vert a, x^{(1)}, \cdots, x^{(l-1)} \right ) .
\end{equation*}

We leverage a bi-directional MLM such as a BERT,
and compute $p(x^{(i)}|a,x^{(<i)})$ auto-regressively:
For any $i \in [l]$,
compute $p(x^{(i)}|a,x^{(<i)})$ with the BERT,
and then replace $a^{(i)}$ with $x^{(i)}$ for $i+1$.
We compute this for a random subset of samples.
We report the $99^{\text{th}}$ percentile of $\kx(x,x)$ instead of $\sup_x \kx(x,x)$ due to the presence of outliers.
Experimental details are deferred to Appendix \ref{app:exp}.  \vspace{-.05 in}

\paragraph{Estimating $\bt$.}
We study Random masking, Block masking, Random masking + Flipping, and Block masking + Flipping.
For Masking + Flipping, we randomly mask $\alpha/2$ of the tokens, and replace another $\alpha/2$ of the tokens with random tokens;
this replace rate is much higher than standard practice, since we want to highlight the effect of flipping.
Figure \ref{subfig:kappa} plots the estimated $\log \bt^2$ of the four augmentations \wrt{} the mask rate $\alpha$,
where $\log \bt^2$ is calculated as the average of five runs with different random seeds.
The complexity drops as $\alpha$ increases as predicted by theory.
One observation is that the ``Random + Flip'' curve intersects with ``Block'' and ``Block + Flip'',
suggesting that block masking has a stronger effect when $\alpha$ is small, whereas flipping has a stronger effect when $\alpha$ is large.  \vspace{-.05 in}

\begin{figure}[!t]
\vskip -.5in
\centering
\begin{subfigure}[b]{.26\textwidth}
    \input{base}
    \vskip -.22in
    \caption{Bases of three $\bt^2$.}
    \label{subfig:base}
\end{subfigure}
\begin{subfigure}[b]{.29\textwidth}
    \input{kappa-plot}
    \vskip -.07in
    \caption{Estimated $\log \bt^2$.}
    \label{subfig:kappa}
\end{subfigure}
\begin{subfigure}[b]{.41\textwidth}
    \input{qnli}
    \vskip -.07in
    \caption{Results on QNLI and SST-2.}
    \label{subfig:qnli}
\end{subfigure}
\vskip -.1in
\caption{Plots for Section \ref{sec:bt-big}. In (b), $\log \bt^2$ is estimated on \texttt{wikipedia-simple}.}
\label{fig:kappa}
\vskip -.23in
\end{figure}

\paragraph{Downstream performance.}
We also study how the mask ratio $\alpha$ affects the downstream performance, using QNLI~\citep{wang2018glue} and SST-2~\citep{SST2}.
For pretraining, we train \texttt{roberta-large} models with different mask ratios using the fast pretraining recipe in \citet{wettig2022should}.
We focus on Random masking and do not apply the 80-10-10 strategy.
For downstream, we fine-tune the encoder together with the linear head following common practice.
To align better with our theory, we explicitly use the average encoder \cref{eqn:avg-enc},
estimated by sampling 16 $a$'s for each $x$. \vspace{-.04 in}

We evaluate the train/test accuracies of the models, and plot the test accuracy (blue solid) and the train-test accuracy gap (green dashed) in Figure \ref{subfig:qnli}.
The highest test accuracy is achieved at $\alpha = 0.15$ on QNLI and at $\alpha = 0.40$ on SST-2 (marked in red).
The test accuracy is low when $\alpha$ is too small due to the large generalization gap, and also low when $\alpha$ is too large due to low training accuracy.
Regarding the train-test gap,
QNLI shows a monotonic decrease in the gap as the mask ratio grows, 
but the gap on SST-2 is U-shaped, with the lowest point at $\alpha = 0.40$.
This is likely because with $\alpha > 0.40$ is too strong an augmentation for SST-2 that Assumption \ref{ass:lap} is broken, in which case our theoretical results will not hold.
Thus, these results align with our theory that while augmentations should be sufficiently robust, they must not be so strong as to violate Assumption \ref{ass:lap}.
This suggests the presence of a ``sweet spot'', which is also supported by evidence in prior work~\citep{tian2020makes}.

\vspace{-.1 in}
\section{Conclusion}
\label{sec:conclusion}
\vspace{-.08 in}
This work establishes an RKHS approximation/regression framework that completely disentangles the effects of the model and the augmentation.
Our framework (i) clarifies the roles of the two stages of representation learning, as well as the role of augmentation with the isometry property \cref{eqn:lap-equ};
(ii) leads to nonparametric statistical guarantees free of model complexity or inductive bias, so that they are valid for an \textit{arbitrary} encoder; 
and (iii) formulates the augmentation complexity $\bt$ as a tool to quantitatively analyze the effect of augmentations.
We also refer our readers to two recent works that used similar techniques but derived quite different results:
\citet{wang2022spectral} studied an RKHS similar to $\laph$, but in a completely different context (they studied conditional moment models);
\citet{cabannes2023ssl} studied a similar problem, but they used a pre-defined kernel to define the function class and derived guarantees, while the RKHS $\laph$ in this work is completely induced by the data augmentation.
We defer discussions for more related work to Appendix \ref{app:related-work}. \vspace{-.05 in}

\paragraph{Limitations and open problems.}
The major limitation of this work is that $\bt$ can be exponential in $d_{\gX}$, which as we pointed out could be a natural consequence of the curse of dimensionality, and that the low-dimensional structure of real-world data could be a potential reason of circumventing the exponential dependency.
Another limitation is that we studied the average encoder, while usually people directly apply the original encoder to downstream.
The latter case requires the model to perform some type of implicit transformation as the upstream and downstream have different input spaces ($\gA$ and $\gX$),
and this transformation cannot be studied within the architecture-free framework of this work.
Additionally, an open problem is how to deal with the common upstream-to-downstream data distribution shift in practice, for which downstream fine-tuning could be critical.
This has been discussed in \citet{cabannes2023ssl}, but a more general analysis is desired.

%% file: framework.tex
\begin{tikzpicture}
\tikzstyle{every node}=[font=\small]











\draw[thick] (9-14.4,-0.1) rectangle ++(3.4,1.65);
\node[] at (11.8-14.4,1.8) {$\lxp$};
\draw[thick, violet] (10.7-14.4,0.7) ellipse (.65in and .28in);
\node[violet] at (12.15-14.4,1.3) {$\laph$};
\draw[orange, thick] (10.7-14.4,0.7) ellipse (.35in and .24in);
\node[orange] at (9.55-14.4,0.7) {$\eph$};
\draw[teal, thick] (10.7-14.4,0.5) ellipse (.25in and .14in);
\node[teal] at (10.3-14.4,1.03) {$\ephd$};

\node[circle,fill=black,inner sep=0pt,minimum size=3.9pt,label=above:{$f^*$}] (a1) at (12-14.4, 0.455) {};
\node[circle,fill=black,inner sep=0pt,minimum size=3.9pt,label=above:{$\fproj$}] (a2) at (11.05-14.4, 0.64) {};
\node[circle,fill=black,inner sep=0pt,minimum size=3.9pt,label=above:{$\hat{f}$}] (a3) at (10.7-14.4, 0.25) {};
\draw[dashed, blue, line width=.6mm] (a1) -- (a2);
\draw[dashed, red, line width=.6mm] (a2) -- (a3);

\node[draw, align=left, anchor=south west] at (-1.9,-0.125) {
$\lxp$: Large space of $L^2$ bounded functions, containing the target function $f^*$ \\
{\color{violet} Induced RKHS $\laph$}: Functions with soft invariance to augmentation, $f^* \in \laph$ \\ 
{\color{orange} Empirical RKHS $\eph$ }: Obtained with $N$ samples to approximate inaccessible $\laph$ \\
{{\color{teal} Empirical top-$d$ RKHS $\ephd$ }: Induced by the pretrained $d$-dimensional encoder } \\ 
{\color{blue} Approximation error $\| f^* - \fproj \|$ }: Distance from $f^*$ to $\ephd$ (\cref{lem:f_diff}, \cref{thm:main-3}) \\ 
{\color{red} Estimation error  $\| \fproj - \hat{f} \|$}: Entailed by downstream regression (\cref{thm:main})
};

\end{tikzpicture}

%% file: example-demo.tex
\begin{tikzpicture}[background rectangle/.style={fill=black!10}, show background rectangle,outer sep=0, inner sep=1pt]
\tikzstyle{every node}=[font=\small]
  \node[anchor=west] at (-0.0,0) {3. Block Mask + Flip};
  \node[anchor=west] at (-0.0,0.5) {2. Block Mask};
  \node[anchor=west] at (-0.0,1) {1. Random Mask};
  \node[anchor=west] at (-0.0,1.5) {Original Sample};

  \node[] at (3.2,1.5) {1};
  \node[] at (3.6,1.5) {1};
  \node[] at (4.0,1.5) {-1};
  \node[] at (4.4,1.5) {-1};
  \node[] at (4.8,1.5) {1};
  \node[] at (5.2,1.5) {-1};
  \node[] at (5.6,1.5) {1};
  \node[] at (6.0,1.5) {1};

  \node[draw,inner sep=2pt] at (3.2,1.0) {0};
  \node[draw,inner sep=2pt] at (3.6,1.0) {0};
  \node[] at (4.0,1.0) {-1};
  \node[draw,inner sep=2pt] at (4.4,1.0) {0};
  \node[] at (4.8,1.0) {1};
  \node[] at (5.2,1.0) {-1};
  \node[] at (5.6,1.0) {1};
  \node[draw,inner sep=2pt] at (6.0,1.0) {0};

  \node[] at (3.2,0.5) {1};
  \node[] at (3.6,0.5) {0};
  \node[] at (4.0,0.5) {0};
  \node[] at (4.4,0.5) {0};
  \node[] at (4.8,0.5) {0};
  \node[] at (5.2,0.5) {-1};
  \node[] at (5.6,0.5) {1};
  \node[] at (6.0,0.5) {1};

  \draw (3.45,0.7) -- (4.95,0.7) -- (4.95,0.3) -- (3.45,0.3) -- (3.45,0.7);

  \node[] at (3.2,0.0) {1};
  \node[] at (3.6,0.0) {0};
  \node[] at (4.0,0.0) {0};
  \node[] at (4.4,0.0) {-1};
  \node[] at (4.8,0.0) {1};
  \node[draw, red,rounded corners=1mm,inner sep=2pt] at (5.2,0.0) {1};
  \node[] at (5.6,0.0) {1};
  \node[draw, red,rounded corners=1mm,inner sep=2pt] at (6.0,0.0) {-1};

  \draw (3.45,0.2) -- (4.15,0.2) -- (4.15,-0.2) -- (3.45,-0.2) -- (3.45,0.2);

  \node[draw,inner sep=2pt] at (2, -0.55) {Masked};
  \node[draw, red,rounded corners=1mm,inner sep=2pt] at (4,-0.55) {Flipped};
\end{tikzpicture}

%% file: base.tex
\begin{tikzpicture}
\tikzstyle{every node}=[font=\small]
\begin{axis}[
label style={font=\tiny},
tick label style = {font=\tiny},
    xmin = 0, xmax = 1,
    ymin = 1, ymax = 2,
    xtick distance = 0.5,
    ytick distance = 0.5,
    xlabel={Mask Ratio $\alpha$},
    x label style={at={(axis description cs:0.5,-0.25)},anchor=north,inner sep=0},
    major grid style = {lightgray},
    minor grid style = {lightgray!25},
    width = 1.2\linewidth,
    height = .8\linewidth,
    legend style = {nodes={scale=0.5, transform shape},font={\tiny},at={(0.35,0.9)},anchor=west,inner sep=.5pt},
    legend cell align=left,
    ]
    \addplot[domain = 0:1,
        samples = 200,
        smooth,
        ultra thick,
        blue,
        dotted
    ] {2-x};
    \addlegendentry{$y=2-\alpha$};
    \addplot[domain = 0:1,
        samples = 200,
        smooth,
        very thick,
        red,
        dashed
    ] {2^(1-x)};   
    \addlegendentry{$y=2^{1 - \alpha}$};
    \addplot[domain = 0:1,
        samples = 200,
        smooth,
        very thick,
        teal,
    ] {(x^2-2*x+2)^(1-x/2)};
    \addlegendentry{$y=(\alpha^2 - 2\alpha + 2)^{1 - \frac{\alpha}{2}}$};
\end{axis}
\end{tikzpicture}

%% file: kappa-plot.tex
\begin{tikzpicture}
\tikzstyle{every node}=[font=\small]
\begin{axis}[
label style={font=\tiny},
tick label style = {font=\tiny},
    xmin = 0, xmax = 0.8,
    ymin = 80, ymax = 360,
    xtick distance = 0.2,
    ytick distance = 100,
    yticklabel shift = -2.5,
    xlabel={Mask Ratio $\alpha$},
    x label style={at={(axis description cs:0.5,-0.28)},anchor=north,inner sep=0},
    ylabel={$\log \bt^2$},
    y label style={at={(axis description cs:-0.35,0.2)},anchor=west,inner sep=0},
    width = .81\linewidth,
    height = .75\linewidth,
    legend style = {nodes={scale=0.5, transform shape},font={\tiny},at={(1.05,0.5)},anchor=west,inner sep=.5pt},
    legend cell align=left,
    ]
    \addplot[
        ultra thick,
        blue,
        dotted
    ] 
    coordinates {
    (0.05,353.9439)(0.10,348.45206)(0.15,346.39916)(0.20,339.02706)(0.30,325.07554)(0.40,307.0025)(0.50,278.59496)(0.60,250.01232)(0.70,214.14242)(0.80,160.24052)
    };
    \addlegendentry{Random};
    \addplot[
        thick,
        red,
        solid
    ] 
    coordinates {
    (0.05,350.98668)(0.10,349.13778)(0.15,342.77486)(0.20,335.01478)(0.30,312.20536)(0.40,280.7551)(0.50,241.58934)(0.60,198.1497)(0.70,151.57564)(0.80,84.51188)
    };
    \addlegendentry{Random + Flip};
    \addplot[
        very thick,
        olive,
        dashed
    ] 
    coordinates {
    (0.05,346.77186)(0.10,339.8619)(0.15,327.13244)(0.20,313.89546)(0.30,280.6536)(0.40,254.40494)(0.50,217.35124)(0.60,187.59812)(0.70,152.70328)(0.80,113.85798)
    };
    \addlegendentry{Block};
    \addplot[
        ultra thick,
        darkgray,
        dash dot
    ] 
    coordinates {
    (0.05,350.4642)(0.10,343.733)(0.15,336.8131)(0.20,322.86658)(0.30,302.60346)(0.40,279.2389)(0.50,240.72884)(0.60,205.4515)(0.70,164.48324)(0.80,108.44108)
    };
    \addlegendentry{Block + Flip};
\end{axis}
\end{tikzpicture}

%% file: qnli.tex
\begin{tikzpicture}
\tikzstyle{every node}=[font=\small]
\begin{axis}[
name=plot1,
label style={font=\tiny},
tick label style = {font=\tiny},
    xmin = 0, xmax = 0.8,
    ymin = 70, ymax = 90,
    xtick distance = 0.2,
    ytick distance = 10,
    xtick pos=left,
    yticklabel shift = -2.5,
    axis y line*=left,
    xlabel={Mask Ratio $\alpha$},
    x label style={at={(axis description cs:0.5,-0.24)},anchor=north,inner sep=0},
    ylabel={Test Accuracy},
    y label style={at={(axis description cs:-0.25,0.1)},anchor=west,inner sep=0,color=blue},
    width = .60\linewidth,
    height = .535\linewidth,
    ]
    \addplot[
        very thick,
        blue
    ] 
    coordinates {
    (0.05,85.5208)(0.10,85.4109)(0.15,86.8021)(0.20,86.4543)(0.30,86.3262)(0.40,86.2713)(0.50,84.7703)(0.60,82.354)(0.70,78.3452)(0.80,73.2748)
    };
    \node[circle,fill=red,inner sep=0pt,minimum size=3.9pt]  at (0.15,86.8021) {};
    \node[red, font=\tiny] at (0.44,88.3) {$(0.15,86.8)$};
\end{axis}
\begin{axis}[
label style={font=\tiny},
tick label style = {font=\tiny},
    xmin = 0, xmax = 0.8,
    ymin = -0.4, ymax = 4.4,
    ytick distance = 2,
    hide x axis,
    axis y line*=right,
    width = .60\linewidth,
    height = .535\linewidth,
    ]
    \addplot[
        ultra thick,
        teal,
        dashed
    ] 
    coordinates {
    (0.05,4.2055)(0.10,3.8285)(0.15,3.4483)(0.20,3.2481)(0.30,2.9752)(0.40,2.0678)(0.50,0.9194)(0.60,0.9785)(0.70,0.8668)(0.80,-0.386)
    };
\end{axis}

\node[magenta, font=\tiny] at (0.3,0.2) {QNLI};

\begin{axis}[
name=plot2,
at={($(plot1.east)+(20,0)$)},
anchor=west,
label style={font=\tiny},
tick label style = {font=\tiny},
    xmin = 0, xmax = 0.8,
    ymin = 80, ymax = 95,
    xtick distance = 0.2,
    ytick distance = 5,
    xtick pos=left,
    yticklabel shift = -2.5,
    axis y line*=left,
    xlabel={Mask Ratio $\alpha$},
    x label style={at={(axis description cs:0.5,-0.24)},anchor=north,inner sep=0},
    width = .60\linewidth,
    height = .535\linewidth,
    ]
    \addplot[
        very thick,
        blue
    ] 
    coordinates {
    (0.05,88.5321)(0.10,89.6789)(0.15,91.8578)(0.20,91.7431)(0.30,91.3991)(0.40,92.4312)(0.50,89.2202)(0.60,86.6972)(0.70,84.289)(0.80,80.1606)
    };
    \node[circle,fill=red,inner sep=0pt,minimum size=3.9pt]  at (0.40,92.4312) {};
    \node[red, font=\tiny] at (0.375,93.7) {$(0.40,92.4)$};
\end{axis}
\begin{axis}[
at={($(plot1.east)+(20,0)$)},
anchor=west,
label style={font=\tiny},
tick label style = {font=\tiny},
    xmin = 0, xmax = 0.8,
    ymin = 5.5, ymax = 12,
    xtick distance = 0.2,
    ytick distance = 3,
    yticklabel shift = -2.5,
    axis y line*=right,
    hide x axis,
    ylabel={Train-Test Gap},
    y label style={at={(axis description cs:1.24,0.05)},anchor=west,inner sep=0,color=teal},
    width = .60\linewidth,
    height = .535\linewidth,
    ]
    \addplot[
        ultra thick,
        teal,
        dashed
    ] 
    coordinates {
    (0.05,10.6408)(0.10,9.6411)(0.15,7.3939)(0.20,7.4581)(0.30,7.2631)(0.40,5.7573)(0.50,8.0032)(0.60,9.1022)(0.70,9.8549)(0.80,11.6106)
    };
\end{axis}
\node[magenta, font=\tiny] at (2.85,0.2) {SST-2};
\end{tikzpicture}

%% file: appendix.tex
\section{Discussion}
\label{app:discussion}

\paragraph{Why do we need Assumption \ref{ass:lap} and can it be relaxed?}
We first note that an assumption on the target function is necessary to establish any non-vacuous learning theory.
This is implied by the No Free Lunch Theorem \citep{wolpert1997no},
which states that if the target function is sampled uniformly at random from all possible target functions,
then any learner achieves the same performance as random guess.

Assumption \ref{ass:lap} implies an isometry property.
As pointed out by \citet{2023arXiv230405366G},
real-world datasets are highly structured, and the structures are largely shared across different domains.
Assumption \ref{ass:lap} and implied isometry property suggest that a well-designed augmentation is a key to capturing such a low-dimensional cross-domain structure.
Indeed, better augmentation has been the main driving force in the recent progress in contrastive learning and masked language modeling.
Our focus on augmentation is in stark contrast to prior work on the generalization of overparameterized models (see Appendix \ref{app:related-work}),
which can find the global minima fitting all samples,
yet cannot capture the invariant structure within the data and the target concept.

If we were to relax Assumption \ref{ass:lap} (\eg{} the target function $f^*$ can be outside but very close to $\laph$), additional assumptions on $f^*$ would be required for our proof with uniform deviation bounds to hold.
Namely, $f^*$ can no longer be \textit{any} $L^2$ function that is close to $\laph$.
For instance, we at least need to assume that $|f^*(x)|$ is \aew{} uniformly bounded by some constant $C$ which the final bound will depend on, or assume that $f^*$ belongs to some predefined RKHS similar to \citet{cabannes2023ssl}.
We feel that these additional assumptions would complicate our analysis and obscure the role of augmentation, 
hence we choose to assume $f^* \in \laph$ instead.

\paragraph{The main theorem allows the encoder to use any model. Does it mean that the inductive bias of the model is not useful as opposed to prior work?}
No, the inductive bias is still useful.
In fact, we believe that the model family and data augmentation are two key ingredients of modern representation learning.
The former has been studied by prior work \citep{saunshi2022understanding,haochen2022theoretical}, \rebuttal{while this work studies the latter and introduces a class of augmentation complexity based generalization bounds that is orthogonal to the class of classical model complexity based bounds, in the sense that the bounds are nonparametric and do not require any inductive biases from the model family.}
This new class of guarantees deepens our understanding because it clarifies the roles of these two ingredients.
However, a downside of our analysis is that it cannot leverage the geometric structure of the data.
For example, the data might be so good that it is linearly separable in the Euclidean space,
but our nonparametric framework cannot make use of this information.

The inductive bias of the model has at least two roles:
(i) It further reduces the complexity on top of the data augmentation, leading to even tighter bounds;
(ii) While we used the average encoder in our analysis,
in practice people usually directly apply the pretrained encoder to the data in the downstream task.
In this case, the model architecture has to implicitly perform some kind of transformation,
since the pretraining and downstream samples come from different spaces, in which case the inductive biases of the model will be helpful.

\rebuttal{
\paragraph{Comparing to the claim in \citet{saunshi2022understanding}.}
\citet{saunshi2022understanding} argued that there exist setups where any learning guarantees independent of the model inductive bias are necessarily vacuous.
They presented an example of a disjoint augmentation, such that any two different original samples $x_1$ and $x_2$ cannot be augmented to the same $a$.
In this case, our $K_X$ defined in \cref{eq:Ka} is zero everywhere, so apparently our analysis will not work.
The observation made by \citet{saunshi2022understanding} is that even for this augmentation, if there is a good model (feature map), then good generalization can still be achieved, and therefore they claimed that model inductive bias is irreplaceable. 

The analysis in this work is compatible with their observation.
Our analysis shows that
if the augmentation is good enough (such that \cref{ass:lap} is satisfied and the augmentation complexity is small enough, which is in fact a strong assumption), then we can achieve good generalization only with the augmentation and without any help of the model.
In other words, the disjoint augmentation example showcases a poor choice of augmentation, which is orthogonal to the emphasis of this paper.
}

\section{Related Work}
\label{app:related-work}
\subsection{Learning Guarantees for Big Models}

Big models with millions and even billions of parameters have been proved to work well on a wide variety of tasks.
However, their empirical success cannot be explained by
classical generalization bounds that use the VC-dimension or the Rademacher complexity.
This discrepancy was made clear by \citet{zhang2017understanding},
which shows that while enjoying good generalization on real tasks, modern neural networks have sufficient capacity to fit random labels.
This challenges the classical generalization theory, which appeals to the Occam's razor argument that big models cannot generalize well in general.
There is a large body of research on improving classical generalization bounds for neural networks and big models, and we highlight a few seminal works below.

\paragraph{Algorithmic stability.}
This line of research attributes the good generalization of complex models to the stability of the (random) training algorithm.
There is a folklore that one key reason why neural networks generalize so well is that they are trained by stochastic gradient descent (SGD).
Let $f(x;A,S)$ be the predictor trained by a random training algorithm $A$ on a finite training set $S$.
$A$ is called \textit{uniformly stable} \citep{kearns1997algorithmic,bousquet2002stability},
if $E_A|f(x;A,S) - f(x;A,S')|$ is uniformly bounded when $S$ and $S'$ only differ by one training sample.
$f$ is guaranteed to enjoy good generalization if $A$ is uniformly stable.
Then, \citet{hardt2016train} proved that SGD is uniformly stable under certain conditions.
There are also weaker notions of algorithmic stability \citep{mukherjee2006learning,shalev2010learnability}.
It has also been shown in \citet{li2018algorithmic,arora2019implicit} that gradient methods for overparameterized models provide an implicit regularization effect that helps generalization.
However, \citet{zhang2021understanding} challenged this line of work by
empirically showing that SGD on real neural networks is not very stable,
even with regularization.

\paragraph{PAC-Bayes and sharpness.}
Compared to algorithmic stability, the PAC-Bayes line of analysis attributes generalization to model stability, that is the empirical risk is insensitive to perturbations in model parameters.
If the model is stable, then the empirical risk is small within a large region around the global minima,
meaning that the risk landscape has low sharpness \citep{keskar2016large,neyshabur2017exploring} around the global minima.
The PAC-Bayes generalization bound \citet{mcallester1999pac,langford2001bounds} considers a prior distribution $P$ over the model's function space (usually Gaussian or uniform),
and a posterior stationary distribution $Q$ found by a training algorithm.
If $D_{KL}(Q \parallel P)$ is small,
then good generalization is guaranteed.
Thus, the more models there are in the support of $Q$,
the better the generalization.
So we can see that a lower sharpness around the global minima leads to better generalization.
Sharpness-based bounds are shown to be empirically tighter than other bounds in \citet{jiang2019fantastic},
and inspire a class of training method called sharpness-aware minimization (SAM) \citep{foret2021sharpnessaware,bartlett2022dynamics,wen2022does}.
Other generalization guarantees derived from the PAC-Bayes analysis with similar ideas include low spectral norm \citep{bartlett2017spectrally,neyshabur2017pac}, flat minima \citep{hochreiter1997flat}, and derandomization \citep{negrea2020defense}.

One prominent class of analysis within the PAC-Bayes paradigm uses \textit{compression} to obtain tighter, and empirically non-vacuous generalization bounds.
For instance, \citet{dziugaite2017computing} proposed to directly optimize the PAC-Bayes bound for multi-layer perceptrons,
and it was improved by \citet{rivasplata2019pac} with a relaxation on the bound.
Then, \citet{arora2018stronger} formulated the compression framework, and showed that models with noise stability are more compressible and thus lead to tighter bounds (also discovered in \citet{nagarajan2019deterministic}).
\citet{zhou2018non} then used the compression framework to obtain non-vacuous generalization bounds for ImageNet.
More recently, \citet{lotfi2022pac} compressed models by considering model equivariance,
and \citet{2023arXiv230405366G} verified that large language models favor simpler sequences.
Both papers proposed to use generalization bounds based on the Kolmogorov complexity \citep{kolmogorov1963tables}.

\paragraph{Implicit bias of overparameterized models.}
This line of analysis attributes the good generalization to the implicit bias of a specific model architecture,
\ie{} with certain training algorithms such as SGD, these models can find the global minima that can generalize.
First, on the optimization side, there is a series of work that shows that gradient descent finds the global optima of overparameterized neural networks,
starting from the results in \citet{kawaguchi2016deep} for linear neural networks,
and then \citet{li2017convergence,du2018gradient,du2019gradient,arora2019fine} for two-layer fully-connected ReLU networks,
and \citet{allen2019convergence} for other architectures such as CNNs and ResNets.
The analysis on overparameterized fully-connected networks with two or more layers is then extended to generalization in \citet{du2018power,allen2019learning}.
There is also analysis based on neural tangent kernels (NTK) \citep{jacot2018neural},
including \citet{lee2019wide,arora2019exact}.
Most of these work requires the neural network to be overparameterized, \ie{} the number of neurons is polynomially large comparing to the input size \citep{allen2019convergence}.
One caveat of this line of analysis, as argued in \citet{chizat2019lazy}, 
is that these overparameterized models might fall into the ``lazy training'' regime,
meaning that the training dynamics of these models resemble those of a linear model.
They argued that the phenomenal success of deep neural networks is unlikely to be due to their similarity to linear models,
and \citet{wang2020beyond} provided a counter-argument.

\subsection{Theoretical Study on Augmentation Based Representation Learning}
\label{app:comparison}

Many recent studies have been devoted to theoretically understanding representation learning.
One line of research studies the effectiveness of contrastive learning by showing its features are optimal for linear predictor of certain downstream tasks~\citep{saunshi2019theoretical,tosh2021contrastive,tosh2021multiview}, 
robust to class imbalance~\citep{liu2021selfsupervised}, 
and suitable for unsupervised domain adaptation~\citep{shen2022connect,haochen2022beyond}.
Masked prediction tasks have been shown to be useful for reducing the downstream sample complexity~\citep{lee2021predicting} and for parameter identifiability~\citep{liu2022masked}.
Relating to language applications, \citet{saunshi2019theoretical} explained why next-word prediction can benefit sentiment classification, 
and \citet{wei2021pretrained} studied the effect of prompt tuning through the lens of implicit Bayesian inference.
Regarding the optimization in representation learning,
there have been prior works on the training dynamics and loss landscapes of contrastive learning~\citep{WenLi21,jing2022understanding,tian2022deep}, non-contrastive learning~\citep{TianCG21,PokleTLR22,wen2022mechanism}, and masked prediction \citep{xiong2020layer,pmlr-v119-huang20f}.

Representation learning has adopted the idea of compression.
Classical representation learning aims to extract the low-dimensional manifold on which the data resides \citep{belkin2003laplacian,belkin2004semi,bengio2004learning}.
More recently, \cite{yu2020learning} proposes maximal coding rate reduction (MCR$^2$), which learns an explicitly meaningful low-dimensional representation by maximizing the rate reduction, defined as the difference in coding rate, \ie{} the number of bits needed to encode the embedding before and after a certain partition (\eg{} according to different classes).
\citet{yu2020learning} showed that MCR$^2$ is closely related to contrastive learning.
Follow-up work includes \citet{chan2022redunet,dai2022ctrl}.

More closely related to this work is the line of work that formulates contrastive learning as a Laplacian operator over the augmentation graph.
The idea of studying data augmentation from a kernel perspective was first explored in \citet{mroueh2015learning,raj2017local,dao2019kernel}.
The augmentation graph was defined in \citet{haochen2021provable}
in whose Theorem 4.2 proved a generalization bound for the spectral contrastive loss that includes a Rademacher complexity term \wrt{} $a_1,\cdots,a_N$,
which is incurred by sampling only one $a$ from each $x$ (see their Definition D.3).
Then, \citet{saunshi2022understanding} pointed out that this model-class-free bound could be vacuous with a hypercube construction.
One thing to note is that while \citet{saunshi2022understanding} claimed that the bound in \citet{haochen2021provable} is ``function class independent'', actually it is not.\footnote{Theorem 4.2 in the NeurIPS version and the arXiv [v7] version of \citet{haochen2021provable} have different forms. We have confirmed with the authors that this theorem does depend on the complexity of the function class.}
As a response to \citet{saunshi2022understanding},
\citet{haochen2022theoretical} included the effect of the encoder's inductive bias into their generalization bound.
Then, \citet{wang2023a} connected contrastive learning to message passing on the augmentation graph.

In \cref{sec:conclusion} we highlighted two related recent papers.
\cite{wang2022spectral} studied a similar RKHS as the $\laph$ defined in this work.
In particular, the kernels $k_x$ and $k_z$ in their work are similar to the kernels $K_A$ and $K_X$ in this work.
However, their work studies the conditional moment model, which is a completely different problem than ours. 
Their bounds are also distinct from ours and are with stronger assumptions.
Nevertheless,
the fact that they derived similar kernels to ours from a completely different angle suggests a strong innate connection,
which alludes to the role of augmentation in self-supervised pretraining, namely, it defines a ordered set of features as we articulated after \cref{eqn:lap-equ}.

Another related work is \citet{cabannes2023ssl}, who studied the same problem as this work.
The difference is that their work uses a pre-defined kernel,
while the kernels $K_A$ and $K_X$ in this work are completely induced by the augmentation, \ie{} there would be not be a kernel without the augmentation.
Their work studied regularized kernel regression,
while this work studies unregularized kernel regression.
Moreover, the bounds derived in their work and ours are significantly different.
We would also like to credit \citet{cabannes2023ssl} for studying the upstream-to-downstream distribution shift which provides inspiring preliminary analysis.
Last but not least, the follow-up work by \citet{zhai2024stkr} applies the framework developed in this paper to a more general semi-supervised kernel learning setting.

Finally, similar to \citet{saunshi2022understanding}, we point out one common issue of existing learning guarantees for contrastive learning: Most of them require the 
samples from the same downstream class to have some type of ``overlap'' in the augmentation graph,
but we find that sometimes these assumptions could be too strong to be reasonable.
For instance, \citet{wang2022chaos} assumed intra-class connectivity, \ie{} the subgraph of the augmentation graph given by each downstream class is connected; 
\citet{huang2023towards} assumed that each downstream class has a ``main part'' of samples, in which any two samples can be augmented to be close in Euclidean distance.
We find these two assumptions hard to be reasonable for any practical scenario.

\section{Variational Objectives that Extract the Top-d Eigenspace}
\label{app:objectives}
In Proposition \ref{prop:reg}, we demonstrated the optimality of the top-$d$ eigenspace.
The question is: How to extract the top-$d$ eigenspace?
We cannot do so via explicit eigendecomposition of the kernel in most real tasks,
where the kernel is huge and sparse.
Instead, we will show that several pretraining objectives can extract the top-$d$ eigenspace assuming optimization converges to the global minima,
because these objective are uniquely optimized by the top-$d$ eigenspace.

\paragraph{Contrastive learning.}
Let us warm up with contrastive learning, whose connection to kernels has been studied in prior work,
among which \citet{haochen2021provable} introduced the \textit{spectral contrastive loss}:
\begin{equation}
    \label{eqn:spectral-loss}
    \gL_{\textnormal{SCL}}(\hat{\Phi}) =  -2 \E \left [ \langle \hat{\Phi}(A), \hat{\Phi}(A^+) \rangle \right ] + \E \left [ \langle \hat{\Phi}(A), \hat{\Phi}(A^-) \rangle^2 \right ] ,
\end{equation}
where $(A,A^+)$ is a positive pair that is jointly sampled from $P_X(a,a^+)$, and $A^-$ is a negative sample that is sampled from $\pa$ independently of $A$. It has been shown that:
\begin{thm}
\label{thm:contrastive-minimizer}
\citep{haochen2021provable}
Eqn. (\ref{eqn:spectral-loss}) is minimized by the top-$d$ eigenspace, and is uniquely minimized by the top-$d$ eigenspace if $\lambda_{d+1} < \lambda_d$.
\end{thm}

Here, ``minimized by the top-$d$ eigenspace'' means that the objective is minimized when $\hat{\Phi}$ spans the top-$d$ eigenspace,
and ``uniquely minimized'' means that this condition is necessary. The proof is based on the following Eckart-Young-Mirsky Theorem:
\begin{thm}[Eckart-Young-Mirsky]
\label{thm:eym}
Given a matrix $\mA \in \R^{m \times n}$, consider the following optimization problem:
\begin{equation}
    \min_{\mB \in \R^{m \times n}, \textnormal{rank} (\mB) \le k} \| \mA - \mB \|_F^2
\end{equation}
where $1 \le k \le \min \{ m, n \}$. Let the SVD of $\mA$ and $\mB$ be $\mA = \mU \mSigma \mV^{\top}$ and $\mB = \mU_1 \mSigma_1 \mV_1^{\top}$, where $\mSigma = \diag(\sigma_1,\cdots,\sigma_{\min \{ m,n \}})$, and $\sigma_1 \ge \sigma_2 \ge \cdots \ge \sigma_{\min \{ m, n \}}$. Then $\mB$ is a minimizer of the above problem if $\mU_1$ is the first $k$ columns of $\mU$, $\mV_1$ is the first $k$ columns of $\mV$, and $\mSigma_1$ is the upper left $k \times k$ block of $\mSigma$. Moreover, if $\sigma_{k+1} < \sigma_k$, then this is the unique minimizer.
\end{thm}

To make this paper self-contained, here we recap the proof of Theorem \ref{thm:contrastive-minimizer} in \citet{haochen2021provable}:
\begin{proof}
Define matrix $\mC$ as: $\hat{\Phi} = \mC \Phistar$, and matrix $\mB$ as $\mB = \mC^{\top} \mC$. Then, $\langle \hat{\Phi}, \hat{\Phi} \rangle = \Phistar^{\top} \mB \Phistar$.
Let the SVD of $\mC$ be $\mC = \mU \mSigma \mV^{\top}$.
Denote $\mD_\lambda = \diag(\lambda_1,\lambda_2,\cdots)$.
Then, the spectral contrastive loss is equivalent to
\begin{equation}
    \gL_{\textnormal{SCL}}(\hat{\Phi}) = -2 \sum_i \lambda_i b_{i,i} + \sum_{i,j} b_{i,j}^2 = \| \mB - \mD_\lambda \|_F^2 - \| \mD_\lambda \|_F^2.
\end{equation}

So it suffices to minimize $\| \mB - \mD_\lambda \|_F^2$ where $\textnormal{rank}(\mB) \le d$. 
By Theorem \ref{thm:eym}, this is minimized if $\mV(i,j) = \delta_{i,j}$ and $\mSigma = \diag(\sqrt{\lambda_i})$, which means that $\hat{\Phi} = \mU \diag(\sqrt{\lambda_1},\cdots,\sqrt{\lambda_d}) \Phistar$ for some orthogonal matrix $\mU \in \R^{d \times d}$. Moreover, by Thm. \ref{thm:eym}, this is the unique minimizer if $\lambda_{d+1} < \lambda_d$.
\end{proof}

\paragraph{CLIP.}
CLIP \citep{CLIP} maximizes the similarity between two samples from different modalities.
There is an objective similar to the spectral contrastive loss, which we term the \textit{spectral CLIP loss}:
\begin{equation}
\label{eqn:spectral-clip-loss}
    \gL_{\textnormal{SCLIP}}(\hat{\Phi}, \hat{\Xi})  =  -2 \E [ \langle \hat{\Phi}(A), \hat{\Xi}(X^+) \rangle ] + \E [ \langle \hat{\Phi}(A), \hat{\Xi}(X^-) \rangle^2],
\end{equation}
where $\hat{\Phi}$ and $\hat{\Xi}$ are encoders for different modalities (\eg{} image and text).
$(A,X^+)$ is a positive pair, while $X^-$ is a negative sample independent of $A$.
This loss also appeared in \citet{wang2022spectral} (see their Section 4.1).
We will show that this loss also extracts the top-$d$ eigenspace.
\begin{thm}
\label{thm:spectral-clip}
Eqn. (\ref{eqn:spectral-clip-loss}) is minimized by the top-$d$ eigenspace, and is uniquely minimized by the top-$d$ eigenspace if $\lambda_{d+1} < \lambda_d$.
\end{thm}
\begin{proof}
    Let $\hat{\Phi} = \mC \Phistar$ and $\hat{\Xi} = \mS \Psistar$.
    Let $\Xi^\dagger = \tax \hat{\Xi} = \mS \mD_\lambda^{1/2} \Phistar$.
    Then, we have
    \begin{align*}
        \E [ \langle \hat{\Phi}(A), \hat{\Xi}(X^+) \rangle ] &= \iint \left \langle \hat{\Phi}(a), \hat{\Xi}(x) \right  \rangle p(x|a) p(a) da dx \\ 
        & = \int \left \langle \hat{\Phi}(a), \int \hat{\Xi}(x) p(x|a) dx \right \rangle p(a) da \\ 
        &= \int \left \langle \hat{\Phi}(a), \Xi^\dagger(a) \right \rangle p(a) da = \Tr\left (\mC^{\top} \mS \mD_\lambda^{1/2} \right )  .
    \end{align*}
    We also have
    \begin{align*}
        \E [ \langle \hat{\Phi}(A), \hat{\Xi}(X^-) \rangle^2] & = \iint \left ( \sum_{k=1}^d \hat{\Phi}(a)_k \hat{\Xi}(x)_k \right )^2 p(a) p(z) da dz \\ 
        &= \sum_{k,l=1}^d \left ( \int \hat{\Phi}(a)_k \hat{\Phi}(a)_l p(a) da \right ) \left ( \int \hat{\Xi}(x)_k \hat{\Xi}(x)_l p(x) dx \right ) \\ 
        & = \Tr \left ( \mC \mC^{\top} \mS \mS^{\top} \right ) = \left \| \mC^{\top} \mS \right \|_F^2   .
    \end{align*}
    Therefore, the spectral CLIP loss is equivalent to
    \begin{equation}
    \label{eqn:spec-clip-equiv}
        \gL_{\textnormal{SCLIP}}(\hat{\Phi}, \hat{\Xi}) = \left \| \mC^{\top} \mS \right \|_F^2 - 2 \Tr\left (\mC^{\top} \mS \mD_\lambda^{1/2} \right ) = \left \| \mC^{\top} \mS - \mD_\lambda^{1/2} \right \|_F^2 - \left \| \mD_\lambda^{1/2} \right \|_F^2  .
    \end{equation}
    Let the SVD of $\mC$ be $\mC = \mU \mSigma \mV^{\top}$.
    Let $\mM = \mSigma \mU^{\top} \mS$. Then, $\mC^{\top} \mS = \mV \mM$.
    It is well known in linear regression that
    Eqn. (\ref{eqn:spec-clip-equiv}) is minimized when $\mM = \left ( \mV^{\top} \mV  \right )^{-1} \mV^{\top} \mD_\lambda^{1/2} = \mV^{\top} \mD_\lambda^{1/2}$, so
    \begin{equation*}
        \min_{\hat{\Phi}, \hat{\Xi}} \gL_{\textnormal{SCLIP}}(\hat{\Phi}, \hat{\Xi}) = \min_{\mV} \left \| (\mI - \mV \mV^{\top} ) \mD_\lambda^{1/2} \right \|_F^2 - \left \| \mD_\lambda^{1/2} \right \|_F^2  .
    \end{equation*}
By Theorem \ref{thm:eym}, this is minimized when the top $d$ rows of $\mV$ is an orthogonal matrix, and this condition is also necessary if $\lambda_{d+1} < \lambda_d$.
This proves the result.
Moreover, we can see that $\hat{\Xi}$ is the minimizer if it extracts the top-$d$ eigenspace of $\Psistar$.
\end{proof}

\paragraph{Regularized Barlow Twins.}
Barlow twins \citep{zbontar2021barlow} is a non-contrastive learning method that trains without negative sampling.
We consider the following \textit{regularized Barlow Twins loss}:
\begin{equation}
\label{eqn:modified-bt}
 \gL_{\textnormal{RBT}}(\hat{\Phi};\alpha,\beta) = \sum_{k=1}^d \left ( \E \left [ \hat{\Phi}(A)_k \hat{\Phi}(A^+)_k \right ] - 1 \right ) ^2 + \alpha \sum_{k \neq l} \left ( \E \left [ \hat{\Phi}(A)_k \hat{\Phi}(A^+)_l \right ] \right ) ^2 + \beta \E \left [ \left \| \hat{\Phi}(A) \right \|_2^2 \right ] .
\end{equation}
When the regularization term $\beta$ is close to 0 so that the last term is much smaller than the first two, the above can be considered as a constrained optimization task:
\begin{equation}
\label{eqn:rbt-constrained}
    \min_\theta \quad  \E \left [ \left \| \hat{\Phi}(A) \right \|_2^2 \right ] \qquad \text{s.t.} \qquad  \gL_{\textnormal{RBT}}(\hat{\Phi};\alpha,0) = 0.
\end{equation}
\begin{thm}
\label{thm:modified-bt-prob}
    The constrained optimization problem Eqn. (\ref{eqn:rbt-constrained}) is minimized by the top-$d$ eigenspace, and is uniquely minimized by the top-$d$ eigenspace if $\lambda_{d+1} < \lambda_d$.
\end{thm}
\begin{proof}
    Define matrix $\mC = (c_{i,j})$ as: $\hat{\Phi} = \mC \Phistar$.
    Then, for all $k,l \in [d]$, there is 
    \begin{equation*}
    \E [\hat{\Phi}(A)_k \hat{\Phi}(A^+)_l] = \sum_i \lambda_i c_{k,i} c_{l,i}.
\end{equation*}
Thus, $ \gL_{\textnormal{RBT}}(\hat{\Phi};\alpha,0) = 0$ is equivalent to $\mC \mD_\lambda \mC^{\top} = \mI$.
Meanwhile, $\E \left [ \left \| \hat{\Phi}(A) \right \|_2^2 \right ] = \| \mC \|_F^2$.
So the optimization problem is equivalent to minimizing $\| \mC \|_F^2$ subject to $\mC \mD_\lambda \mC^{\top} = \mI$.

Let $\mG = \mC \sqrt{\mD_\lambda} = (g_{k,i})$, then $\mG \mG^{\top} = \sum_k \vg_k \vg_k^{\top} = \mI$, where $\vg_k = [g_{k,1},\cdots,g_{k,d}]$.
For all $j \in [d]$, define $\mM_j = \sum_{k=1}^j \vg_k \vg_k^{\top}$.
Then, $\textnormal{rank}(\mM_j) \le j$,
and all eigenvalues of $\mM_j$ belong to $[0,1]$.
So the sum of eigenvalues of $\mM_j$ is at most $j$, which implies that
\begin{equation*}
    \sum_{k=1}^j \vg_j^{\top} \vg_j = \sum_{k=1}^j \Tr \left ( \vg_j^{\top} \vg_j \right ) = \sum_{k=1}^j \Tr \left ( \vg_j \vg_j^{\top} \right ) = \Tr(\mM_j) \le j .
\end{equation*}
Denote $\frac{1}{\lambda_0} = 0$. Thus, by Abel transformation, we have
\begin{equation*}
    \| \mC \|_F^2 = \sum_i \frac{1}{\lambda_i} \vg_i^{\top} \vg_i = \sum_{j \ge 0} \left ( d - \sum_{k=1}^j \vg_k^{\top} \vg_k \right ) \left ( \frac{1}{\lambda_{j+1}} - \frac{1}{\lambda_j} \right ) \ge \sum_{j=0}^{d-1} (d-j) \left ( \frac{1}{\lambda_{j+1}} - \frac{1}{\lambda_j} \right ) = \sum_{i=1}^d \frac{1}{\lambda_i}.
\end{equation*}
A sufficient condition of achieving the above equality is $\sum_{k=1}^j \vg_k^{\top} \vg_k = j$ for all $j \in [d]$,
and it is also necessary if $\lambda_{d+1} < \lambda_d$.
This condition is equivalent to $\vg_1,\cdots,\vg_d$ forming an orthonormal basis of $\R^d$, \ie{} $\hat{\Phi}$ extracts the top-$d$ eigenspace.
\end{proof}

\paragraph{Maximizing the ratio trace.}
Now we demonstrate the connection of contrastive learning and Barlow Twins to maximizing the ratio trace defined in Definition \ref{def:rt}.
We start with the regularized Barlow Twins loss, which we have shown is equivalent to minimizing $\| \mC \|_F^2$ subject to $\mC \mD_\lambda \mC^{\top} = \mI$.
Now, observe that $\mG = \mC \mC^{\top}$,
and $\mF = \mC \mD_\lambda \mC^{\top}$.
Thus, regularized Barlow Twins is essentially minimizing $\Tr(\mF^{-1} \mG)$,
which is similar but not equivalent to maximizing $\Tr(\mG^{-1} \mF)$.

Now we consider a variant of VICReg \citep{bardes2021vicreg}, also given by \citet[Eqn.~(2)]{cabannes2023ssl}:
\begin{equation}
\label{eqn:vicreg}
    \gL_{\textnormal{VICReg}}(\hat{\Phi}; \beta) = \left \| \E \left [ \hat{\Phi}(A) \hat{\Phi}(A)^{\top}  \right ] - \mI \right \|_F^2 + \beta \E \left [ \left \| \hat{\Phi}(A) - \hat{\Phi}(A^+) \right \|_2^2 \right ]  .
\end{equation}
When the regularization term $\beta$ is close to 0,
we have $\E \left [ \hat{\Phi}(A) \hat{\Phi}(A)^{\top}  \right ] = \mG = \mI$.
And we minimize $\E \left [ \left \| \hat{\Phi}(A) - \hat{\Phi}(A^+) \right \|_2^2 \right ]$, which is equivalent to $-2 \E \left [ \Tr \left (  \hat{\Phi}(A) \hat{\Phi}(A^+) \right ) \right ] + c = -2\Tr(\mF) + c$ for some constant $c$.
Thus, this objective is equivalent to maximizing $\Tr(\mG^{-1} \mF)$, the ratio trace, when $\beta \approx 0$.
And when $\beta = 1$,
Eqn. (\ref{eqn:vicreg}) is equivalent to the spectral contrastive loss.

\input{proofs_sec2}

\input{proofs_sec3}

\input{proofs_sec4}

\section{Experiment Details}
\label{app:exp}

This section contains the details of the experiments we conducted in Section \ref{sec:exp}.

\subsection{Estimating the Augmentation Complexity of Masked Language Modeling}

We estimate the $\bt$ of four augmentations, namely Random masking, Block masking, Random masking + Flipping, and Block masking + Flipping, on the widely used NLP dataset \texttt{wikipedia-simple}.
Since it is prohibitively expensive to iterate over the entire dataset,
we instead aim to estimate $\bt$ with a subset of the dataset.
In our experiments, instead of estimating $\sup_x \kx(x,x)$,
we estimate its 99$^{th}$ percentile.
This is because of two reasons:
\begin{enumerate}[(i)]
    \item $\sup_x \kx(x,x)$ is statistically impossible to estimate from a subset of data without any extra assumptions on the distribution of $\kx(x,x)$. This means that without iterating over the entire dataset, we cannot get a finite confidence interval.
    The percentile, on the other hand, can be estimated with a finite confidence interval via sampling regardless of the distribution of $\kx(x,x)$, which is a well-known result in statistics \citep[Section~5.2]{hahn2011statistical}.
    \item Almost all real datasets contain outliers, \ie{} samples that are very different from the other samples. These samples will have a very large $\kx(x,x)$, because their augmentations rarely overlap with the augmentations of other samples.
    Therefore, $\sup_x \kx(x,x)$ itself is not really meaningful because it is too sensitive to outliers, and cannot show how the augmentation works on most part of the population.
\end{enumerate}

As a demonstration of (ii), the following is an example outlier in the \texttt{wikipedia-simple} dataset,\footnote{Dataset available at \href{https://huggingface.co/datasets/wikipedia/viewer/20220301.simple/train}{huggingface.co/datasets/wikipedia}. The index of the example outlier is 199562.}
and it has a very large $\kx(x,x)$:
{
\color{darkgray}
\begin{verbatim}
Geroldsgrün is a municipality in Hof, Bavaria, Germany.

Geography

Boroughs

 Dürrenwaid 
 Dürrenwaiderhammer
 Langenbachtal
 Langenau
 Langenbach
 Lotharheil
 Mühlleiten
 Geroldsgrün  
 Geroldsreuth
 Großenreuth
 Hermesgrün 
 Hertwegsgrün   
 Hirschberglein     
 Silberstein  
 Steinbach
 Untersteinbach

References
\end{verbatim}
}

We can see that this sample is very different from a typical sample in the dataset in that (a) It contains lots of German words while this is an English dataset; (b) It mostly consists of individual words, while most other samples are comprised of sentences and paragraphs.
Such outliers are rare in the dataset,
but they result in a large $\kx(x,x)$.
Therefore, we instead estimate the $\beta$-th percentile of $\kx(x,x)$ for some $\beta$ close to 100. This is a more robust estimator that can be computed with a subset of data.

Specifically, we sample $m=1000$ samples $x_1,\cdots,x_m$ from the dataset uniformly at random.
The maximum length is set at $l = 64$.
For each $x_i$,
we estimate $\int \frac{p(x_i|a)}{p(x_i)} p(a|x_i) da$ via sampling $r=255$ \iid{} augmentations from $p(a|x_i)$ and computing the mean of $\frac{p(x_i|a_j)}{p(x_i)}$.
So it suffices to estimate $p(x_i|a_j)$ and $p(x_i)$.
To do this, we leverage a \texttt{bert-base-uncased}, which is a bi-directional language model.
Then, we can estimate $p(x|a)$ for any $x$ and $a$ as follows:
For a sample $x=[x^{(1)}, \cdots, x^{(l)}]$ ($x^{(i)}$ is a token) and its augmentation $a$,
there is
\begin{equation*}
    \log p \left (x \middle \vert a \right ) = \log p \left (x^{(1)} \middle \vert a \right ) + \log p \left (x^{(2)} \middle \vert a,x^{(1)} \right ) + \cdots + \log p \left (x^{(l)} \middle \vert a, x^{(1)}, \cdots, x^{(l-1)} \right ) .
\end{equation*}

So we can estimate $\log p \left (x \middle \vert a \right )$ in an autoregressive fashion:
First estimate $\log p \left (x^{(1)} \middle \vert a \right )$,
then replace $a^{(1)}$ with $x^{(1)}$ and estimate $\log p \left (x^{(2)} \middle \vert a,x^{(1)} \right )$,
then replace $a^{(2)}$ with $x^{(2)}$ and estimate $\log p \left (x^{(3)} \middle \vert a,x^{(1)}, x^{(2)} \right )$,
and so on.
The bi-directionality of BERT is important for such estimation.
To compute $p(x)$,
we set $a$ to be a fully masked sentence,
and then compute $p(x|a) = p(x)$.

First, we decide the percentile of $\kx(x,x)$ to report.
In Figure \ref{fig:hist}, we plot the histogram of $\log \left ( \int \frac{p(x|a)}{p(x)} \hat{p}(a|x) da \right )$ estimated using the above method of random masking with mask ratio $\alpha=0.15$ and $\alpha = 0.4$.
The red dashed lines indicate the $99^{th}$ percentile,
and we can see that they cut the tail of the data while preserving the bulk of the mass.
Thus, we report the $99^{th}$ percentile for this dataset.

\begin{figure}[!t]
    \centering
\begin{subfigure}[b]{.49\textwidth}
\includegraphics[width=\linewidth]{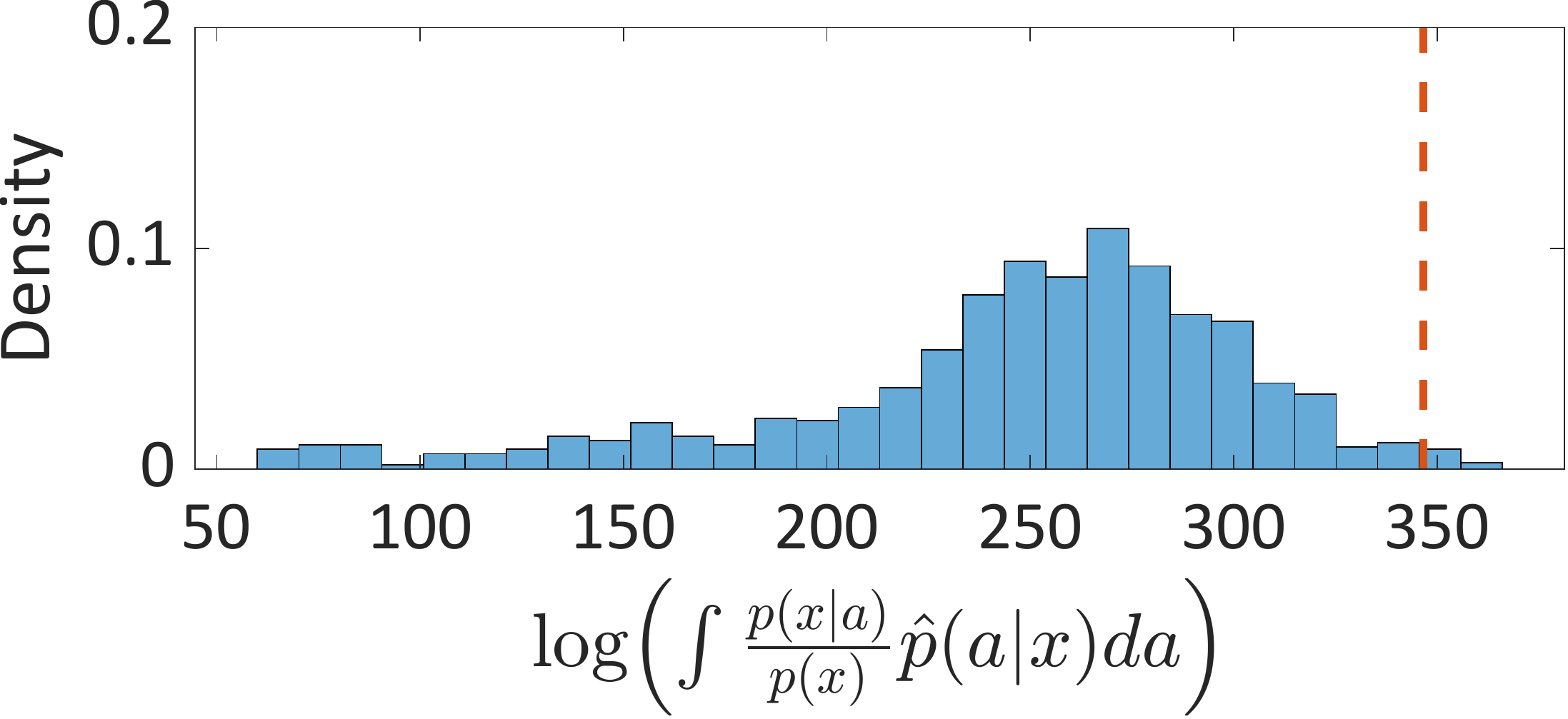}
\end{subfigure}
\hfill
\begin{subfigure}[b]{.49\textwidth}
\includegraphics[width=\linewidth]{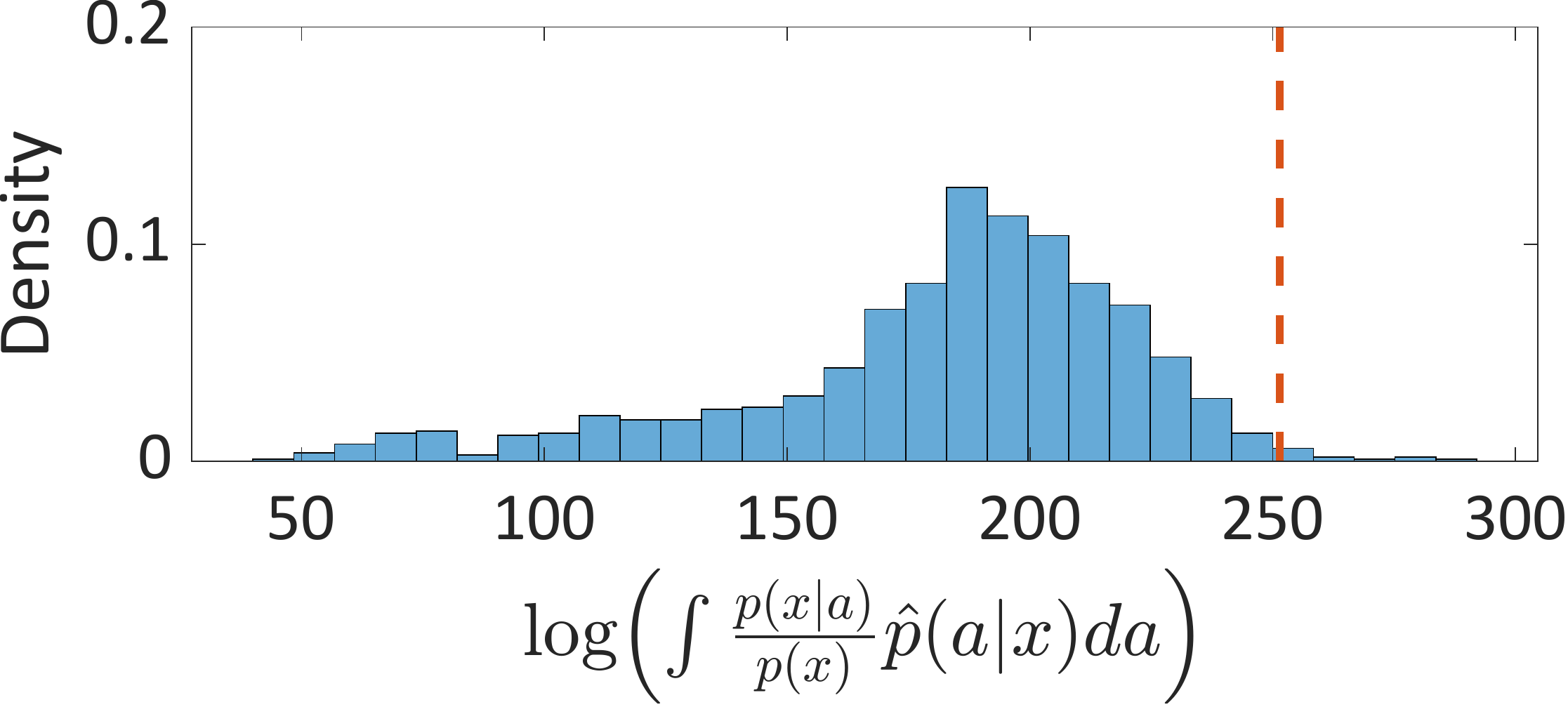}
\end{subfigure}
\caption{Plots of $\log \left ( \int \frac{p(x|a)}{p(x)} \hat{p}(a|x) da \right )$ for random masking. \textbf{Left:} $\alpha=0.15$; \textbf{Right:} $\alpha=0.4$.}
\label{fig:hist}
\end{figure}

We study four masking augmentations: Random masking (mask $r = \lceil \alpha d_\gX \rceil$ tokens uniformly at random), block masking (mask a block of $r$ tokens uniformly at random), random masking and flipping (mask $r/2$ tokens uniformly at random and then replace $r / 2$ remaining tokens with random tokens), and block masking and flipping (mask a block of $r/2$ tokens uniformly at random and then replace $r / 2$ remaining tokens with random tokens). 
In standard BERT, the number of randomly replaced tokens is much smaller than the number of masked tokens, but here we make them the same in order to magnify their difference.

\begin{table}[!t]
  \caption{99$^{th}$ percentile of estimated $\log \kx(x,x)$ on \texttt{wikipedia-simple}.}
  \label{tab:est-wikipedia}
  \centering
  \begin{tabular}{l|cccc}
    \toprule
    & Random Mask & Random Mask + Flip & Block Mask  & Block Mask +  Flip \\ 
    \midrule
    $\alpha=0.05$ & \res{353.94}{4.63} & \res{350.99}{3.20}  & \res{346.77}{4.47}  &  \res{350.46}{1.43}  \\ 
    $\alpha=0.10$ & \res{348.45}{2.94} & \res{349.14}{1.15}  & \res{339.86}{5.62}  &  \res{343.73}{5.76}  \\ 
    $\alpha=0.15$ & \res{346.40}{1.94} & \res{342.77}{4.17}  & \res{327.13}{2.33}  &  \res{336.81}{5.27}  \\ 
    $\alpha=0.20$ & \res{339.03}{2.88} & \res{335.01}{2.65}  & \res{313.90}{2.55}  & \res{322.87}{3.55} \\ 
    $\alpha=0.30$ & \res{325.08}{2.35} & \res{312.21}{5.96}  & \res{280.65}{1.87}  & \res{302.60}{2.49} \\ 
    $\alpha=0.40$ & \res{307.00}{4.55} & \res{280.76}{0.71}  & \res{254.40}{1.48}  & \res{279.24}{4.59} \\ 
    $\alpha=0.50$ & \res{278.59}{2.46}   & \res{241.59}{4.50} & \res{217.35}{2.81}  & \res{240.73}{1.58}  \\
    $\alpha=0.60$ & \res{250.01}{3.43} & \res{198.15}{4.01}   & \res{187.60}{2.45} & \res{205.45}{2.63}  \\
    $\alpha=0.70$ & \res{214.14}{0.88} & \res{151.58}{3.32}   & \res{152.70}{1.60} & \res{164.48}{2.59}  \\
    $\alpha=0.80$ & \res{160.24}{2.44} & \res{84.51}{2.35}   & \res{113.86}{2.39} & \res{108.44}{1.43}  \\
    \bottomrule
  \end{tabular}
\end{table}

In Table \ref{tab:est-wikipedia}, we report the estimated 99$^{th}$ percentile of $\log \kx(x,x)$.
Each experiment is run five times with different random seeds, and we report the average and standard deviation.
We can see that:
\begin{enumerate}[(i)]
    \item $\bt$ decreases as $\alpha$ grows, showing that a stronger augmentation has a lower complexity.
    \item Random masking always has the highest complexity. Regarding the effect of block masking and flipping, block masking has a greater effect when $\alpha$ is small and it makes the complexity lower, whereas flipping has a greater effect when $\alpha$ is larger.
\end{enumerate}

To conclude, as long as there is a way to estimate $\frac{p(x|a)}{p(x)}$ or $\frac{p(a|x)}{p(a)}$, we can estimate $\bt$ via sampling from the dataset. Note that in our estimation, BERT is only used for estimating $p(x|a)$.
$\bt$ itself is defined free of any model.

\subsection{Analyzing the Effect of Augmentation on Downstream Performance}

Our theoretical analysis implies that as long as Assumption \ref{ass:lap} is satisfied, a smaller $\bt$, \ie{} a stronger augmentation, leads to better generalization.
However, this does not mean that a stronger augmentation will always lead to a better test performance of tighter generalization gap, due to the following reasons:
\begin{enumerate}[(i)]
    \item If the augmentation is too strong, then the model might have low training performance. In the extreme case, if the sentences are 100\% masked, then the language model can learn nothing.
    \item Stronger augmentations usually work better with larger models.
    For instance, \citet{wettig2022should} demonstrated that large models with strong augmentations can achieve higher performance than standard models and augmentations.
    However, for small models, stronger augmentations can lead to lower performance.
    One possible reason is that small models are not expressive enough, but we conjecture that this is not the true cause since a Transformer with a moderate size is already sufficiently complex.
    We conjecture that the real reason lies in optimization, in that it would be harder for the optimizer to find a point close to the global minima with a small model and a strong augmentation.
    The big model, on the other hand, has more parameters and could be easier to optimize in practice.
    \item If the augmentation is too strong, then Assumption \ref{ass:lap} will be violated, and our results will not hold so that the generalization gap can be large.
\end{enumerate}

We study the effect of the augmentation on two real NLP downstream tasks, namely QNLI and SST-2. 
We study random masking with different mask ratios.
Unlike the experiments in \citet{wettig2022should} that applied the pretrained encoder directly to downstream tasks,
in our experiments we explicitly use the average encoder Eqn. (\ref{eqn:avg-enc}),
which is estimated by sampling 16 augmentations for each sample $x$ and then averaging their embeddings.
We do, however, fine-tune the encoder during downstream as people do in common practice,
which has been proven useful,
\footnote{We find in our experiments that without fine-tuning,
the downstream classifiers cannot achieve meaningful training accuracy, merely higher than 70\% for a binary classification task.
With the training accuracy so low, there is no empirical generalization gap, and the test accuracy is even higher than the training accuracy in many cases,
so these results are not meaningful.}
even though our theory does not analyze fine-tuning.
We acknowledge that there is a discrepancy between the theory and real practice,
which we aim to address in the future.

Details of our experiments:
\begin{enumerate}[(i)]
    \item We train \texttt{roberta-large} models using the fast training recipe provided by \citet{wettig2022should}.
    For $\alpha \ge 0.15$,
    we directly use the Huggingface checkpoints they provide.
    For $\alpha = 0.05, 0.10$,
    we pretrain new encoders using their source code\footnote{\url{https://github.com/princeton-nlp/DinkyTrain}} without any modification. 
    Note that these models use pre-layer normalization (see the original github repository regarding this detail).
    We use 8 NVIDIA A6000 GPUs for pretraining.
    \item For downstream fine-tuning and evaluation, we use the official repository provided by Huggingface.\footnote{\url{https://github.com/huggingface/transformers/tree/main/examples/pytorch/text-classification}}
    The only modification we make is that we explicitly use the average encoder.
    The classifiers are trained for 3 epochs on QNLI and 6 epochs on SST-2.
    All hyperparameters are kept the same as the official repository.
    We use 4 NVIDIA A6000 GPUs for downstream training and evaluation.
\end{enumerate}

The results are plotted in Figure \ref{subfig:qnli} in Section \ref{sec:exp}. As discussed there,
the ``sweet spot'' where the model achieves the highest test performance is in the middle,
where the augmentation is strong enough to have good generalization,
yet not too strong to break Assumption \ref{ass:lap} or to make the training performance too low.

%% file: proofs_sec2.tex
\section{Proofs for Section \ref{sec:prelim}}

\label{app:sec2}

\paragraph{$\tax^*\tax$ and $\tax \tax^*$ are integral operators.}
\begin{equation}
\begin{cases}
(\taxstar \taxlong f)(x) = (\tax^* \tax f)(x) = \int  \kx(x,x') f(x') p(x') dx' ; \\ 
(\taxlong \taxstar g)(a) = (\tax \tax^* g)(a) = \int \ka(a,a')  g(a') p(a') da' .
\end{cases}
\end{equation}

\begin{proof}
    We only show the first equation, and the second one can be proved in the same way.
    \begin{align*}
        (\tax^* \tax f)(x) & = \tax^* \left ( \int f(x') p(x'|a) dx' \right ) = \int \left ( \int f(x') p(x'|a) dx' \right ) p(a|x) da \\ 
        & = \iint f(x') p(a|x) p(x'|a) da dx' = \iint f(x') \frac{p(a|x) p(a|x')}{p(a)} p(x') da dx' \\ 
        & = \int  \kx(x,x') f(x') p(x') dx'   .
    \end{align*}
\end{proof}

\paragraph{Proposition \ref{prop:dual-2} (Duality).} 
\textit{
$\tax \tax^*$ shares the same non-zero eigenvalues as $\tax^* \tax$, and there exist eigenfunctions $\{ \phi_i \}$ of $ \tax \tax^*$ that form an orthonormal basis of $\lap$, such that for any $\lambda_i > 0$,
\begin{equation}
    \psi_i = \lambda_i^{-1/2} \tax^* \phi_i \quad \text{and} \quad \phi_i = \lambda_i^{-1/2} \tax \psi_i ,
\end{equation}
and we also have the following spectral decomposition of the Radon-Nikodym derivative:
\begin{equation}
\label{eqn:rn-der-restated}
    \frac{dP_{ \gA \gX}}{d(\pa \otimes \px)} = \frac{p(a,x)}{p(a) p(x)} = \sum_i \lambda_i^{1/2} \phi_i(a) \psi_i(x)  .
\end{equation}
}
\begin{proof}
    Suppose $\lambda_i, \psi_i(x)$ is a pair of eigenvalue and eigenfunction of $\tax^* \tax$,
and $\lambda_i > 0$.
Then, we have
$\tax \tax^* \tax \psi_i = \lambda_i \tax \psi_i$,
which means that $\tax \psi_i$ is an eigenfunction of $\tax \tax^*$ with eigenvalue $\lambda_i$.
The $\lambda_i^{-1/2}$ is used for normalization.
To see this,
let $\phi_i = \lambda_i^{-1/2} \tax \psi_i$. Then, we have
\begin{align*}
    \langle \phi_i, \phi_j \rangle_\pa & = \lambda_i^{-1/2} \lambda_j^{-1/2} \langle \tax \psi_i, \tax \psi_j \rangle_\pa \\ 
    & = \lambda_i^{-1/2} \lambda_j^{-1/2} \langle \tax^* \tax \psi_i, \psi_j \rangle_\px \\ 
    & = \lambda_i^{-1/2} \lambda_j^{-1/2} \langle \lambda_i \psi_i, \psi_j \rangle_\px = \delta_{i,j} .
\end{align*}

We can prove the reverse direction similarly.
And for any fixed $x$, there is
\begin{equation}
    \left \langle \frac{p(a,x)}{p(a)p(x)}, \phi_i \right \rangle_\pa = \int \frac{p(a,x)}{p(a)p(x)} \phi_i(a) p(a) da = \int p(a|x) \phi_i(a) da = \sqrt{\lambda_i} \psi_i(x)  . 
\end{equation}
which implies Eqn. (\ref{eqn:rn-der-restated}). 
\end{proof}

\paragraph{Basic properties of $\laph$.}
\textit{
\begin{enumerate}[(i)]
    \item $\kx$ is the reproducing kernel of $\laph$, such that for all $f \in \laph$, $f(x) = \langle f, \kx(x, \cdot) \rangle_{\laph}$.
    \item $\laph = R(\tax^*)$. 
    \item $\laph$ is isometric to $\sspan(\set{\phi_i}_{\lambda_i > 0})$, a subspace of $\lap$, and $\| f\|_{\laph} = \inf_{g: f = \tax^* g} \|g \|_\pa$.
    \item \rebuttal{For any $f^* \in \Bdowne \subset R(\tax^*)$, let $f^* = \sum_i u_i \psi_i$. Define $g_0 := \sum_i \lambda_i^{-1/2} u_i \phi_i$. Then, $g_0$ must satisfy \cref{eqn:ass-lap}, so we can choose $g^* = g_0$, in which case \cref{eqn:ass-lap} is equivalent to:}
\begin{equation}
\label{eqn:ip-equiv}
  \langle g^*, (I - \tax \tax^*) g^* \rangle_\pa \le \epsilon \| g^* \|_\pa^2 \; \Leftrightarrow \; \sum_i \frac{1-\lambda_i}{\lambda_i} u_i^2 \le \epsilon \sum_i \frac{1}{\lambda_i} u_i^2 .
\end{equation}
\end{enumerate}
}

\begin{proof}
\begin{enumerate}[(i)]
    \item 
    First, note that $\laph = \sset{\sum_{i: \lambda_i > 0} a_i \ve_i }{\sum_i a_i^2 < \infty} $ where $\ve_i = \lambda_i^{-1/2} \psi_i$, so it is isomorphic to $\ell^2((a_i)_{i: \lambda_i > 0})$ and is thus a Hilbert space.
    Then, $\kx(x,x') = \sum_i \lambda_i \psi_i(x) \psi_i(x')  $. For any $f \in \laph$, let $f = \sum_i u_i \psi_i$, then
    \[
    \langle f(x'), \kx(x, x') \rangle_{\laph} = \sum_i \frac{1}{\lambda_i} u_i (\lambda_i \psi_i(x)) = \sum_i u_i \psi_i(x) = f(x).
    \]
    \item For any $f = \sum_i u_i \psi_i \in \laph$, there is $\sum_i \lambda_i^{-1} u_i^2 < \infty$ by definition. So for any $\lambda_i = 0$, there must be $u_i = 0$. Let $g = \sum_i \lambda_i^{-1/2} u_i \psi_i$. Then, $\|g\|_\pa^2 = \sum_i \lambda_i^{-1} u_i^2 < \infty$, meaning that $g \in \lap$.
    And there is $f = \tax^* g$, so $f \in R(\tax^*)$, which implies that $\laph \subseteq R(\tax^*)$.
    Meanwhile, for any $f = \tax^* g \in R(\tax^*)$, let $g = \sum_i v_i \phi_i$, then $\sum_i v_i^2 < \infty$.
    Then, $f = \sum_i \lambda_i^{1/2} v_i \psi_i$ by duality,
    so $\sum_i \lambda_i^{-1} (\lambda_i^{1/2} v_i)^2 < \infty$,
    meaning that $f \in \laph$, so $R(\tax^*) \subseteq \laph$.
    \item For any $f = \sum_i u_i \psi_i \in \laph$, let $g = \sum_i \lambda_i^{-1/2} u_i \psi_i$. By the proof of (ii) we know that $f \mapsto g$ is bijective, and $\|f\|_\laph = \|g\|_\pa$. Moreover, $g \in \sspan(\set{\phi_i}_{\lambda_i > 0})$.
    \item \rebuttal{Let $g^* = \sum_i v_i \phi_i$.
Then, since we have $f^* = \tax^* g^* = \sum_i \lambda_i^{1/2} v_i \psi_i$,
for any $\lambda_i > 0$,
there is $v_i = \lambda^{-1/2} u_i$;
and for any $\lambda_i = 0$, there is $u_i = 0$.
Let $g^* = g_0 + g_1$, where $g_0 = \sum_i \lambda_i^{-1/2} u_i \phi_i$,
and $g_1 \perp g_0$ and $\tax^* g_1 = 0$.
By duality, $\tax \tax^* g_0$ belongs to the linear span of $\{ \phi_i \}_{\lambda_i > 0}$,
so $g_1 \perp \tax \tax^* g_0$.
As we will show later, Eqn. (\ref{eqn:ass-lap}) is equivalent to \cref{eqn:ip-equiv}, which is the random walk normalized Laplacian over the augmentation graph \citep[Section~1.2]{chung1997spectral}.
This is equivalent to $\langle g^*, (I - \tax \tax^*) g^* \rangle_\pa \le \epsilon \| g^* \|_\pa^2$,
which is further equivalent to
$\langle g_0, (I - \tax \tax^*) g_0 \rangle_\pa + \| g_1 \|_\pa^2 \le \epsilon (\| g_0 \|_\pa^2 + \| g_1 \|_\pa^2)$ (note that $\tax \tax^* g_1 = 0$).
This implies that $\langle g_0, (I - \tax \tax^*) g_0 \rangle_\pa \le \epsilon \| g_0 \|_\pa^2 $,
\ie{} $g_0$ satisfies Eqn. (\ref{eqn:ass-lap}).
Since we also have $f^* = \tax^* g_0$,
$g_0$ must satisfy Assumption \ref{ass:lap}, so we can choose $g^* = g_0$. }

Next, to show the equivalence to \cref{eqn:ip-equiv}, We just need to show that $\langle g^*, (I - \tax \tax^*) g^* \rangle_\pa = \frac{1}{2} \E_{X \sim \px} \E_{A, A' \sim p(\cdot | X)} \left [ (g^*(A) - g^*(A'))^2 \right ]$. And indeed, we have:
\begin{align*}
    \langle g^*, (I - \tax \tax^*) g^* \rangle_\pa &= \left \langle g^*,g^* - \int g^*(a') \ka(\cdot,a') p(a') da' \right \rangle_\pa \\ 
    & = \|g^*\|_\pa^2 - \iint g(a) g(a') \frac{\int p(a|x) p(a'|x) p(x) dx}{p(a) p(a')} p(a') p(a) da da' \\ 
    & = \frac{1}{2} \E[g^*(A)^2] + \frac{1}{2} \E[g^*(A')^2] - \frac{1}{2} \E_{X \sim \px} \E_{A, A' \sim p(\cdot | X)}[2 g^*(A) g^*(A')] \\ 
    & = \frac{1}{2} \E_{X \sim \px} \E_{A, A' \sim p(\cdot | X)} \left [ (g^*(A) - g^*(A'))^2 \right ] ,
\end{align*}
\end{enumerate}
as desired.
\end{proof}

%% file: proofs_sec3.tex
\section{Proofs for Section 3}

\subsection{Local Gaussian Complexity and Localized Rademacher Complexity}

We first provide the definition of the two complexities we will use in our analysis.
For a function $f$, let $\|f\|_n^2 := \frac{1}{n} \sum_{i=1}^n f(\tilde{x}_i)^2 $ be its mean on the downstream samples.

\begin{defi}
\citep[Eqns.~(13.16) \& (14.3)]{wainwright2019high}
For any $B, \epsilon > 0$, define
\begin{equation}
    \gF_0 := \sset{f_1 - f_2}{f_i \in \learnh, \|f_i\|_{\laph} \le \frac{B}{\sqrt{1 - \epsilon}}} = \sset{f \in \learnh}{\|f\|_{\laph} \le \frac{2B}{\sqrt{1 - \epsilon}}}.
\end{equation}
Then, the \textit{local Gaussian complexity} around $\fproj$ at scale $\delta > 0$ is given by
\begin{equation}
    \gG_n(\delta; \gF_0) := \underset{\omega_1,\cdots,\omega_n}{\E} \left [ \sup_{f \in \gF_0, \| f \|_n \le \delta} \left | \frac{1}{n} \sum_{i=1}^n \omega_i f(\tilde{x}_i) \right | \right ] ,
\end{equation}
where $\omega_1,\cdots,\omega_n$ are \iid{} $\gN(0,1)$ variates.
And define
\begin{equation}
    \gF_* := \sset{f = f_1 + \alpha f^*}{\alpha \in [-1,1], f_1 \in \learnh, \|f_1\|_{\laph} \le \frac{B}{\sqrt{1 - \epsilon}}} .
\end{equation}
Then, the \textit{localized population Rademacher complexity} of radius $\delta > 0$ is given by
\begin{equation}
    \lrad_n(\delta; \gF_*) := \underset{\sigma_1,\cdots,\sigma_n,x_1,\cdots,x_n}{\E} \left [ \sup_{f \in \gF_*, \| f \|_\px \le \delta} \left | \frac{1}{n} \sum_{i=1}^n \sigma_i f(x_i) \right | \right ] ,
\end{equation}
where $\sigma_1,\cdots,\sigma_n$ are \iid{} Rademacher variables taking values in $\set{-1,+1}$ equiprobably.
\end{defi}

Our master plan is to apply Theorems 13.13 and 14.1 of \citet{wainwright2019high} to $\fproj = \tax^*(\Pi_{\hPhi} g^*)$, where $\Pi_{\hPhi}$ is the projection operator onto $\hPhi$ in $\lxp$, and $\fproj$ is the projection of $f^*$ onto $\learnh$ \wrt{} $\langle \cdot, \cdot \rangle_\laph$.
Therefore, we need to bound $\gG_n(\delta; \gF_0)$ and $\lrad_n(\delta; \gF_*)$.
We start with the following uniform bound:
\begin{prop}
\label{prop:m-bound}
    If $f = \tax^* g$, and $\|g\|_\pa \le T$,
    then $|f(x)| \le \bt T$ for all $x$.
\end{prop}
\begin{proof}
By Eqn. (\ref{eqn:rn-der}), we have $p(a|x) = \sum_i \sqrt{\lambda_i} \phi_i(a) \psi_i(x) p(a)$.
For any $g = \sum_i u_i \phi_i \in \lap$ such that $\|g\|_\pa \le T$,
$(\tax^* g)(x) = \int g(a) p(a|x) da = \sum_i \sqrt{\lambda_i} u_i \psi_i(x)$.
Then, by Cauchy-Schwarz inequality, we have for all $x$,
$f(x)^2 = (\tax^* g)(x)^2 \le (\sum_i \lambda_i \psi_i(x)^2) (\sum_i u_i^2) \le \bt^2 T^2$.
\end{proof}

This proposition immediately implies that $f^*$ and $\fproj$ are uniformly bounded:
\begin{cor}
\label{cor:con-1}
For any $f^* \in \Bdowne$,
    Eqn. (\ref{eqn:lap-equ}) ensures that $\|g^*\|_\pa^2 \le \frac{B^2}{1-\epsilon}$,
   so $|f^*(x)| \le \frac{\bt B}{\sqrt{1-\epsilon}}$ for all $x$.
    Moreover, $\| \Pi_{\hPhi} g^* \|_\pa \le \|g^*\|_\pa$
   implies that $\|\fproj\|_\laph \le \frac{B}{\sqrt{1-\epsilon}}$, and $| \fproj(x) | \le \frac{\bt B}{\sqrt{1-\epsilon}}$ for all $x$.
\end{cor}

We will also use the following simple result in linear algebra:
\begin{lem}
\label{lem:abel-trace}
    Let $\mD_\lambda = \diag(\lambda_1,\lambda_2,\cdots)$ where $\lambda_1 \ge \lambda_2 \ge \cdots \ge 0$ and $\lambda_i \rightarrow 0$.
    Let $\mQ$ be a matrix with $d$ rows that are unit vectors.
    Then, $\Tr(\mQ \mD_\lambda \mQ^{\top}) \le \lambda_1 + \cdots + \lambda_d$.
\end{lem}
\begin{proof}
Let $\vq_i$ be the $i$-th column of $\mQ$.
Then for all $j \in [d]$, there is $\sum_{i=1}^j \vq_i^{\top} \vq_i \le j$.
And for $j > d$, $\sum_{i=1}^j \vq_i^{\top} \vq_i \le d$.
Thus, using Abel transformation, we have
\begin{align*}
 \Tr(\mQ \mD_{\lambda} \mQ^{\top}) &= \Tr(\mD_{\lambda} \mQ^{\top} \mQ ) = \sum_{i=1}^{\infty} \lambda_i \vq_i^{\top} \vq_i =  \sum_{j=1}^\infty \left (\sum_{i=1}^j \vq_i^{\top} \vq_i \right ) (\lambda_j - \lambda_{j+1}) \le \sum_{i=1}^d \lambda_i ,
\end{align*}
which proves the assertion.
\end{proof}

This result has many implications. For instance, for any rank-$d$ subspace of $\laph$, its trace (the sum of its eigenvalues) is at most $S_\lambda(d)$.

Now, let us bound $\gG_n(\delta; \gF_0)$ with the following result:
\begin{lem}
\label{lem:13-22}
(Application of \citet[Lemma~13.22]{wainwright2019high})
    Let $\gH$ be an RKHS with reproducing kernel $K$.
    Given samples $\tilde{x}_1,\cdots,\tilde{x}_n$,
    let $\mK$ be the normalized kernel matrix with entries $\mK(i,j) = K(\tilde{x}_i, \tilde{x}_j) / n$.
    Let $\mu_1 \ge \cdots \ge \mu_n \ge 0$ be the eigenvalues of $\mK$.
    Then for all $\delta > 0$, we have
    \begin{equation}
        \E \left [ \sup_{\|f\|_\gH \le T, \|f\|_n \le \delta} \left |  \frac{1}{n} \sum_{i=1}^n \omega_i f(\tilde{x}_i)   \right | \right ] \le \sqrt{\frac{2}{n}} \sqrt{\sum_{j=1}^n \min \{ \delta^2, \mu_j T^2 \} } ,
    \end{equation}
    where $\omega_1,\cdots,\omega_n$ are \iid{} $\gN(0,1)$ variates.
    We apply this result to $K = \kx$.
    By Definition \ref{def:bt}, all elements on the diagonal of $\mK$ are at most $\bt^2 / n$,
    so $\sum_j \mu_j = \Tr(\mK) \le \bt^2$.
    Thus,  we have
    \begin{equation}
        \gG_n(\delta; \gF_0) \le \sqrt{\frac{8 \bt^2 B^2}{n (1-\epsilon)}} \qquad \text{for any } \learnh  .
    \end{equation}
\end{lem}

Regarding $\lrad_n(\delta; \gF_*)$,
$\gF_*$ is also a subset of RKHS $\hat{\gH}_*$, which is the linear span of $\hat{\Psi}$ and $f^*$, and is a subspace of $\laph$ whose rank is at most $(d+1)$.
By Lemma \ref{lem:abel-trace}, the sum of eigenvalues of $\hat{\gH}_*$ is at most $S_\lambda (d+1)$.
Since $\|f^*\|_{\laph} \le \frac{B}{\sqrt{1 - \epsilon}}$,
all $f \in \gF_*$ satisfy $\|f\|_{\laph} \le \frac{2B}{\sqrt{1 - \epsilon}}$.
So we have the following bound for $\lrad_n(\delta; \gF_*)$:

\begin{lem}
(Application of \citet[Corollary~14.5]{wainwright2019high})
Let $\mu_1,\mu_2,\cdots$ be the eigenvalues of the RKHS $\hat{\gH}_*$.
Since $\rank(\hat{\gH}_*) \le \rank(\learnh) + 1$, we have
\begin{equation}
\label{eqn:radn-bound}
    \lrad_n(\delta; \gF_*) \le \sqrt{\frac{2}{n}} \sqrt{\sum_{j=1}^\infty \min \set{ \delta^2, \frac{4\mu_j B^2}{1-\epsilon} } } \le \sqrt{\frac{8 B^2}{n (1-\epsilon)} S_{\lambda}(d+1)} \quad \text{if }\rank(\learnh) \le d ,
\end{equation}
and for an arbitrary $\learnh$, we can simply replace $S_\lambda(d+1)$ with $S_\lambda$.
\end{lem}

\subsection{Proofs}

\paragraph{Lemma \ref{thm:main}.}
\textit{
    Suppose $\nse_1,\cdots,\nse_n$ are \iid{} $\gN(0,\sigma^2)$ variates.
    If $\hat{\Phi}$ has $d$ dimensions ($d$ can be $\infty$),
    then we have the following uniform bound
    over all $f^* = \tax^* g^* \in \Bdowne$:
    \begin{equation*}
    \begin{aligned}
        & \underset{\tilde{x}_i,\nse_i}{\sP} \left [ \forall f^* \in \Bdowne, \| \hat{f} - f^* \|_\px^2 \le  9 \|  \fproj - f^* \|_\px^2 + \frac{c_0 \bt (B^2 + \sigma B)}{1-\epsilon} \sqrt{\frac{S_\lambda(d+1)}{n} } \right ] \\ 
        \ge \; & 1 - c_1 \exp \left ( -\frac{c_2 \sqrt{2n S_\lambda(d+1)}}{\bt} \right ) - \exp \left ( - \sqrt{\frac{2 n \bt^2 B^2 }{1 - \epsilon}}  \right )  ,
    \end{aligned}
    \end{equation*}
    where $\fproj = \tax^*(\Pi_{\hPhi} g^*)$ is the projection of $f^*$ onto $\learnh$ \wrt{} $\langle \cdot, \cdot \rangle_\laph$,
    and $c_0,c_1, c_2$ are universal constants.
    Moreover, $S_\lambda(d+1) \le \min \set{d+1, \bt^2}$.
}
\begin{proof}
    By Proposition \ref{prop:m-bound}, all functions in $\gF_*$ are $b$-uniformly bounded,
    with $b = \frac{2 \bt B}{\sqrt{1-\epsilon}}$.
    And obviously $\gF_*$ is star-shaped, meaning that for all $f \in \gF_*$ and all $\beta \in [0,1]$, $\beta f \in \gF_*$.
    Let $t^2 = b \cdot \sqrt{\frac{8 B^2}{n (1-\epsilon)} S_\lambda (d+1)} \ge b \lrad_n(\delta; \gF_*)$.
    Then, by \citet[Theorem~14.1]{wainwright2019high}, we have
\begin{equation}
    \sP \left [ \left | \|f \|_n^2 - \|f \|_\px^2 \right |  \ge \frac{1}{2} \| f \|_\px^2 + \frac{t^2}{2}  \right ] \le c_1 \exp \left ( - c_2 \frac{n t^2}{b^2} \right ) \qquad \text{for all } f \in \gF_* 
\end{equation}
for universal constant $c_1, c_2$.
We know that $\hat{f}-f^* \in \gF_*$ and $\fproj - f^* \in \gF_*$, which means that

\begin{equation}
\begin{aligned}
  & \sP \left [ \left (\| \hat{f} - f^* \|_\px^2 \ge 2 \| \hat{f} - f^* \|_n^2 + t^2 \right ) \vee \left ( \| \fproj - f^*\|_n^2 \ge \frac{3}{2}\| \fproj - f^*\|_\px^2 + \frac{t^2}{2} \right ) \right ]  \\ 
  & \le c_1 \exp \left (-c_2 \frac{nt^2}{b^2} \right )  .  
\end{aligned}
\end{equation}

Let $\delta_n^2 = 2\sigma \sqrt{\frac{8 \bt^2 B^2}{n(1-\epsilon)}}$.
By Lemma \ref{lem:13-22}, we have $\delta_n^2 \ge 2 \sigma \gG_n(\delta_n; \gF_0)$.
And $\gF_0$ is also star-shaped.
Thus, by setting $\gamma = 1/2$ in \citet[Theorem~13.13]{wainwright2019high}, we have\footnote{Please refer to the proof of \citet[Theorem~13.13]{wainwright2019high} for removing the universal constants in this theorem.}
\begin{equation}
\label{eqn:sim-135}
    \sP \left [ \| \hat{f} - f^* \|_n^2 \ge 3 \| \fproj - f^*\|_n^2 + 32  \delta_n^2   \right ] \le \exp \left ( - \frac{n  \delta_n^2}{2 \sigma^2}  \right )  .
\end{equation}

Combining the two inequalities above with the union bound,
we obtain the result.
\end{proof}

Now we prove Lemma \ref{lem:f_diff}.
Without loss of generality, suppose $h_1,\cdots,h_{d'}$ are linearly independent.
Let $\hat{\gH}_{d'} := \sspan \set{h_1,\cdots,h_{d'}}$.
Let $g^* = g_0 + \beta g_1$,
where $g_0 = \Pi_{\hat{\gH}_{d'}} g^*$,
$g_1 \perp g_0$,
and $\|g_1\|_\pa = 1$.
So by Lemma \ref{lem:abel-trace}, we have:
\begin{prop}
\label{prop:lda}
    $\| \tax^* (\mG_h^{-1/2}h_1)\|_\px^2 + \cdots + \| \tax^* (\mG_h^{-1/2}h_{d'})\|_\px^2 + \| \tax^* g_1 \|_\px^2 \le \lambda_1 + \cdots + \lambda_{d'+1}$.
\end{prop}
\begin{proof}
    Let $\left [ \mG_h^{-1/2}h_1, \cdots, \mG_h^{-1/2}h_{d'}, g_1 \right ] = \mQ \Phistar$, where $\mQ$ is a matrix with $(d'+1)$ orthonormal rows.
    Then, $\left [ \tax^*( \mG_h^{-1/2}h_1), \cdots, \tax^* (\mG_h^{-1/2}h_{d'}), \tax^* g_1 \right ] = \mQ \dlbd^{1/2} \Phistar$.
    Thus, we have
\begin{equation*}
\| \tax^* (\mG_h^{-1/2}h_1)\|_\px^2 + \cdots + \| \tax^* (\mG_h^{-1/2}h_{d'})\|_\px^2 + \| \tax^* g_1 \|_\px^2 = \Tr(\mQ \dlbd \mQ^{\top}).
\end{equation*}
Then, applying Lemma \ref{lem:abel-trace} completes the proof.
\end{proof}
\begin{remark}
    This proposition is the functional version of \citet[Theorem~1]{fan1949theorem}.
\end{remark}

Notice that $\| \tax^* (\mG_h^{-1/2}h_1)\|_\px^2 + \cdots + \| \tax^* (\mG_h^{-1/2}h_{d'})\|_\px^2 = \Tr(\mG_h^{-1/2} \mF_h \mG_h^{-1/2}) = \Tr(\mG_h^{-1} \mF_h)$.
With this, we can prove Lemma \ref{lem:f_diff}:
\paragraph{Lemma \ref{lem:f_diff}.}
\textit{
    For any $f^* \in \Bdowne$, there is
    \begin{equation*}
       \|\fproj-f^*\|_\px^2 \le   \frac{\tau^2 }{1 - \tau^2} \frac{\tau + \epsilon}{1-\epsilon} B^2 .
    \end{equation*}
}
\begin{proof}
    Let $\alpha^2 = \|g_0\|_\pa^2$,
    and $\beta^2 = \|g_0 - g^*\|_\pa^2$.
    By Corollary \ref{cor:con-1}, $\alpha^2 + \beta^2 \le \frac{B^2}{1-\epsilon}$.
    Eqn. (\ref{eqn:lap-equ}) implies that
\begin{equation*}
    (1-\epsilon)(\alpha^2 + \beta^2) \le \| \tax^* (g_0 + \beta g_1) \|_\px^2 \le \alpha^2 + \beta^2 \tau^2 + 2 \alpha \beta \tau  ,
\end{equation*}
since $\|\tax^* g_0\|_\px^2 \le \|g_0\|_\pa^2 = \alpha^2$,
and $\|\tax^* g_1\|_\px^2 \le \tau^2$ by Proposition \ref{prop:lda}.
Thus,
\begin{equation*}
    (1-\tau^2) \beta^2 \le \epsilon (\alpha^2 + \beta^2) + 2 \alpha \beta \tau \le (\epsilon + \tau) (\alpha^2 + \beta^2) \le (\epsilon + \tau) \frac{B^2}{1-\epsilon}  .
\end{equation*}
Thus, we have $\|\fproj-f^*\|_\px^2 = \| \tax^* (g_0 - g^*) \|_\px^2  = \beta^2 \|\tax^* g_1\|_\px^2 \le \beta^2 \tau^2 $, which leads to the inequality we need to prove.
Finally, by setting $h_i = \hat{\phi}_i$, we can see that $ \tau^2 \le S_\lambda(d+1) - \Tr(\mG^{-1} \mF)$.
And for all $d' \le d$, $\Tr(\mG_h^{-1} \mF_h) \le S_\lambda(d')$,
so $\tau^2 \ge \lambda_{d+1}$.
\end{proof}

%% file: proofs_sec4.tex
\section{Proofs for Section 4}

\paragraph{Proposition \ref{prop:reg}.}
\textit{
   For any $\hPsi=[\hat{\psi}_1,\cdots,\hat{\psi}_d]$ where $\hat{\psi}_i \in \lxp$, it holds that
    \begin{equation}
    \label{eqn:guarantee-1-restated}
     \err(\hPsi; \Bdowne) \ge \frac{\lambda_{d+1}}{1 - \lambda_{d+1}} \frac{\epsilon}{1 - \epsilon}  B^2 \quad \text{given that} \quad  \frac{\lambda_{d+1}}{1 - \lambda_{d+1}} \frac{\epsilon}{1 - \epsilon} \le \frac{1}{2}  .
    \end{equation}
    To attain equality, it is sufficient for $\hPsi$ to span the top-$d$ eigenspace, and also necessary if $\lambda_{d+1} < \lambda_d$.
}

\begin{proof}
\textit{Necessity:}
Since $\hat{\Psi}$ is at most rank-$d$,
there must be a function in $\sspan \set{\psi_1,\cdots,\psi_{d+1}}$ that is orthogonal to $\hat{\Psi}$.
Thus, we can find two functions $f_1,f_2 \in \sspan \set{\psi_1,\cdots,\psi_{d+1}}$ such that: $\| f_1 \|_\px = \| f_2 \|_\px = 1$,
$f_1$ is orthogonal to $\hat{\Psi}$,
$f_2 = \vu^{\top} \hat{\Psi}$ (which means that $f_2 \perp f_1$),
and $\psi_1 \in \sspan \set{f_1, f_2}$.
Recall that $\lambda_1 = 1$, and $\psi_1 \equiv 1$.
Let $\psi_1 = \alpha_1 f_1 + \alpha_2 f_2$, then $\alpha_1^2 + \alpha_2^2 = 1$.
Without loss of generality, suppose $\alpha_1,\alpha_2 \in [0,1]$.
Let $f_0 = \alpha_2 f_1 - \alpha_1 f_2$.
Then, $\| f_0 \|_\px = 1$, $f_0 \perp \psi_1$.
Note that we also have $\langle \psi_1, f_0 \rangle_\laph = 0$ by duality.
Let $\beta_1, \beta_2 \in [0,1]$ be any value such that $f = \beta_1 \psi_1 + \beta_2 f_0$ satisfies $\| f \|_\px^2 = \beta_1^2 + \beta_2^2 = 1$, and $\|f \|_\laph^2 \le \frac{1}{1-\epsilon}$.
This is satisfied as long as $\beta_2^2 \le \frac{\epsilon}{1-\epsilon} \frac{\lambda_{d+1}}{1 - \lambda_{d+1}}$,
because $ \|f \|_\laph^2 \le \beta_1^2 + \frac{\beta_2^2}{\lambda_{d+1}} = 1 + \frac{1 - \lambda_{d+1}}{\lambda_{d+1}} \beta_2^2 \le \frac{1}{1-\epsilon}$.
Moreover, we have $Bf \in \Bdowne$.

It is easy to show that $F(\alpha_1) = \alpha_1 \beta_1 + \alpha_2 \beta_2 = \alpha_1 \beta_1 + \sqrt{1 - \alpha_1^2} \beta_2 \  (\alpha_1 \in [0,1])$ first increases then decreases, so $F(\alpha_1)^2 \ge \min \set{F(0)^2, F(1)^2} = \min \set{\beta_1^2, \beta_2^2}$, which can be $\frac{\epsilon}{1-\epsilon} \frac{\lambda_{d+1}}{1 - \lambda_{d+1}}$ in the worst case given that it is at most $\frac{1}{2}$,
in which case the prediction error of $Bf$ is $\| B (\alpha_1 \beta_1 + \alpha_2 \beta_2)f_1 \|_\px^2 =  F(\alpha_1)^2 B^2 = \frac{\epsilon}{1 - \epsilon} \frac{\lambda_{d+1}}{1 - \lambda_{d+1}} B^2$.
Thus, for any $\hat{\Psi}$, we can find a function $Bf \in \Bdowne$ such that $\min_{w} \err(w^{\top} \hat{\Psi}, Bf) \ge \frac{\epsilon}{1 - \epsilon} \frac{\lambda_{d+1}}{1 - \lambda_{d+1}} B^2$.

When $\lambda_d > \lambda_{d+1}$, to attain equality,
we need $\alpha_1 = 0$, and $\| f \|_\laph^2 = \beta_1^2 + \frac{\beta_2^2}{\lambda_{d+1}}$,
which means that $f_0 = \psi_{d+1}$.
Thus, only $f_1 = f_0 = \psi_{d+1}$ is orthogonal to $\hat{\Psi}$, so $\hat{\Psi}$ must span the top-$d$ eigenspace.

    \textit{Sufficiency:}
Suppose $\hat{\Psi}$ spans the top-$d$ eigenspace.
For any $f \in \Bdowne$ such that $f = \sum_i u_i \psi_i$,
we have $\sum_i u_i^2 \le B^2$, and $\sum_i \frac{1-\epsilon-\lambda_i}{\lambda_i} u_i^2 \le 0$.
Let $a = \sum_{i \ge d+1} u_i^2$ and $b = \sum_{i=1}^d u_i^2$. Then, $a = \min_{w} \err(w^{\top} \hat{\Psi}, f)$, and $a + b \le B^2$.
So we have
\begin{align*}
    0 & \ge \sum_i \frac{1-\epsilon-\lambda_i}{\lambda_i} u_i^2 \ge -\epsilon b + \frac{1 - \epsilon - \lambda_{d+1}}{\lambda_{d+1}} a \quad \left (\text{since } \frac{1-\epsilon-\lambda}{\lambda}\text{ decreases with }\lambda \right ) \\ 
    & \ge - \epsilon (B^2 - a) + \frac{1 - \epsilon - \lambda_{d+1}}{\lambda_{d+1}} a  \\ 
    &= -\epsilon B^2 + (1-\epsilon) \frac{1-\lambda_{d+1}}{\lambda_{d+1}}  a  ,
\end{align*}
which combined with the necessity part implies that $\err(\hPsi; \Bdowne) = \frac{\epsilon}{1 - \epsilon} \frac{\lambda_{d+1}}{1 - \lambda_{d+1}} B^2$.
\end{proof}

\paragraph{Lemma \ref{lem:ratio-trace-gen}.}
\textit{
Suppose there exists a constant $C > 0$ such that $\E_\pa[g^4] \le C^2 \| g \|_\pa^2$, for all $g = w^{\top} \hat{\Phi}$ where $\|g\|_\pa \le 1$.
Then, for any $\delta > 0$, it holds with probability at least $1 - \delta$ that
\begin{equation*}
| \Tr(\mhG^{-1} \mhF) - \Tr(\mG^{-1} \mF) | \le \left (2 + \sqrt{2 \log \frac{2}{\delta}} \right ) \frac{ C \bt + \bt^2}{\sqrt{N}} d  .
\end{equation*}
}

\begin{proof}
Since multiplying an invertible $d \times d$ matrix to $\hat{\Phi}$ does not change either $\Tr(\mhG^{-1} \mhF)$ or $\Tr(\mG^{-1} \mF)$,
for simplicity let us multiply $\mG^{-1/2}$ to $\hat{\Phi}$,
so that $\langle \hat{\phi}_i, \hat{\phi}_j \rangle_\pa = \delta_{i,j}$ for all $i,j \in [d]$ (\ie{} $\mG = \mI$).
Define $\gF_1 = \sset{f \in \laph}{\| f\|_{\laph} \le 1}$.
Its Rademacher complexity is given by
\begin{equation}
    \rad_N (\gF_1) = \underset{x_1,\cdots,x_N}{\E} \underset{\sigma_1,\cdots,\sigma_N}{\E} \left [ \sup_{f \in \gF_1} \frac{1}{N} \sum_{k=1}^N \sigma_k f(x_k) \right ].
\end{equation}
By \citet[Theorem~6.12]{mohri2018foundations},
we have $\rad_N (\gF_1) \le \bt N^{-1/2}$.
Moreover, by Proposition \ref{prop:m-bound},
all $f \in \gF_1$ satisfy $|f(x)| \le \bt$ for all $x$.
Thus, by \citet[Theorem~4.10]{wainwright2019high}, for any $\delta > 0$, with probability at least $1 - \delta / 2$, it holds for all $f \in \gF_1$ that
\begin{equation}
\label{eqn:rad-fc-restated}
    \left | \frac{1}{N} \sum_{k=1}^N f(x_k) - \E[f(X)]  \right | \le 2 \rad_N (\gF_1) + \bt  \sqrt{\frac{2}{N} \log \frac{2}{\delta}} \le \left (2 + \sqrt{2 \log \frac{2}{\delta}} \right ) \frac{ \bt}{\sqrt{N}}.
\end{equation}

Define matrix $\mM = \mhG^{-1/2} \mhF \mhG^{-1/2} = (m_{i,j})_{i,j \in [d]}$. $\|\mM \|_2 \le 1$, so $\sum_{i=1}^d m_{i,j}^2 \le 1$ for all $j \in [d]$.
Consider $\Tr((\mI - \mhG) \mM)$.
For any $j \in [d]$, we have
\begin{equation*}
    ((\mI - \mhG) \mM)(j,j) = \left \langle \hat{\phi}_j, \sum_{i=1}^d m_{i,j} \hat{\phi}_i \right \rangle_{\hpa} - \left \langle \hat{\phi}_j, \sum_{i=1}^d m_{i,j} \hat{\phi}_i \right \rangle_{\pa}   .
\end{equation*}

Note that $\left \| \sum_{i=1}^d m_{i,j} \hat{\phi}_i \right \|_\pa \le 1$, so $\left \| \hat{\phi}_j \left ( \sum_{i=1}^d m_{i,j} \hat{\phi}_i \right ) \right \|_\pa^2 \le \sqrt{\E[\hat{\phi}_j^4] \E \left[ \left ( \sum_{i=1}^d m_{i,j} \hat{\phi}_i \right )^4 \right ]} \le C^2$, which means that $ C^{-1} \tax^* \left ( \hat{\phi}_j \left ( \sum_{i=1}^d m_{i,j} \hat{\phi}_i \right ) \right ) \in \gF_1$.
So if Eqn. (\ref{eqn:rad-fc-restated}) holds, then for all $j \in [d]$, we have
\begin{align*}
    ((\mI - \mhG) \mM)(j,j) & = \left | \frac{1}{N} \sum_{k=1}^N \tax^* \left (  \hat{\phi}_j \left ( \sum_{i=1}^d m_{i,j} \hat{\phi}_i \right ) \right )(x_k) - \E \left [ \tax^* \left (  \hat{\phi}_j \left ( \sum_{i=1}^d m_{i,j} \hat{\phi}_i \right ) \right ) (X) \right ] \right | \\ 
    & \le \left (2 + \sqrt{2 \log \frac{2}{\delta}} \right ) \frac{ C \bt}{\sqrt{N}},
\end{align*}

which implies that
\begin{align*}
    \Tr \left (\mhG^{-1} \mhF - \mhF \right ) = \Tr \left ( \mhG^{-1/2} ( \mI - \mhG  ) \mhG^{-1/2} \mhF \right ) = \Tr \left (  (\mI - \mhG) \mM \right )  \le \left (2 + \sqrt{2 \log \frac{2}{\delta}} \right ) \frac{ C \bt d}{\sqrt{N}} .
\end{align*}

Next, define $\gF_2 = \sset{f_1f_2}{f_1, f_2 \in \laph, \|f_1\|_\laph \le 1, \|f_2 \|_\laph \le 1}$.
By Proposition \ref{prop:emp-rad-bound} (proved after this lemma), we have $\rad_N (\gF_2) \le \bt^2 N^{-1/2}$.
And all $f \in \gF_2$ satisfy $|f(x)| \le \bt^2$ for all $x$ by Proposition \ref{prop:m-bound}.
So with probability at least $1 - \delta/2$, we have for all $f \in \gF_2$,
\begin{equation}
\label{eqn:rad-f1}
    \left | \frac{1}{N} \sum_{k=1}^N f(x_k) - \E[f(X)]  \right | \le 2 \rad_N (\gF_2) +  \bt^2  \sqrt{\frac{2}{N} \log \frac{2}{\delta}} \le \left (2 + \sqrt{2 \log \frac{2}{\delta}} \right ) \frac{  \bt^2}{\sqrt{N}}.
\end{equation}

Note that $\| \hat{\psi}_i \|_\laph \le 1$. 
So under Eqn. (\ref{eqn:rad-f1}), we have for all $i,j \in [d]$,
\begin{equation*}
    \left | \langle \hat{\psi}_i, \hat{\psi}_j \rangle_\hpx -\langle \hat{\psi}_i, \hat{\psi}_j \rangle_\px  \right | = \left | \frac{1}{N} \sum_{k=1}^N \hat{\psi}_i(x_k) \hat{\psi}_j(x_k) - \E[\hat{\psi}_i \hat{\psi}_j] \right | \le \left (2 + \sqrt{2 \log \frac{2}{\delta}} \right ) \frac{  \bt^2}{\sqrt{N}},
\end{equation*}
which implies that $\Tr \left ( \mhF - \mG^{-1} \mF \right ) = \Tr \left ( \mhF - \mF \right ) \le  \left (2 + \sqrt{2 \log \frac{2}{\delta}} \right ) \frac{  \bt^2 d}{\sqrt{N}}$.

Finally, applying the union bound completes the proof.
\end{proof}

\begin{prop}
\label{prop:emp-rad-bound}
Let $\gF_2 = \sset{f_1f_2}{f_1, f_2 \in \laph, \|f_1\|_\laph \le 1, \|f_2 \|_\laph \le 1}$.
Then,
    $\rad_N(\gF_2) \le \frac{\bt^2}{\sqrt{N}}$.
\end{prop}

\begin{proof}
 For any $h(x) = f_1(x) f_2(x) \in \gF_2$, let $f_1 = \tax^* g_1$ and $f_2 = \tax^* g_2$, where $\| g_1\|_\pa \le 1$ and $\|g_2\|_\pa \le 1$. Let $g_1 = \sum_i u_i \phi_i$ and $g_2 = \sum_i v_i \phi_i$. Let $\vu = [u_1, u_2, \cdots]$ and $\vv = [v_1, v_2, \cdots]$. Then, $\| \vu \|_2 \le 1$ and $\| \vv \|_2 \le 1$. And we have $f_1 = \sum_i \lambda_i^{1/2} u_i \psi_i$, and $f_2 = \sum_i \lambda_i^{1/2} v_i \psi_i$.

For any $x \in \gX$, let $\Psi(x) = [\lambda_1^{1/2} \psi_1(x), \lambda_2^{1/2} \psi_2(x), \cdots ]$.  Then, $f_1(x) = \vu^{\top} \Psi(x)$ and $f_2(x) = \vv^{\top} \Psi(x)$.
Denote $\Psi_k = \Psi(x_k)$.
Then, $\Psi_k^{\top} \Psi_k \le \bt^2$ for all $k \in [N]$.
So for any $S = \{ x_1,\cdots,x_N\}$,
the empirical Rademacher complexity satisfies
\begin{align*}
    \hat{\rad}_S(\gF_2) & \le \underset{\vsigma}{\E} \left [ \sup_{\|\vu\|_2 \le 1, \| \vv \|_2 \le 1} \left | \frac{1}{N} \sum_{k=1}^N \sigma_k \vu^{\top} \Psi_k \Psi_k^{\top} \vv \right |  \right ] \\ 
    & \le \frac{1}{N} \underset{\vsigma}{\E} \left [  \left \| \sum_{k=1}^N \sigma_k \Psi_k \Psi_k^{\top}  \right  \|_2   \right ] \\ 
    & \le \frac{1}{N} \underset{\vsigma}{\E} \left [  \left \| \sum_{k=1}^N \sigma_k \Psi_k \Psi_k^{\top}  \right  \|_F   \right ] \\ 
    & = \frac{1}{N} \underset{\vsigma}{\E} \left [ \Tr \left ( \left ( \sum_{k=1}^N \sigma_k \Psi_k \Psi_k^{\top} \right )^{\top} \left ( \sum_{l=1}^N \sigma_l \Psi_l \Psi_l^{\top} \right ) \right )^{1/2}  \right ] \\ 
    & \le \frac{1}{N}   \sqrt{ \underset{\vsigma}{\E} \left [  \Tr \left ( \sum_{k,l=1}^N \sigma_k \sigma_l \Psi_k \Psi_k^{\top}  \Psi_l \Psi_l^{\top}   \right )    \right ]  } \qquad (\text{Jensen}) \\ 
    & = \frac{1}{N}\sqrt{ \Tr \left ( \sum_{k,l=1}^N  \E[\sigma_k \sigma_l]   \Psi_k \Psi_k^{\top}  \Psi_l \Psi_l^{\top} \right ) } \\ 
    & = \frac{1}{N} \sqrt{\Tr \left ( \sum_{k=1}^N  \Psi_k \Psi_k^{\top} \Psi_k \Psi_k^{\top} \right )} \\ 
    & \le \frac{1}{N} \sqrt{N \bt^4} = \frac{\bt^2}{\sqrt{N}}  .
\end{align*}
Then, since $\rad_N(\gF_2) = \E_S [\hat{\rad}_S(\gF_2)]$, we obtain the result.
\end{proof}

\paragraph{Lemma \ref{lem:lambda-diff}.}
\textit{
    Suppose $\hat{\phi}_i = \bar{\phi}_i$ for $i \in [d]$.
    Let $\gamma_{\mG} := \lambda_{\max}(\mG) / \lambda_{\min}(\mG)$,
    which is the condition number of $\mG$.
    Then, for any $\delta > 0$, both
    \begin{equation*}
        \sum_{j=1}^{d} \bar{\lambda}_j \ge \sum_{i=1}^{d} \lambda_i - \left (2 + \sqrt{2 \log \frac{2}{\delta}} \right ) \frac{ (\lambda_d^{-1}+1) \bt^2 }{\sqrt{N}} d
    \end{equation*}
    and Eqn. (\ref{eqn:ratio-trace-gen}) with $C = \bt \bar{\lambda}_d^{-1} \gamma_{\mG}^{1/2}$ hold simultaneously for $\learnh = \ephd$ with probability at least $1 - \delta$.
}
    
\begin{proof}
Denote $\Phistard = [\phi_1,\cdots,\phi_d]$ and $\hPhistard = [\bar{\phi}_1,\cdots,\bar{\phi}_d]$.
Let $\hPhistard =\mP \Phistar$,
where $\mP$ is a matrix with $d$ rows.
Observe that for any $g = \sum_i u_i \bar{\phi}_i$ such that $\|g\|_\pa \le 1$, 
we have $g = \bar{\tax} \tax^* \left (\sum_i \bar{\lambda}_i^{-1}  u_i \bar{\phi}_i  \right )$.
Let $\vu = (u_1,\cdots,u_d)$,
then there is $g = \vu^{\top} \hPhistard = \vu^{\top} \mP \Phistar$,
so $\| \mP^{\top} \vu\|_2 \le 1$.
Thus, we have $\|\sum_i \bar{\lambda}_i^{-1}  u_i \bar{\phi}_i\|_\pa = \|\mP^{\top} \mD_{\bar{\lambda}^d}^{-1} \vu \|_2 = \|\mP^{\top} \mD_{\bar{\lambda}^d}^{-1} (\mP \mP^{\top})^{-1} \mP \mP^{\top} \vu \|_2$.

So we just need to show that $\|\mP^{\top} \mD_{\bar{\lambda}^d}^{-1} (\mP \mP^{\top})^{-1} \mP   \|_2 \le \bar{\lambda}_d^{-1} \gamma_{\mG}^{1/2}$.
$\|\mP^{\top} \mD_{\bar{\lambda}^d}^{-1} (\mP \mP^{\top})^{-1} \mP \|_2$ is equal to the square root of the largest eigenvalue of $\mP^{\top} \mD_{\bar{\lambda}^d}^{-1}  (\mP \mP^{\top})^{-1} \mD_{\bar{\lambda}^d}^{-1} \mP$,
and by using two simple linear algebra exercises: (i) $\lambda_{\max}(\mA \mB) \le \lambda_{\max}(\mA) \lambda_{\max}(\mB)$ for positive definite matrices $\mA$ and $\mB$, and (ii) $\mA \mB$ and $\mB \mA$ share the same non-zero eigenvalues (Sylvester's Theorem),
and the fact that $\mG = \mP \mP^{\top}$,
we can show that the largest eigenvalue of this matrix is at most $\bar{\lambda}_d^{-2} \gamma_{\mG}$.

Therefore, we have $\|\mP^{\top} \mD_{\bar{\lambda}^d}^{-1} (\mP \mP^{\top})^{-1} \mP \|_2 \le \bar{\lambda}_d^{-1} \gamma_{\mG}^{1/2}$,
which combined with $\| \mP^{\top} \vu\|_2 \le 1$ implies that $\|\sum_i \bar{\lambda}_i^{-1}  u_i \bar{\phi}_i\|_\pa \le \bar{\lambda}_d^{-1} \gamma_{\mG}^{1/2}$.
By Proposition \ref{prop:m-bound}, $|\tax^* \left (\sum_i \bar{\lambda}_i^{-1}  u_i \bar{\phi}_i  \right ) (x)| \le \bt \bar{\lambda}_d^{-1} \gamma_{\mG}^{1/2}$ for all $x$,
so we have $|\bar{\tax} \tax^* \left (\sum_i \bar{\lambda}_i^{-1}  u_i \bar{\phi}_i  \right )(a)| = |\int \tax^* \left (\sum_i \bar{\lambda}_i^{-1}  u_i \bar{\phi}_i  \right )(x) p(x|a) dx | \le \bt \bar{\lambda}_d^{-1} \gamma_{\mG}^{1/2}$ for all $a$.
This means that with $C = \bt \bar{\lambda}_d^{-1} \gamma_{\mG}^{1/2}$,
$g$ satisfies the condition of Lemma \ref{lem:ratio-trace-gen}.
Therefore, with probability at least $1-\delta$, both Eqn. (\ref{eqn:rad-fc-restated}) and Eqn. (\ref{eqn:rad-f1}) hold and they lead to Eqn. (\ref{eqn:ratio-trace-gen}).

Now let $\Phistard =\mQ \hPhistar$, where $\mQ$ is a matrix with $d$ rows.
Consider two matrices $\mQ \mQ^{\top}, \mQ \mD_{\bar{\lambda}} \mQ^{\top} \in \R^{d \times d}$ where $\mD_{\bar{\lambda}} = \diag(\bar{\lambda}_1, \bar{\lambda}_2, \cdots)$, 
for which we have
\begin{equation*}
    (\mQ \mQ^{\top}) (i,j) = \langle \phi_i, \phi_j \rangle_\hpa \qquad \text{and} \qquad (\mQ \mD_{\bar{\lambda}} \mQ^{\top}) (i,j) = \langle \tax^* \phi_i, \tax^* \phi_j \rangle_\hpx  .
\end{equation*}
We have $(\langle \phi_i, \phi_j \rangle_\pa)_{i,j \in [d]} = \mI$ and $( \langle \tax^* \phi_i, \tax^* \phi_j \rangle_\px)_{i,j \in [d]} = \mD_{\lambda^d} := \diag(\lambda_1,\cdots,\lambda_d)$.
Moreover, for any $g = \vu^{\top} \Phistard$ such that $\| g \|_\pa \le 1$,
there is $g = \tax \tax^* \left ( \sum_i \lambda_i^{-1} u_i \phi_i \right )$,
and obviously $\| \sum_i \lambda_i^{-1} u_i \phi_i\|_\pa \le \lambda_d^{-1}$.
Thus, we can show that for all $a$, $|g(a)| \le \bt \lambda_d^{-1}$,
which means that $\Phistard$ satisfies the fourth-moment control assumption in Lemma \ref{lem:ratio-trace-gen} with $C' = \bt \lambda_d^{-1}$.
So similar to the proof of Lemma \ref{lem:ratio-trace-gen},
for all $\vu \in \R^d$ such that $\| \vu \|_2 \le 1$,
we can show that
\begin{equation*}
    \left | \vu^{\top} (\mQ \mQ^{\top} - \mI) \vu \right | = \left | \left \langle \vu^{\top} \Phistard, \vu^{\top} \Phistard \right \rangle_{\hpa} - \left \langle \vu^{\top} \Phistard, \vu^{\top} \Phistard \right \rangle_{\pa} \right | \le \left (2 + \sqrt{2 \log \frac{2}{\delta}} \right ) \frac{  \bt^2 \lambda_d^{-1} }{\sqrt{N}} ,
\end{equation*}
which implies that $\|\mQ \mQ^{\top} \|_2 \le 1+ \left (2 + \sqrt{2 \log \frac{2}{\delta}} \right ) \frac{\bt^2 \lambda_d^{-1}}{\sqrt{N}}$ .
It is easy to show that all non-zero eigenvalues of $\mQ^{\top} \mQ$ are also eigenvalues of $\mQ \mQ^{\top}$, so $\| \mQ^{\top} \mQ \|_2 \le 1+ \left (2 + \sqrt{2 \log \frac{2}{\delta}} \right ) \frac{\bt^2 \lambda_d^{-1}}{\sqrt{N}}$.
Moreover, similar to the proof of Lemma \ref{lem:ratio-trace-gen}, we can show that for all $i,j \in [d]$,
\begin{numcases}{}
    \left | \left ( \mQ \mQ^{\top} - \mI \right ) (i,j) \right | \le \left (2 + \sqrt{2 \log \frac{2}{\delta}} \right ) \frac{ \bt^2 \lambda_d^{-1}}{\sqrt{N}}  ;  \label{eqn:proof-lem-lbd-1} \\ 
    \left | \left ( \mQ \mD_{\bar{\lambda}} \mQ^{\top} - \mD_{\lambda^d} \right ) (i,j) \right | \le \left (2 + \sqrt{2 \log \frac{2}{\delta}} \right ) \frac{  \bt^2}{\sqrt{N}} . \label{eqn:proof-lem-lbd-2}
\end{numcases}

Let $\vq_i$ be the $i$-th column of $\mQ$.
Then for all $i \in [d]$, $\vq_i^{\top} \vq_i \le 1+ \left (2 + \sqrt{2 \log \frac{2}{\delta}} \right ) \frac{\bt^2 \lambda_d^{-1}}{\sqrt{N}}$.
And we also have $\sum_{i=1}^{\infty} \vq_i^{\top} \vq_i = \Tr(\mQ^{\top} \mQ) = \Tr(\mQ \mQ^{\top}) \le d + \left (2 + \sqrt{2 \log \frac{2}{\delta}} \right ) \frac{ \bt^2 \lambda_d^{-1} d}{\sqrt{N}}$.
Thus, we have
\begin{align*}
    & \sum_{i=1}^d \lambda_i - \left (2 + \sqrt{2 \log \frac{2}{\delta}} \right ) \frac{ \bt^2}{\sqrt{N}} d  \le \Tr(\mQ \mD_{\bar{\lambda}} \mQ^{\top}) = \Tr(\mD_{\bar{\lambda}} \mQ^{\top} \mQ ) \\ 
    = \; & \sum_{i=1}^{\infty} \bar{\lambda}_i \vq_i^{\top} \vq_i  =  \sum_{j=1}^\infty \left (\sum_{i=1}^j \vq_i^{\top} \vq_i \right ) (\bar{\lambda}_j - \bar{\lambda}_{j+1}) \\ 
    \le \; & \sum_{i=1}^d \bar{\lambda}_i \left [ 1+ \left (2 + \sqrt{2 \log \frac{2}{\delta}} \right ) \frac{\bt^2 \lambda_d^{-1}}{\sqrt{N}} \right ] \le \sum_{i=1}^d \bar{\lambda}_i + \left (2 + \sqrt{2 \log \frac{2}{\delta}} \right ) \frac{ \bt^2 \lambda_d^{-1} d }{\sqrt{N}}  ,
\end{align*}
which proves the assertion.
\end{proof}

\section{Proofs for Section \ref{sec:bt-big}}
\label{sec:proofs_hypercube_examples}

\paragraph{Example 1.}
\textit{
Consider $\gX = \{-1, 1 \}^{d_\gX}$, where each $x \in \gX$ is a vector of $-1$ and $1$ with length $d_\gX$.
Let $p_{\gX}$ be the uniform distribution over $\gX$.
Consider a random masking augmentation, where for any $x \in \gX$, each coordinate $x^i$ is randomly masked to be $0$
with probability $\alpha \in (0,1)$ independently,
where $\alpha$ is the mask ratio ($0$ is the \texttt{[MASK]} token).
Then, $\bt^2 = (2 - \alpha)^{d_\gX}$.
}

\begin{proof}
We know that
$\bt^2 \ge \int \frac{p(a|x)^2}{p(a)} da$,
and by symmetry, the right-hand-side is the same for all $x$.
Given an $a$, suppose $a$ has $r$ coordinates masked and $(d_{\gX} - r)$ coordinates unmasked.
Then, there are $2^r$ possible $x$ that can be augmented to $a$.
For each of these $x$, $p(a|x) = \alpha^r (1-\alpha)^{d_{\gX} - r}$.
So $p(a) = \int p(a|x) p(x) dx = 2^{r - d_{\gX}} \alpha^r (1-\alpha)^{d_{\gX} - r}$.
Thus, we have
\begin{align*}
   \bt^2 & = \int \frac{p(a|x)^2}{p(a)} da = \sum_{r=0}^{d_\gX} \binom{d_\gX}{r} \frac{\alpha^{2r} (1-\alpha)^{2d_{\gX} - 2r}}{2^{r - d_{\gX}} \alpha^r (1-\alpha)^{d_{\gX} - r}} \\ 
   & = \sum_{r=0}^{d_\gX} \binom{d_\gX}{r} \alpha^r (2-2\alpha)^{d_\gX - r} \\ 
   & = (\alpha + 2 - 2 \alpha)^{d_\gX} = (2 - \alpha)^{d_\gX} ,
\end{align*}
which completes the proof.
\end{proof}

\paragraph{Example 2.}
\textit{
Consider a random block masking augmentation with ratio $\alpha$, which for any $x \in \gX$ masks $x^i,x^{i+1},\cdots,x^{i+r-1}$ for a uniformly random $i$, and $r = \lceil \alpha d_{\gX} \rceil$.
Then, $\bt^2 \le [2^{(1-\alpha) }]^{d_{\gX}}$.
}

\begin{proof}
For any $a$, we have $p(a) = \frac{1}{d_\gX - r + 1} \frac{1}{2^{d_\gX - r}}$, and $p(a|x) = \frac{1}{d_\gX - r + 1}$ if $x$ can be augmented to $a$.
So there always is $\frac{p(a|x)}{p(a)} = 2^{d_\gX - r} \le 2^{(1-\alpha) d_\gX}$. Thus, we have $\bt^2 \le 2^{(1-\alpha)d_\gX}$.
\end{proof}

\paragraph{Example 3.}
\textit{
Consider random block masking $+$ flipping with ratio $\alpha$, where for any $x \in \gX$, first $x^i,\cdots,x^{i+r-1}$ are masked to be $0$ for a uniformly random $i$ and $r = \lceil \frac{\alpha}{2} d_{\gX} \rceil$, and then each remaining coordinate is randomly flipped sign ($1 \rightarrow -1$ and $-1 \rightarrow 1$) with probability $\frac{\alpha}{2}$ independently.
Then, $\bt^2 \le \left [(\alpha^2-2\alpha+2)^{(1-\alpha/2)} \right ]^{d_{\gX}}$.
}

\begin{proof}
For any $a$, we have $p(a) = \frac{1}{d_\gX - r + 1} \frac{1}{2^{d_\gX - r}}$.
Suppose $a$ is augmented from $x$, and among the unmasked $(d_\gX - r)$ coordinates, $a$ and $x$ have $k$ disagreeing coordinates.
For a given $k$,
there are $(d_\gX - r + 1) \binom{d_\gX - r}{k}$ possible $a$,
and we have $p(a|x) = \frac{1}{d_\gX - r + 1}  (\frac{\alpha}{2})^{k} (1 - \frac{\alpha}{2})^{d_\gX - r - k}$.
Thus, we have
\begin{align*}
    \int \frac{p(a|x)^2}{p(a)} da & = \sum_{k=0}^{d_\gX - r} (d_\gX - r + 1) \binom{d_\gX - r}{k} \frac{\frac{1}{(d_\gX - r + 1)^2}(\frac{\alpha}{2})^{2k} (1 - \frac{\alpha}{2})^{2d_\gX - 2r - 2k}}{\frac{1}{d_\gX - r + 1} \frac{1}{2^{d_\gX - r}}} \\ 
   & = \sum_{k=0}^{d_\gX - r}\binom{d_\gX - r}{k} 2^{d_\gX - r} \left ( \frac{\alpha^2}{4} \right )^k \left ( 1 - \alpha + \frac{\alpha^2}{4}\right )^{d_\gX - r - k} \\ 
   & = 2^{d_\gX - r} \left ( \frac{\alpha^2}{4} + 1 - \alpha + \frac{\alpha^2}{4} \right )^{d_\gX - r} \\ 
   & \le \left ( \alpha^2 - 2 \alpha + 2 \right )^{d_\gX - r} \le \left ( \alpha^2 - 2 \alpha + 2 \right )^{(1-\alpha/2)d_{\gX}}  ,
\end{align*}
which completes the proof.
\end{proof}

%% file: main.bbl
\begin{thebibliography}{101}
  \providecommand{\natexlab}[1]{#1}
  \providecommand{\url}[1]{\texttt{#1}}
  \expandafter\ifx\csname urlstyle\endcsname\relax
    \providecommand{\doi}[1]{doi: #1}\else
    \providecommand{\doi}{doi: \begingroup \urlstyle{rm}\Url}\fi
  
  \bibitem[Allen-Zhu et~al.(2019{\natexlab{a}})Allen-Zhu, Li, and Liang]{allen2019learning}
  Zeyuan Allen-Zhu, Yuanzhi Li, and Yingyu Liang.
  \newblock Learning and generalization in overparameterized neural networks, going beyond two layers.
  \newblock \emph{Advances in neural information processing systems}, 32, 2019{\natexlab{a}}.
  
  \bibitem[Allen-Zhu et~al.(2019{\natexlab{b}})Allen-Zhu, Li, and Song]{allen2019convergence}
  Zeyuan Allen-Zhu, Yuanzhi Li, and Zhao Song.
  \newblock A convergence theory for deep learning via over-parameterization.
  \newblock In \emph{International Conference on Machine Learning}, pp.\  242--252. PMLR, 2019{\natexlab{b}}.
  
  \bibitem[Arora et~al.(2018)Arora, Ge, Neyshabur, and Zhang]{arora2018stronger}
  Sanjeev Arora, Rong Ge, Behnam Neyshabur, and Yi~Zhang.
  \newblock Stronger generalization bounds for deep nets via a compression approach.
  \newblock In \emph{International Conference on Machine Learning}, pp.\  254--263. PMLR, 2018.
  
  \bibitem[Arora et~al.(2019{\natexlab{a}})Arora, Cohen, Hu, and Luo]{arora2019implicit}
  Sanjeev Arora, Nadav Cohen, Wei Hu, and Yuping Luo.
  \newblock Implicit regularization in deep matrix factorization.
  \newblock \emph{Advances in Neural Information Processing Systems}, 32, 2019{\natexlab{a}}.
  
  \bibitem[Arora et~al.(2019{\natexlab{b}})Arora, Du, Hu, Li, and Wang]{arora2019fine}
  Sanjeev Arora, Simon Du, Wei Hu, Zhiyuan Li, and Ruosong Wang.
  \newblock Fine-grained analysis of optimization and generalization for overparameterized two-layer neural networks.
  \newblock In \emph{International Conference on Machine Learning}, pp.\  322--332. PMLR, 2019{\natexlab{b}}.
  
  \bibitem[Arora et~al.(2019{\natexlab{c}})Arora, Du, Hu, Li, Salakhutdinov, and Wang]{arora2019exact}
  Sanjeev Arora, Simon~S Du, Wei Hu, Zhiyuan Li, Russ~R Salakhutdinov, and Ruosong Wang.
  \newblock On exact computation with an infinitely wide neural net.
  \newblock \emph{Advances in neural information processing systems}, 32, 2019{\natexlab{c}}.
  
  \bibitem[Bardes et~al.(2022)Bardes, Ponce, and LeCun]{bardes2021vicreg}
  Adrien Bardes, Jean Ponce, and Yann LeCun.
  \newblock {VICR}eg: Variance-invariance-covariance regularization for self-supervised learning.
  \newblock In \emph{International Conference on Learning Representations}, 2022.
  \newblock URL \url{https://openreview.net/forum?id=xm6YD62D1Ub}.
  
  \bibitem[Bartlett et~al.(2017)Bartlett, Foster, and Telgarsky]{bartlett2017spectrally}
  Peter~L Bartlett, Dylan~J Foster, and Matus~J Telgarsky.
  \newblock Spectrally-normalized margin bounds for neural networks.
  \newblock \emph{Advances in neural information processing systems}, 30, 2017.
  
  \bibitem[Bartlett et~al.(2023)Bartlett, Long, and Bousquet]{bartlett2022dynamics}
  Peter~L. Bartlett, Philip~M. Long, and Olivier Bousquet.
  \newblock The dynamics of sharpness-aware minimization: Bouncing across ravines and drifting towards wide minima.
  \newblock \emph{Journal of Machine Learning Research}, 24\penalty0 (316):\penalty0 1--36, 2023.
  \newblock URL \url{http://jmlr.org/papers/v24/23-043.html}.
  
  \bibitem[Belkin \& Niyogi(2003)Belkin and Niyogi]{belkin2003laplacian}
  Mikhail Belkin and Partha Niyogi.
  \newblock Laplacian eigenmaps for dimensionality reduction and data representation.
  \newblock \emph{Neural computation}, 15\penalty0 (6):\penalty0 1373--1396, 2003.
  
  \bibitem[Belkin \& Niyogi(2004)Belkin and Niyogi]{belkin2004semi}
  Mikhail Belkin and Partha Niyogi.
  \newblock Semi-supervised learning on riemannian manifolds.
  \newblock \emph{Machine learning}, 56:\penalty0 209--239, 2004.
  
  \bibitem[Bengio et~al.(2004)Bengio, Delalleau, Roux, Paiement, Vincent, and Ouimet]{bengio2004learning}
  Yoshua Bengio, Olivier Delalleau, Nicolas~Le Roux, Jean-Fran{\c{c}}ois Paiement, Pascal Vincent, and Marie Ouimet.
  \newblock Learning eigenfunctions links spectral embedding and kernel pca.
  \newblock \emph{Neural computation}, 16\penalty0 (10):\penalty0 2197--2219, 2004.
  
  \bibitem[Bengio et~al.(2013)Bengio, Courville, and Vincent]{bengio2013representation}
  Yoshua Bengio, Aaron Courville, and Pascal Vincent.
  \newblock Representation learning: A review and new perspectives.
  \newblock \emph{IEEE transactions on pattern analysis and machine intelligence}, 35\penalty0 (8):\penalty0 1798--1828, 2013.
  
  \bibitem[Bousquet \& Elisseeff(2002)Bousquet and Elisseeff]{bousquet2002stability}
  Olivier Bousquet and Andr{\'e} Elisseeff.
  \newblock Stability and generalization.
  \newblock \emph{The Journal of Machine Learning Research}, 2:\penalty0 499--526, 2002.
  
  \bibitem[Cabannes et~al.(2023)Cabannes, Kiani, Balestriero, Lecun, and Bietti]{cabannes2023ssl}
  Vivien Cabannes, Bobak Kiani, Randall Balestriero, Yann Lecun, and Alberto Bietti.
  \newblock The {SSL} interplay: Augmentations, inductive bias, and generalization.
  \newblock In Andreas Krause, Emma Brunskill, Kyunghyun Cho, Barbara Engelhardt, Sivan Sabato, and Jonathan Scarlett (eds.), \emph{Proceedings of the 40th International Conference on Machine Learning}, volume 202 of \emph{Proceedings of Machine Learning Research}, pp.\  3252--3298. PMLR, 23--29 Jul 2023.
  
  \bibitem[Chan et~al.(2022)Chan, Yu, You, Qi, Wright, and Ma]{chan2022redunet}
  Kwan Ho~Ryan Chan, Yaodong Yu, Chong You, Haozhi Qi, John Wright, and Yi~Ma.
  \newblock Redunet: A white-box deep network from the principle of maximizing rate reduction.
  \newblock \emph{The Journal of Machine Learning Research}, 23\penalty0 (1):\penalty0 4907--5009, 2022.
  
  \bibitem[Chen et~al.(2020)Chen, Kornblith, Norouzi, and Hinton]{chen2020simple}
  Ting Chen, Simon Kornblith, Mohammad Norouzi, and Geoffrey Hinton.
  \newblock A simple framework for contrastive learning of visual representations.
  \newblock In \emph{International conference on machine learning}, pp.\  1597--1607. PMLR, 2020.
  
  \bibitem[Chizat et~al.(2019)Chizat, Oyallon, and Bach]{chizat2019lazy}
  Lenaic Chizat, Edouard Oyallon, and Francis Bach.
  \newblock On lazy training in differentiable programming.
  \newblock \emph{Advances in neural information processing systems}, 32, 2019.
  
  \bibitem[Chung(1997)]{chung1997spectral}
  Fan~RK Chung.
  \newblock \emph{Spectral graph theory}, volume~92.
  \newblock American Mathematical Soc., 1997.
  
  \bibitem[Dai et~al.(2022)Dai, Tong, Li, Wu, Psenka, Chan, Zhai, Yu, Yuan, Shum, et~al.]{dai2022ctrl}
  Xili Dai, Shengbang Tong, Mingyang Li, Ziyang Wu, Michael Psenka, Kwan Ho~Ryan Chan, Pengyuan Zhai, Yaodong Yu, Xiaojun Yuan, Heung-Yeung Shum, et~al.
  \newblock Ctrl: Closed-loop transcription to an ldr via minimaxing rate reduction.
  \newblock \emph{Entropy}, 24\penalty0 (4):\penalty0 456, 2022.
  
  \bibitem[Dao et~al.(2019)Dao, Gu, Ratner, Smith, De~Sa, and R{\'e}]{dao2019kernel}
  Tri Dao, Albert Gu, Alexander Ratner, Virginia Smith, Chris De~Sa, and Christopher R{\'e}.
  \newblock A kernel theory of modern data augmentation.
  \newblock In \emph{International conference on machine learning}, pp.\  1528--1537. PMLR, 2019.
  
  \bibitem[Devlin et~al.(2019)Devlin, Chang, Lee, and Toutanova]{devlin2018bert}
  Jacob Devlin, Ming-Wei Chang, Kenton Lee, and Kristina Toutanova.
  \newblock {BERT}: Pre-training of deep bidirectional transformers for language understanding.
  \newblock In \emph{Proceedings of the 2019 Conference of the North {A}merican Chapter of the Association for Computational Linguistics: Human Language Technologies, Volume 1 (Long and Short Papers)}, pp.\  4171--4186, Minneapolis, Minnesota, June 2019. Association for Computational Linguistics.
  \newblock \doi{10.18653/v1/N19-1423}.
  \newblock URL \url{https://aclanthology.org/N19-1423}.
  
  \bibitem[DeVries \& Taylor(2017)DeVries and Taylor]{devries2017improved}
  Terrance DeVries and Graham~W Taylor.
  \newblock Improved regularization of convolutional neural networks with cutout.
  \newblock \emph{arXiv preprint arXiv:1708.04552}, 2017.
  
  \bibitem[Dosovitskiy et~al.(2021)Dosovitskiy, Beyer, Kolesnikov, Weissenborn, Zhai, Unterthiner, Dehghani, Minderer, Heigold, Gelly, Uszkoreit, and Houlsby]{dosovitskiy2021an}
  Alexey Dosovitskiy, Lucas Beyer, Alexander Kolesnikov, Dirk Weissenborn, Xiaohua Zhai, Thomas Unterthiner, Mostafa Dehghani, Matthias Minderer, Georg Heigold, Sylvain Gelly, Jakob Uszkoreit, and Neil Houlsby.
  \newblock An image is worth 16x16 words: Transformers for image recognition at scale.
  \newblock In \emph{International Conference on Learning Representations}, 2021.
  \newblock URL \url{https://openreview.net/forum?id=YicbFdNTTy}.
  
  \bibitem[Du \& Lee(2018)Du and Lee]{du2018power}
  Simon Du and Jason Lee.
  \newblock On the power of over-parametrization in neural networks with quadratic activation.
  \newblock In \emph{International conference on machine learning}, pp.\  1329--1338. PMLR, 2018.
  
  \bibitem[Du et~al.(2019{\natexlab{a}})Du, Lee, Li, Wang, and Zhai]{du2019gradient}
  Simon Du, Jason Lee, Haochuan Li, Liwei Wang, and Xiyu Zhai.
  \newblock Gradient descent finds global minima of deep neural networks.
  \newblock In \emph{International conference on machine learning}, pp.\  1675--1685. PMLR, 2019{\natexlab{a}}.
  
  \bibitem[Du et~al.(2019{\natexlab{b}})Du, Zhai, Poczos, and Singh]{du2018gradient}
  Simon~S. Du, Xiyu Zhai, Barnabas Poczos, and Aarti Singh.
  \newblock Gradient descent provably optimizes over-parameterized neural networks.
  \newblock In \emph{International Conference on Learning Representations}, 2019{\natexlab{b}}.
  \newblock URL \url{https://openreview.net/forum?id=S1eK3i09YQ}.
  
  \bibitem[Dziugaite \& Roy(2017)Dziugaite and Roy]{dziugaite2017computing}
  Gintare~Karolina Dziugaite and Daniel~M Roy.
  \newblock Computing nonvacuous generalization bounds for deep (stochastic) neural networks with many more parameters than training data.
  \newblock In Gal Elidan, Kristian Kersting, and Alexander~T. Ihler (eds.), \emph{Proceedings of the Thirty-third Conference on Uncertainty in Artificial Intelligence}, 2017.
  
  \bibitem[Erd{\H{o}}s et~al.(2009)Erd{\H{o}}s, Schlein, and Yau]{erdHos2009local}
  L{\'a}szl{\'o} Erd{\H{o}}s, Benjamin Schlein, and Horng-Tzer Yau.
  \newblock Local semicircle law and complete delocalization for wigner random matrices.
  \newblock \emph{Communications in Mathematical Physics}, 287\penalty0 (2):\penalty0 641--655, 2009.
  
  \bibitem[Fan(1949)]{fan1949theorem}
  Ky~Fan.
  \newblock On a theorem of weyl concerning eigenvalues of linear transformations i.
  \newblock \emph{Proceedings of the National Academy of Sciences}, 35\penalty0 (11):\penalty0 652--655, 1949.
  
  \bibitem[Foret et~al.(2021)Foret, Kleiner, Mobahi, and Neyshabur]{foret2021sharpnessaware}
  Pierre Foret, Ariel Kleiner, Hossein Mobahi, and Behnam Neyshabur.
  \newblock Sharpness-aware minimization for efficiently improving generalization.
  \newblock In \emph{International Conference on Learning Representations}, 2021.
  \newblock URL \url{https://openreview.net/forum?id=6Tm1mposlrM}.
  
  \bibitem[{Goldblum} et~al.(2023){Goldblum}, {Finzi}, {Rowan}, and {Wilson}]{2023arXiv230405366G}
  Micah {Goldblum}, Marc {Finzi}, Keefer {Rowan}, and Andrew~Gordon {Wilson}.
  \newblock {The No Free Lunch Theorem, Kolmogorov Complexity, and the Role of Inductive Biases in Machine Learning}.
  \newblock \emph{arXiv e-prints}, art. arXiv:2304.05366, April 2023.
  \newblock \doi{10.48550/arXiv.2304.05366}.
  
  \bibitem[Gy{\"o}rfi et~al.(2002)Gy{\"o}rfi, Kohler, Krzyzak, Walk, et~al.]{gyorfi2002distribution}
  L{\'a}szl{\'o} Gy{\"o}rfi, Michael Kohler, Adam Krzyzak, Harro Walk, et~al.
  \newblock \emph{A distribution-free theory of nonparametric regression}, volume~1.
  \newblock Springer, 2002.
  
  \bibitem[Hahn \& Meeker(2011)Hahn and Meeker]{hahn2011statistical}
  Gerald~J Hahn and William~Q Meeker.
  \newblock \emph{Statistical intervals: a guide for practitioners}, volume~92.
  \newblock John Wiley \& Sons, 2011.
  
  \bibitem[HaoChen \& Ma(2023)HaoChen and Ma]{haochen2022theoretical}
  Jeff~Z. HaoChen and Tengyu Ma.
  \newblock A theoretical study of inductive biases in contrastive learning.
  \newblock In \emph{The Eleventh International Conference on Learning Representations}, 2023.
  \newblock URL \url{https://openreview.net/forum?id=AuEgNlEAmed}.
  
  \bibitem[HaoChen et~al.(2021)HaoChen, Wei, Gaidon, and Ma]{haochen2021provable}
  Jeff~Z HaoChen, Colin Wei, Adrien Gaidon, and Tengyu Ma.
  \newblock Provable guarantees for self-supervised deep learning with spectral contrastive loss.
  \newblock \emph{Advances in Neural Information Processing Systems}, 34:\penalty0 5000--5011, 2021.
  
  \bibitem[HaoChen et~al.(2022)HaoChen, Wei, Kumar, and Ma]{haochen2022beyond}
  Jeff~Z. HaoChen, Colin Wei, Ananya Kumar, and Tengyu Ma.
  \newblock Beyond separability: Analyzing the linear transferability of contrastive representations to related subpopulations.
  \newblock In Alice~H. Oh, Alekh Agarwal, Danielle Belgrave, and Kyunghyun Cho (eds.), \emph{Advances in Neural Information Processing Systems}, 2022.
  \newblock URL \url{https://openreview.net/forum?id=vmjckXzRXmh}.
  
  \bibitem[Hardt et~al.(2016)Hardt, Recht, and Singer]{hardt2016train}
  Moritz Hardt, Ben Recht, and Yoram Singer.
  \newblock Train faster, generalize better: Stability of stochastic gradient descent.
  \newblock In \emph{International conference on machine learning}, pp.\  1225--1234. PMLR, 2016.
  
  \bibitem[He et~al.(2022)He, Chen, Xie, Li, Doll{\'a}r, and Girshick]{he2022masked}
  Kaiming He, Xinlei Chen, Saining Xie, Yanghao Li, Piotr Doll{\'a}r, and Ross Girshick.
  \newblock Masked autoencoders are scalable vision learners.
  \newblock In \emph{Proceedings of the IEEE/CVF Conference on Computer Vision and Pattern Recognition}, pp.\  16000--16009, 2022.
  
  \bibitem[Hochreiter \& Schmidhuber(1997)Hochreiter and Schmidhuber]{hochreiter1997flat}
  Sepp Hochreiter and J{\"u}rgen Schmidhuber.
  \newblock Flat minima.
  \newblock \emph{Neural computation}, 9\penalty0 (1):\penalty0 1--42, 1997.
  
  \bibitem[Huang et~al.(2023)Huang, Yi, Zhao, and Jiang]{huang2023towards}
  Weiran Huang, Mingyang Yi, Xuyang Zhao, and Zihao Jiang.
  \newblock Towards the generalization of contrastive self-supervised learning.
  \newblock In \emph{The Eleventh International Conference on Learning Representations}, 2023.
  \newblock URL \url{https://openreview.net/forum?id=XDJwuEYHhme}.
  
  \bibitem[Huang et~al.(2020)Huang, Perez, Ba, and Volkovs]{pmlr-v119-huang20f}
  Xiao~Shi Huang, Felipe Perez, Jimmy Ba, and Maksims Volkovs.
  \newblock Improving transformer optimization through better initialization.
  \newblock In Hal~Daumé III and Aarti Singh (eds.), \emph{Proceedings of the 37th International Conference on Machine Learning}, volume 119 of \emph{Proceedings of Machine Learning Research}, pp.\  4475--4483. PMLR, 13--18 Jul 2020.
  
  \bibitem[Jacot et~al.(2018)Jacot, Gabriel, and Hongler]{jacot2018neural}
  Arthur Jacot, Franck Gabriel, and Cl{\'e}ment Hongler.
  \newblock Neural tangent kernel: Convergence and generalization in neural networks.
  \newblock \emph{Advances in neural information processing systems}, 31, 2018.
  
  \bibitem[Jiang et~al.(2020)Jiang, Neyshabur, Mobahi, Krishnan, and Bengio]{jiang2019fantastic}
  Yiding Jiang, Behnam Neyshabur, Hossein Mobahi, Dilip Krishnan, and Samy Bengio.
  \newblock Fantastic generalization measures and where to find them.
  \newblock In \emph{International Conference on Learning Representations}, 2020.
  \newblock URL \url{https://openreview.net/forum?id=SJgIPJBFvH}.
  
  \bibitem[Jing et~al.(2022)Jing, Vincent, LeCun, and Tian]{jing2022understanding}
  Li~Jing, Pascal Vincent, Yann LeCun, and Yuandong Tian.
  \newblock Understanding dimensional collapse in contrastive self-supervised learning.
  \newblock In \emph{International Conference on Learning Representations}, 2022.
  \newblock URL \url{https://openreview.net/forum?id=YevsQ05DEN7}.
  
  \bibitem[Johnson et~al.(2023)Johnson, Hanchi, and Maddison]{johnson2022contrastive}
  Daniel~D. Johnson, Ayoub~El Hanchi, and Chris~J. Maddison.
  \newblock Contrastive learning can find an optimal basis for approximately view-invariant functions.
  \newblock In \emph{The Eleventh International Conference on Learning Representations}, 2023.
  \newblock URL \url{https://openreview.net/forum?id=AjC0KBjiMu}.
  
  \bibitem[Kawaguchi(2016)]{kawaguchi2016deep}
  Kenji Kawaguchi.
  \newblock Deep learning without poor local minima.
  \newblock \emph{Advances in neural information processing systems}, 29, 2016.
  
  \bibitem[Kearns \& Ron(1997)Kearns and Ron]{kearns1997algorithmic}
  Michael Kearns and Dana Ron.
  \newblock Algorithmic stability and sanity-check bounds for leave-one-out cross-validation.
  \newblock In \emph{Proceedings of the tenth annual conference on Computational learning theory}, pp.\  152--162, 1997.
  
  \bibitem[Keskar et~al.(2017)Keskar, Mudigere, Nocedal, Smelyanskiy, and Tang]{keskar2016large}
  Nitish~Shirish Keskar, Dheevatsa Mudigere, Jorge Nocedal, Mikhail Smelyanskiy, and Ping Tak~Peter Tang.
  \newblock On large-batch training for deep learning: Generalization gap and sharp minima.
  \newblock In \emph{International Conference on Learning Representations}, 2017.
  \newblock URL \url{https://openreview.net/forum?id=H1oyRlYgg}.
  
  \bibitem[Kolmogorov(1963)]{kolmogorov1963tables}
  Andrei~N Kolmogorov.
  \newblock On tables of random numbers.
  \newblock \emph{Sankhy{\=a}: The Indian Journal of Statistics, Series A}, pp.\  369--376, 1963.
  
  \bibitem[Langford \& Seeger(2001)Langford and Seeger]{langford2001bounds}
  John Langford and Matthias Seeger.
  \newblock \emph{Bounds for averaging classifiers}.
  \newblock School of Computer Science, Carnegie Mellon University, 2001.
  
  \bibitem[Lee et~al.(2019)Lee, Xiao, Schoenholz, Bahri, Novak, Sohl-Dickstein, and Pennington]{lee2019wide}
  Jaehoon Lee, Lechao Xiao, Samuel Schoenholz, Yasaman Bahri, Roman Novak, Jascha Sohl-Dickstein, and Jeffrey Pennington.
  \newblock Wide neural networks of any depth evolve as linear models under gradient descent.
  \newblock \emph{Advances in neural information processing systems}, 32, 2019.
  
  \bibitem[Lee et~al.(2021)Lee, Lei, Saunshi, and Zhuo]{lee2021predicting}
  Jason~D Lee, Qi~Lei, Nikunj Saunshi, and Jiacheng Zhuo.
  \newblock Predicting what you already know helps: Provable self-supervised learning.
  \newblock \emph{Advances in Neural Information Processing Systems}, 34:\penalty0 309--323, 2021.
  
  \bibitem[Li \& Yuan(2017)Li and Yuan]{li2017convergence}
  Yuanzhi Li and Yang Yuan.
  \newblock Convergence analysis of two-layer neural networks with relu activation.
  \newblock \emph{Advances in neural information processing systems}, 30, 2017.
  
  \bibitem[Li et~al.(2018)Li, Ma, and Zhang]{li2018algorithmic}
  Yuanzhi Li, Tengyu Ma, and Hongyang Zhang.
  \newblock Algorithmic regularization in over-parameterized matrix sensing and neural networks with quadratic activations.
  \newblock In \emph{Conference On Learning Theory}, pp.\  2--47. PMLR, 2018.
  
  \bibitem[Liu et~al.(2022)Liu, Hsu, Ravikumar, and Risteski]{liu2022masked}
  Bingbin Liu, Daniel Hsu, Pradeep Ravikumar, and Andrej Risteski.
  \newblock Masked prediction tasks: a parameter identifiability view.
  \newblock \emph{Advances in Neural Information Processing Systems}, 2022.
  
  \bibitem[Liu et~al.(2021)Liu, HaoChen, Gaidon, and Ma]{liu2021selfsupervised}
  Hong Liu, Jeff~Z. HaoChen, Adrien Gaidon, and Tengyu Ma.
  \newblock Self-supervised learning is more robust to dataset imbalance.
  \newblock In \emph{NeurIPS 2021 Workshop on Distribution Shifts: Connecting Methods and Applications}, 2021.
  \newblock URL \url{https://openreview.net/forum?id=vUz4JPRLpGx}.
  
  \bibitem[Lotfi et~al.(2022)Lotfi, Finzi, Kapoor, Potapczynski, Goldblum, and Wilson]{lotfi2022pac}
  Sanae Lotfi, Marc~Anton Finzi, Sanyam Kapoor, Andres Potapczynski, Micah Goldblum, and Andrew~Gordon Wilson.
  \newblock {PAC}-bayes compression bounds so tight that they can explain generalization.
  \newblock In Alice~H. Oh, Alekh Agarwal, Danielle Belgrave, and Kyunghyun Cho (eds.), \emph{Advances in Neural Information Processing Systems}, 2022.
  \newblock URL \url{https://openreview.net/forum?id=o8nYuR8ekFm}.
  
  \bibitem[McAllester(1999)]{mcallester1999pac}
  David~A McAllester.
  \newblock Pac-bayesian model averaging.
  \newblock In \emph{Proceedings of the twelfth annual conference on Computational learning theory}, pp.\  164--170, 1999.
  
  \bibitem[Mohri et~al.(2018)Mohri, Rostamizadeh, and Talwalkar]{mohri2018foundations}
  Mehryar Mohri, Afshin Rostamizadeh, and Ameet Talwalkar.
  \newblock \emph{Foundations of machine learning}.
  \newblock MIT press, 2018.
  
  \bibitem[Mroueh et~al.(2015)Mroueh, Voinea, and Poggio]{mroueh2015learning}
  Youssef Mroueh, Stephen Voinea, and Tomaso~A Poggio.
  \newblock Learning with group invariant features: A kernel perspective.
  \newblock \emph{Advances in neural information processing systems}, 28, 2015.
  
  \bibitem[Mukherjee et~al.(2006)Mukherjee, Niyogi, Poggio, and Rifkin]{mukherjee2006learning}
  Sayan Mukherjee, Partha Niyogi, Tomaso Poggio, and Ryan Rifkin.
  \newblock Learning theory: stability is sufficient for generalization and necessary and sufficient for consistency of empirical risk minimization.
  \newblock \emph{Advances in Computational Mathematics}, 25:\penalty0 161--193, 2006.
  
  \bibitem[Nagarajan \& Kolter(2019)Nagarajan and Kolter]{nagarajan2019deterministic}
  Vaishnavh Nagarajan and Zico Kolter.
  \newblock Deterministic {PAC}-bayesian generalization bounds for deep networks via generalizing noise-resilience.
  \newblock In \emph{International Conference on Learning Representations}, 2019.
  \newblock URL \url{https://openreview.net/forum?id=Hygn2o0qKX}.
  
  \bibitem[Negrea et~al.(2020)Negrea, Dziugaite, and Roy]{negrea2020defense}
  Jeffrey Negrea, Gintare~Karolina Dziugaite, and Daniel Roy.
  \newblock In defense of uniform convergence: Generalization via derandomization with an application to interpolating predictors.
  \newblock In \emph{International Conference on Machine Learning}, pp.\  7263--7272. PMLR, 2020.
  
  \bibitem[Neyshabur et~al.(2017)Neyshabur, Bhojanapalli, McAllester, and Srebro]{neyshabur2017exploring}
  Behnam Neyshabur, Srinadh Bhojanapalli, David McAllester, and Nati Srebro.
  \newblock Exploring generalization in deep learning.
  \newblock \emph{Advances in neural information processing systems}, 30, 2017.
  
  \bibitem[Neyshabur et~al.(2018)Neyshabur, Bhojanapalli, and Srebro]{neyshabur2017pac}
  Behnam Neyshabur, Srinadh Bhojanapalli, and Nathan Srebro.
  \newblock A {PAC}-bayesian approach to spectrally-normalized margin bounds for neural networks.
  \newblock In \emph{International Conference on Learning Representations}, 2018.
  \newblock URL \url{https://openreview.net/forum?id=Skz_WfbCZ}.
  
  \bibitem[Pokle et~al.(2022)Pokle, Tian, Li, and Risteski]{PokleTLR22}
  Ashwini Pokle, Jinjin Tian, Yuchen Li, and Andrej Risteski.
  \newblock Contrasting the landscape of contrastive and non-contrastive learning.
  \newblock In Gustau Camps{-}Valls, Francisco J.~R. Ruiz, and Isabel Valera (eds.), \emph{International Conference on Artificial Intelligence and Statistics, {AISTATS} 2022, 28-30 March 2022, Virtual Event}, volume 151 of \emph{Proceedings of Machine Learning Research}, pp.\  8592--8618. {PMLR}, 2022.
  \newblock URL \url{https://proceedings.mlr.press/v151/pokle22a.html}.
  
  \bibitem[Radford et~al.(2021)Radford, Kim, Hallacy, Ramesh, Goh, Agarwal, Sastry, Askell, Mishkin, Clark, et~al.]{CLIP}
  Alec Radford, Jong~Wook Kim, Chris Hallacy, Aditya Ramesh, Gabriel Goh, Sandhini Agarwal, Girish Sastry, Amanda Askell, Pamela Mishkin, Jack Clark, et~al.
  \newblock Learning transferable visual models from natural language supervision.
  \newblock In \emph{International Conference on Machine Learning}, pp.\  8748--8763. PMLR, 2021.
  
  \bibitem[Raj et~al.(2017)Raj, Kumar, Mroueh, Fletcher, and Sch{\"o}lkopf]{raj2017local}
  Anant Raj, Abhishek Kumar, Youssef Mroueh, Tom Fletcher, and Bernhard Sch{\"o}lkopf.
  \newblock Local group invariant representations via orbit embeddings.
  \newblock In \emph{Artificial Intelligence and Statistics}, pp.\  1225--1235. PMLR, 2017.
  
  \bibitem[Rivasplata et~al.(2019)Rivasplata, Tankasali, and Szepesv{\'a}ri]{rivasplata2019pac}
  Omar Rivasplata, Vikram~M Tankasali, and Csaba Szepesv{\'a}ri.
  \newblock Pac-bayes with backprop.
  \newblock \emph{arXiv preprint arXiv:1908.07380}, 2019.
  
  \bibitem[Saunshi et~al.(2019)Saunshi, Plevrakis, Arora, Khodak, and Khandeparkar]{saunshi2019theoretical}
  Nikunj Saunshi, Orestis Plevrakis, Sanjeev Arora, Mikhail Khodak, and Hrishikesh Khandeparkar.
  \newblock A theoretical analysis of contrastive unsupervised representation learning.
  \newblock In \emph{International Conference on Machine Learning}, pp.\  5628--5637. PMLR, 2019.
  
  \bibitem[Saunshi et~al.(2022)Saunshi, Ash, Goel, Misra, Zhang, Arora, Kakade, and Krishnamurthy]{saunshi2022understanding}
  Nikunj Saunshi, Jordan Ash, Surbhi Goel, Dipendra Misra, Cyril Zhang, Sanjeev Arora, Sham Kakade, and Akshay Krishnamurthy.
  \newblock Understanding contrastive learning requires incorporating inductive biases.
  \newblock In \emph{International Conference on Machine Learning}, pp.\  19250--19286. PMLR, 2022.
  
  \bibitem[Sch{\"o}lkopf \& Smola(2002)Sch{\"o}lkopf and Smola]{scholkopf2002learning}
  Bernhard Sch{\"o}lkopf and Alexander~J Smola.
  \newblock \emph{Learning with kernels: support vector machines, regularization, optimization, and beyond}.
  \newblock MIT press, 2002.
  
  \bibitem[Shalev-Shwartz et~al.(2010)Shalev-Shwartz, Shamir, Srebro, and Sridharan]{shalev2010learnability}
  Shai Shalev-Shwartz, Ohad Shamir, Nathan Srebro, and Karthik Sridharan.
  \newblock Learnability, stability and uniform convergence.
  \newblock \emph{The Journal of Machine Learning Research}, 11:\penalty0 2635--2670, 2010.
  
  \bibitem[Shen et~al.(2022)Shen, Jones, Kumar, Xie, HaoChen, Ma, and Liang]{shen2022connect}
  Kendrick Shen, Robbie~M Jones, Ananya Kumar, Sang~Michael Xie, Jeff~Z HaoChen, Tengyu Ma, and Percy Liang.
  \newblock Connect, not collapse: Explaining contrastive learning for unsupervised domain adaptation.
  \newblock In \emph{International Conference on Machine Learning}, pp.\  19847--19878. PMLR, 2022.
  
  \bibitem[Socher et~al.(2013)Socher, Perelygin, Wu, Chuang, Manning, Ng, and Potts]{SST2}
  Richard Socher, Alex Perelygin, Jean Wu, Jason Chuang, Christopher~D. Manning, Andrew Ng, and Christopher Potts.
  \newblock Recursive deep models for semantic compositionality over a sentiment treebank.
  \newblock In \emph{Proceedings of the 2013 Conference on Empirical Methods in Natural Language Processing}, pp.\  1631--1642, Seattle, Washington, USA, October 2013. Association for Computational Linguistics.
  \newblock URL \url{https://www.aclweb.org/anthology/D13-1170}.
  
  \bibitem[Tian et~al.(2020)Tian, Sun, Poole, Krishnan, Schmid, and Isola]{tian2020makes}
  Yonglong Tian, Chen Sun, Ben Poole, Dilip Krishnan, Cordelia Schmid, and Phillip Isola.
  \newblock What makes for good views for contrastive learning?
  \newblock \emph{Advances in Neural Information Processing Systems}, 33:\penalty0 6827--6839, 2020.
  
  \bibitem[Tian(2022)]{tian2022deep}
  Yuandong Tian.
  \newblock Deep contrastive learning is provably (almost) principal component analysis.
  \newblock \emph{Advances in Neural Information Processing Systems}, 2022.
  
  \bibitem[Tian et~al.(2021)Tian, Chen, and Ganguli]{TianCG21}
  Yuandong Tian, Xinlei Chen, and Surya Ganguli.
  \newblock Understanding self-supervised learning dynamics without contrastive pairs.
  \newblock In Marina Meila and Tong Zhang (eds.), \emph{Proceedings of the 38th International Conference on Machine Learning, {ICML} 2021, 18-24 July 2021, Virtual Event}, volume 139 of \emph{Proceedings of Machine Learning Research}, pp.\  10268--10278. {PMLR}, 2021.
  \newblock URL \url{http://proceedings.mlr.press/v139/tian21a.html}.
  
  \bibitem[Tosh et~al.(2021{\natexlab{a}})Tosh, Krishnamurthy, and Hsu]{tosh2021contrastive}
  Christopher Tosh, Akshay Krishnamurthy, and Daniel Hsu.
  \newblock Contrastive estimation reveals topic posterior information to linear models.
  \newblock \emph{J. Mach. Learn. Res.}, 22:\penalty0 281--1, 2021{\natexlab{a}}.
  
  \bibitem[Tosh et~al.(2021{\natexlab{b}})Tosh, Krishnamurthy, and Hsu]{tosh2021multiview}
  Christopher Tosh, Akshay Krishnamurthy, and Daniel Hsu.
  \newblock Contrastive learning, multi-view redundancy, and linear models.
  \newblock In \emph{Algorithmic Learning Theory}, pp.\  1179--1206. PMLR, 2021{\natexlab{b}}.
  
  \bibitem[Wainwright(2019)]{wainwright2019high}
  Martin~J Wainwright.
  \newblock \emph{High-dimensional statistics: A non-asymptotic viewpoint}, volume~48.
  \newblock Cambridge university press, 2019.
  
  \bibitem[Wang et~al.(2018)Wang, Singh, Michael, Hill, Levy, and Bowman]{wang2018glue}
  Alex Wang, Amanpreet Singh, Julian Michael, Felix Hill, Omer Levy, and Samuel~R. Bowman.
  \newblock Glue: A multi-task benchmark and analysis platform for natural language understanding.
  \newblock \emph{BLACKBOXNLP@EMNLP}, 2018.
  \newblock \doi{10.18653/v1/W18-5446}.
  
  \bibitem[Wang et~al.(2007)Wang, Yan, Xu, Tang, and Huang]{wang2007trace}
  Huan Wang, Shuicheng Yan, Dong Xu, Xiaoou Tang, and Thomas Huang.
  \newblock Trace ratio vs. ratio trace for dimensionality reduction.
  \newblock In \emph{2007 IEEE Conference on Computer Vision and Pattern Recognition}, pp.\  1--8. IEEE, 2007.
  
  \bibitem[Wang et~al.(2020)Wang, Wu, Lee, Ma, and Ge]{wang2020beyond}
  Xiang Wang, Chenwei Wu, Jason~D Lee, Tengyu Ma, and Rong Ge.
  \newblock Beyond lazy training for over-parameterized tensor decomposition.
  \newblock \emph{Advances in Neural Information Processing Systems}, 33:\penalty0 21934--21944, 2020.
  
  \bibitem[Wang et~al.(2022{\natexlab{a}})Wang, Zhang, Wang, Yang, and Lin]{wang2022chaos}
  Yifei Wang, Qi~Zhang, Yisen Wang, Jiansheng Yang, and Zhouchen Lin.
  \newblock Chaos is a ladder: A new theoretical understanding of contrastive learning via augmentation overlap.
  \newblock In \emph{International Conference on Learning Representations}, 2022{\natexlab{a}}.
  \newblock URL \url{https://openreview.net/forum?id=ECvgmYVyeUz}.
  
  \bibitem[Wang et~al.(2023)Wang, Zhang, Du, Yang, Lin, and Wang]{wang2023a}
  Yifei Wang, Qi~Zhang, Tianqi Du, Jiansheng Yang, Zhouchen Lin, and Yisen Wang.
  \newblock A message passing perspective on learning dynamics of contrastive learning.
  \newblock In \emph{The Eleventh International Conference on Learning Representations}, 2023.
  \newblock URL \url{https://openreview.net/forum?id=VBTJqqWjxMv}.
  
  \bibitem[Wang et~al.(2022{\natexlab{b}})Wang, Luo, Li, Zhu, and Sch{\"o}lkopf]{wang2022spectral}
  Ziyu Wang, Yucen Luo, Yueru Li, Jun Zhu, and Bernhard Sch{\"o}lkopf.
  \newblock Spectral representation learning for conditional moment models.
  \newblock \emph{arXiv preprint arXiv:2210.16525}, 2022{\natexlab{b}}.
  
  \bibitem[Wei et~al.(2021)Wei, Xie, and Ma]{wei2021pretrained}
  Colin Wei, Sang~Michael Xie, and Tengyu Ma.
  \newblock Why do pretrained language models help in downstream tasks? an analysis of head and prompt tuning.
  \newblock \emph{Advances in Neural Information Processing Systems}, 34:\penalty0 16158--16170, 2021.
  
  \bibitem[Wen et~al.(2023)Wen, Ma, and Li]{wen2022does}
  Kaiyue Wen, Tengyu Ma, and Zhiyuan Li.
  \newblock How sharpness-aware minimization minimizes sharpness?
  \newblock In \emph{The Eleventh International Conference on Learning Representations}, 2023.
  \newblock URL \url{https://openreview.net/forum?id=5spDgWmpY6x}.
  
  \bibitem[Wen \& Li(2021)Wen and Li]{WenLi21}
  Zixin Wen and Yuanzhi Li.
  \newblock Toward understanding the feature learning process of self-supervised contrastive learning.
  \newblock In Marina Meila and Tong Zhang (eds.), \emph{Proceedings of the 38th International Conference on Machine Learning, {ICML} 2021, 18-24 July 2021, Virtual Event}, volume 139 of \emph{Proceedings of Machine Learning Research}, pp.\  11112--11122. {PMLR}, 2021.
  \newblock URL \url{http://proceedings.mlr.press/v139/wen21c.html}.
  
  \bibitem[Wen \& Li(2022)Wen and Li]{wen2022mechanism}
  Zixin Wen and Yuanzhi Li.
  \newblock The mechanism of prediction head in non-contrastive self-supervised learning.
  \newblock \emph{Advances in Neural Information Processing Systems}, 2022.
  
  \bibitem[Wettig et~al.(2023)Wettig, Gao, Zhong, and Chen]{wettig2022should}
  Alexander Wettig, Tianyu Gao, Zexuan Zhong, and Danqi Chen.
  \newblock Should you mask 15\% in masked language modeling?
  \newblock In \emph{Proceedings of the 17th Conference of the European Chapter of the Association for Computational Linguistics: Main Volume}, 2023.
  
  \bibitem[Wolpert \& Macready(1997)Wolpert and Macready]{wolpert1997no}
  David~H Wolpert and William~G Macready.
  \newblock No free lunch theorems for optimization.
  \newblock \emph{IEEE transactions on evolutionary computation}, 1\penalty0 (1):\penalty0 67--82, 1997.
  
  \bibitem[Xiong et~al.(2020)Xiong, Yang, He, Zheng, Zheng, Xing, Zhang, Lan, Wang, and Liu]{xiong2020layer}
  Ruibin Xiong, Yunchang Yang, Di~He, Kai Zheng, Shuxin Zheng, Chen Xing, Huishuai Zhang, Yanyan Lan, Liwei Wang, and Tieyan Liu.
  \newblock On layer normalization in the transformer architecture.
  \newblock In \emph{International Conference on Machine Learning}, pp.\  10524--10533. PMLR, 2020.
  
  \bibitem[Yu et~al.(2020)Yu, Chan, You, Song, and Ma]{yu2020learning}
  Yaodong Yu, Kwan Ho~Ryan Chan, Chong You, Chaobing Song, and Yi~Ma.
  \newblock Learning diverse and discriminative representations via the principle of maximal coding rate reduction.
  \newblock \emph{Advances in Neural Information Processing Systems}, 33:\penalty0 9422--9434, 2020.
  
  \bibitem[Zbontar et~al.(2021)Zbontar, Jing, Misra, LeCun, and Deny]{zbontar2021barlow}
  Jure Zbontar, Li~Jing, Ishan Misra, Yann LeCun, and St{\'e}phane Deny.
  \newblock Barlow twins: Self-supervised learning via redundancy reduction.
  \newblock In \emph{International Conference on Machine Learning}, pp.\  12310--12320. PMLR, 2021.
  
  \bibitem[Zhai et~al.(2024)Zhai, Pukdee, Jin, Balcan, and Ravikumar]{zhai2024stkr}
  Runtian Zhai, Rattana Pukdee, Roger Jin, Maria-Florina Balcan, and Pradeep Ravikumar.
  \newblock Spectrally transformed kernel regression.
  \newblock In \emph{International Conference on Learning Representations}, 2024.
  \newblock URL \url{https://openreview.net/forum?id=OeQE9zsztS}.
  
  \bibitem[Zhang et~al.(2017)Zhang, Bengio, Hardt, Recht, and Vinyals]{zhang2017understanding}
  Chiyuan Zhang, Samy Bengio, Moritz Hardt, Benjamin Recht, and Oriol Vinyals.
  \newblock Understanding deep learning requires rethinking generalization.
  \newblock In \emph{International Conference on Learning Representations}, 2017.
  \newblock URL \url{https://openreview.net/forum?id=Sy8gdB9xx}.
  
  \bibitem[Zhang et~al.(2021)Zhang, Bengio, Hardt, Recht, and Vinyals]{zhang2021understanding}
  Chiyuan Zhang, Samy Bengio, Moritz Hardt, Benjamin Recht, and Oriol Vinyals.
  \newblock Understanding deep learning (still) requires rethinking generalization.
  \newblock \emph{Communications of the ACM}, 64\penalty0 (3):\penalty0 107--115, 2021.
  
  \bibitem[Zhou et~al.(2019)Zhou, Veitch, Austern, Adams, and Orbanz]{zhou2018non}
  Wenda Zhou, Victor Veitch, Morgane Austern, Ryan~P. Adams, and Peter Orbanz.
  \newblock Non-vacuous generalization bounds at the imagenet scale: a {PAC}-bayesian compression approach.
  \newblock In \emph{International Conference on Learning Representations}, 2019.
  \newblock URL \url{https://openreview.net/forum?id=BJgqqsAct7}.
  
  \end{thebibliography}
